\documentclass[conference]{IEEEtran}
\IEEEoverridecommandlockouts
\usepackage{cite}
\usepackage{amsmath,amssymb,amsfonts}
\usepackage{algorithmic}
\usepackage{graphicx}
\usepackage{textcomp}
\usepackage{wrapfig}
\usepackage{xcolor}

\usepackage{hyperref}
\usepackage[numbers]{natbib}
\usepackage[export]{adjustbox}
\usepackage{booktabs}       % professional-quality tables
\usepackage{amsfonts}       % blackboard math symbols
\usepackage{nicefrac}       % compact symbols for 1/2, etc.
\usepackage{microtype}      % microtypography

\usepackage{amsmath, amssymb, bm, physics, amsthm}
\usepackage{graphicx}
\usepackage{subfloat}
\usepackage{subfig}
\usepackage{url}
\usepackage{caption}
\usepackage{graphicx}
\graphicspath{{./figs/}}%{/home/thanhtung/github/figs/aistats2020-figs/figs/}{../}{/media/thanhtung/DATA2/github/figs/}}

\usepackage{ragged2e} 
\usepackage{tabulary}
\usepackage{siunitx}

\newtheorem{observation}{Observation}
\newtheorem{proposition}{Proposition}

\def\BibTeX{{\rm B\kern-.05em{\sc i\kern-.025em b}\kern-.08em
    T\kern-.1667em\lower.7ex\hbox{E}\kern-.125emX}}
\begin{document}

\title{Catastrophic forgetting and mode collapse in GANs}

\author{\IEEEauthorblockN{1\textsuperscript{st} Hoang Thanh-Tung}
\IEEEauthorblockA{\textit{Applied Artificial Intelligence Institute} \\
\textit{Deakin University}\\
hoangtha@deakin.edu.au}
\and
\IEEEauthorblockN{2\textsuperscript{nd} Truyen Tran}
\IEEEauthorblockA{\textit{Applied Artificial Intelligence Institute} \\
\textit{Deakin University}\\
truyen.tran@deakin.edu.au}
}

\maketitle

\begin{abstract}
In this paper, we show that Generative Adversarial Networks (GANs) suffer from catastrophic forgetting even when they are trained to approximate a single target distribution. 
We show that GAN training is a continual learning problem in which the sequence of changing model distributions is the sequence of tasks to the discriminator.
The level of mismatch between tasks in the sequence determines the level of forgetting.
Catastrophic forgetting is interrelated to mode collapse and can make the training of GANs non-convergent. 
We investigate the landscape of the discriminator's output in different variants of GANs and find that when a GAN converges to a good equilibrium, real training datapoints are wide local maxima of the discriminator.
We empirically show the relationship between the sharpness of local maxima and mode collapse and generalization in GANs.
We show how catastrophic forgetting prevents the discriminator from making real datapoints local maxima, and thus causes non-convergence.
Finally, we study methods for preventing catastrophic forgetting in GANs.
%Our findings shed light on methods for preventing catastrophic forgetting to help GANs learn the target distribution.
\end{abstract}


\section{Introduction}

%\begin{enumerate}
%\item Introduce GAN
%\item Introduce CF, Mode Collapse
%\item Introduce the main finding of this paper: we find that CF is present in GANs although they are trained to approximate a single target distribution.
%\item Other contributions: the relationship between CF and MC; the effect of CF and MC on the value surface; an explanation on how they could make GANs non-convergent; a study on methods for avoiding CF and making GANs converge.
%\end{enumerate}

%Generative Adversarial Networks (GANs) \cite{gan} are one of the most common tools for modeling high dimensional data. 
%However, the original GANs is unstable and often exhibits problems such as mode collapse \cite{gan} and non-convergence \cite{ganStable}. 
%To address the problems, a large number of architectures \cite{dcgan, progressiveGAN, styleGAN, bigGAN} and training strategies \cite{ttur, whichGANConverge, improvedGAN, improveGeneralization, spectralNorm, equiAndGeneralization} have been introduced. 

GANs \cite{gan, predictabilityMinimization} are a powerful tool for modeling complex distributions. Training a GAN to approximate a single target distribution is often considered as a single task.
In this paper, we introduce a novel view of GAN training as a continual learning problem in which the sequence of changing model distributions are considered as the sequence of tasks. We discover a surprising result that GANs  suffer from catastrophic forgetting, a problem often observed in continual learning settings \cite{ewc}. 
%In continual learning, the model is trained on a number of tasks that are introduced sequentially. 
%When the model is trained on one task, it cannot access training data from previous tasks. 
Catastrophic forgetting (CF) in artificial neural networks  \cite{catastrophicInterference, ratcliff1990connectionist, catastrophicConnectionist} is the problem where the knowledge of previously learned tasks is abruptly destroyed by the learning of the current task.
When a GAN suffers from CF, it exhibits undesired behaviors such as mode collapse and non-convergence.
%CF occurs in continual learning because at task $\mathcal{T}^t$, the learner cannot access data from previous tasks $\mathcal{T}^1, ..., \mathcal{T}^{t - 1}$ so the knowledge of previous tasks is not consolidated and is overwritten by the learning of the current task.

In section \ref{sec:catastrophic}, we show that GAN training is actually a continual learning problem and demonstrate the CF problem on a number of datasets. % in which the sequence of generated distributions serves as the sequence of tasks to the discriminator. 
%The similarity between distributions in the sequence determines the level of forgetting. 
We show that catastrophic forgetting and mode collapse \cite{gan} are two different but interrelated problems and together, they can make the training of GANs non-convergent (section \ref{sec:cfGAN}, \ref{sec:ctEffect}).
To avoid mode collapse and improve convergence, it is important to address the CF problem.
We identify 2 factors that causes CF in GANs: 1) Information from previous tasks is not used in the current task, 2) Knowledge from previous tasks is not usable for the current task and vice versa.
Our findings shed light on how to avoid catastrophic forgetting to learn the target distribution properly (Section \ref{sec:methods}).

%Section \ref{sec:catasCollapse} shows that CF is interrelated to mode collapse and together they can make the training of GANs non-convergent. 
In section \ref{sec:landscape}, we investigate the effect of CF and mode collapse on the landscape of the discriminator's output. 
We find that when a GAN converge to a good local equilibrium without mode collapse, real datapoints are wide local maxima of the discriminator.
We show that the sharper the local maxima are, the more severe mode collapse is.
Section \ref{sec:ctEffect} shows that when CF happen, the discriminator is directionally monotonic. A GAN with a directionally monotonic discriminator does not converge to an equilibrium. The fact confirms that CF is a cause of non-convergence.

Section \ref{sec:methods} explains how state-of-the-art methods for stabilizing GANs such as Wasserstein GAN \cite{wgan, wgangp}, zero-centered gradient penalty on training examples (GAN-R1) \cite{whichGANConverge}, zero-centered gradient penalty on interpolated samples (GAN-0GP) \cite{improveGeneralization}, and optimizers with momentum, can prevent CF and mode collapse. Finally, we introduce a new loss function that helps preventing CF while adding zero computational overhead.

\subsection*{Contributions:}
\begin{enumerate}
\item We detect the CF problem in GANs. 
\item We show the relationship between CF, mode collapse, and non-convergence.
\item We study the relationship between the sharpness of local maxima and mode collapse.
\item We show that CF tends to make the discriminator directionally monotonic around real datapoints.
\item We identify the causes of CF and explain the effectiveness of methods for preventing CF in GANs. %A new loss function with imbalanced weights for real and fake samples is proposed to prevent CF.
\end{enumerate}

%A large number of architectures and training techniques \cite{dcgan, improvedGAN, wgan, wgangp, progressiveGAN, styleGAN, whichGANConverge, spectralNorm, improveGeneralization} have been introduced to improve the convergence of GANs and to avoid mode collapse. The authors of each paper usually have their own explanation of how their method addresses the non-convergence and mode collapse problems. By analyzing the value surface of different GANs, we found that the methods in \cite{wgan, whichGANConverge, improveGeneralization} help the discriminator 

\section{Related works}
\textbf{Convergence.} Prior works on the convergence of GANs usually consider the convergence in parameter space \cite{ganStable, ttur, numericGAN, whichGANConverge}.
However, convergence in parameter space tells little about the quality of the equilibrium that a GAN converge to. For example,
\citeauthor{improveGeneralization} demonstrated that TTUR \cite{ttur} can make GAN converge to collapsed equilibrium.
Consensus Optimization \cite{numericGAN} can introduce spurious local equilibria with unknown properties to the game.

We directly study the behaviors of GANs in the data space. 
By analyzing the discriminator's output landscape, we find that when a GAN converges, real datapoints are local maxima of the discriminator. We discover the relationship between the sharpness of local maxima and mode collapse, generalization.

\textbf{Catastrophic forgetting.} \citeauthor{continualGAN} \cite{continualGAN} studied the standard continual learning setting in which a GAN is trained to generate samples from a set of distributions introduced sequentially. The problem is solved by the direct application of continual learning algorithms such as Elastic Weight Consolidation (EWC) \cite{ewc} to GANs. \citeauthor{ganIsCont} \cite{ganIsCont} independently came up with a similar intuition that GAN training is a continual learning problem.\footnote{Liang et al. came up with the idea a few months after us. They agreed that we are the first to consider the catastrophic forgetting problem in a single GAN. Their preprint has not been published at any conferences or journals.} The paper, however, did not study the causes and effects of the problem and focused on applying continual learning algorithms to address catastrophic forgetting in GANs. We focus on explaining the causes and effect of the problem and its relationship to mode collapse and non-convergence.

\section{Catastrophic forgetting problem in GANs}
\label{sec:catastrophic}
%\subsection{Preliminaries}
%\label{sec:prelim}
%Assume that fake datapoints are independent and gradient updates are applied on the datapoints directly: $\bm x^{t+1} = \bm x^t - \alpha g(\nabla_{\bm x^t} \mathcal{L}_G)$.

\begin{table*}
\centering
\begin{tabulary}{\linewidth}{|L|L|L|}
\hline
& $\mathcal{L}_D$ & $\mathcal{L}_G$ \\
\hline
WGANGP & $-\mathbb{E}_{\bm x \sim p_r}[D(\bm x)] $ $+ \mathbb{E}_{\bm z \sim p_z}[D(G(\bm z))] + \lambda \mathbb{E}_{\bm u}[(\norm{(\nabla D)_{\bm u}} - 1)^2]$ & $-\mathbb{E}_{\bm z \sim p_z}[D(G(\bm z))]$ \\
 & where $\bm u = \alpha \bm x + (1 - \alpha) \bm y; \bm x \sim p_x, \bm y \sim p_g, \alpha \sim \mathcal{U}(0, 1)$ & \\
\hline
GAN-NS & $\mathbb{E}_{\bm x \sim p_r}[-\log(D(\bm x))] + \mathbb{E}_{\bm z \sim p_z}[-\log(1 - D(G(\bm z)))]$ & $\mathbb{E}_{\bm z \sim p_z} [-\log(D(G(\bm z))) ] $ \\
\hline
GAN-R1 & $\mathbb{E}_{\bm x \sim p_r}[-\log(D(\bm x))] + \mathbb{E}_{\bm z \sim p_z}[-\log(1 - D(G(\bm z)))] + \lambda \mathbb{E}_{\bm x \sim p_r}[\norm{(\nabla D)_{\bm x}}^2]$
 & $\mathbb{E}_{\bm z \sim p_z} [-\log(D(G(\bm z))) ] $ \\
\hline
GAN-0GP & $\mathbb{E}_{\bm x \sim p_r}[-\log(D(\bm x))] + \mathbb{E}_{\bm z \sim p_z}[-\log(1 - D(G(\bm z)))] + \lambda \mathbb{E}_{\bm u}[\norm{(\nabla D)_{\bm u}}^2]$
%where 
%$\bm u = \alpha \bm x + (1 - \alpha) \bm y,\ \bm x \sim p_r, \bm y \sim p_g, \alpha \sim \mathcal{U}(0, 1)$.
 & $\mathbb{E}_{\bm z \sim p_z} [-\log(D(G(\bm z))) ] $  \\
 & where $\bm u = \alpha \bm x + (1 - \alpha) \bm y; \bm x \sim p_x, \bm y \sim p_g, \alpha \sim \mathcal{U}(0, 1)$ & \\
%
%where 
%
%$\bm u = \beta \bm x + (1 - \beta) \bm y, \beta \sim \mathcal{U}(0, 1), \bm x \sim p_r, \bm y \sim p_g$ & $\mathbb{E}_{\bm z \sim p_z} [-\log(D(G(\bm z))) ] $ \\
\hline
\end{tabulary}
\caption{Loss functions of GAN variants considered in this paper.}
\label{tab:loss}
\end{table*}

\subsection{GANs training as continual learning problems}
Let us consider a GAN with generator $G(\cdot; \bm \theta): \mathbb{R}^{d_z} \rightarrow \mathbb{R}^d$, a continuous function with parameter $\bm \theta \in \mathbb{R}^m$; and discriminator $D(\cdot; \bm \psi): \mathbb{R}^d \rightarrow \mathbb{R}$, a continuous function with parameter $\bm \psi \in \mathbb{R}^n$. 
$G$ transforms a $d_z$-dimensional noise distribution $p_z$ to a $d$-dimensional model distribution $p_g$ that approximates a $d$-dimensional target distribution $p_r$.
$D$ maps $d$-dimensional inputs to $1$-dimensional outputs.
Let $\mathcal{L}_D$ be the loss function for $D$, $\mathcal{L}_G$ be the loss function for $G$ (Table \ref{tab:loss}).
In practice, $G$ and $D$ are neural networks trained by alternating SGD \cite{gan}.

At each iteration of the training process, $G$ is updated to better fool $D$. 
$p_g^t$, the model distribution at iteration $t$, is different from the model distribution at the previous iteration $p_g^{t - 1}$ and the next iteration $p_g^{t + 1}$. %\footnote{We do not assume how many gradient updates are applied to the discriminator/generator in one GAN training iteration}. 
The knowledge required to separate $p_g^t$ from $p_r$ is different from that for the pair $\{p_g^{t - 1}, p_r\}$.
$\{p_g^{t - 1}, p_r\}$ and $\{p_g^{t}, p_r\}$ are two different classification tasks to the discriminator.\footnote{In the original theoretical formulation of GAN, at every GAN iteration, the discriminator and the generator are trained until convergence \cite{gan}. That means $p_g^t$ can be arbitrarily different from $p_g^{t-1}$. In practice, at each iteration, only a limited number of gradient updates are applied to the players. We can consider a chunk of consecutive model distributions as a task to the discriminator.}
The sequence of changing model distributions $\left\lbrace p_g^i \right\rbrace_{i = 1}^T$ and the target distribution $p_r$ form a sequence of tasks $\left\lbrace \mathcal{T}^i = \{ p_g^i, p_r \} \right\rbrace_{i = 1}^T$ to the discriminator.
%The discriminator is trained on a sequence of tasks $\left\lbrace \mathcal{T}^i = \{ p_g^i, p_r \} \right\rbrace_{i = 1}^T$ which consists of a sequence of changing distributions $\left\lbrace p_g^i \right\rbrace_{i = 1}^T$ and a fixed distribution $p_r$.  
Because the generator at iteration $t$, $G^t$, can only generate samples from $p_g^t$, $D^t$, the discriminator at iteration $t$, cannot access samples from previous model distributions $p_g^{<t}$. 
That makes the learning process of $D$ a continual learning problem.
Similarly, the generator has to fool a sequence of changing discriminators $\{ D^i \}_{i = 1}^T$. The training process of a GAN poses a different continual learning problem to each of the players. In this paper, we focus on the continual learning problem in the discriminator as many prior works have showed that the quality of a GAN mainly depends on its discriminator \cite{equiAndGeneralization, doGANLearnDist, disGenTradeoff}.

If the sequence $p_g^t$ converges to a distribution $p_g^*$, then the sequence of tasks $\left\lbrace \mathcal{T}^i \right\rbrace_{i = 1}^T$ converges to a single task of separating 2 distributions $p_g^*$ and $p_r$. 
In practice, however, the sequence of model distributions does not always converge. 
\citeauthor{ganStable} \cite{ganStable} formally proved that the players in Wasserstein GAN \cite{wgan} do not converge to an equilibrium but oscillate in a small cycle around the equilibrium. 
Although non-saturating GAN (GAN-NS) \cite{gan} was proven to be convergent under strong assumptions \cite{ganStable, ttur}, 
\citeauthor{manypatheq} \cite{manypatheq} observed that on many real world datasets, the distance between $p_g^t$ and $p_r$ (measured in  KL-divergence and Jensen-Shannon divergence) does not decrease as $t$ increases.
The authors suggested that $p_g$ can approach $p_r$ in many different and unpredictable ways. 
These results imply that in the most common variants of GANs, $p_g^t$ can be arbitrarily different from $p_g^{t - n}$ for large $n$. 
If the knowledge used for separating $p_g^t$ and $p_r$ cannot be used for separating $p_g^{t - n}$ and $p_r$, a discriminator trained on $\mathcal{T}^t$ could forget $\mathcal{T}^{t - n}$, i.e. it classifies samples in $\mathcal{T}^{t - n}$ wrongly (Fig. \ref{fig:gannsScore}).
When this happens, we say that the discriminator exhibits catastrophic forgetting behaviors. 
%We make the following observation. 
%When a network is trained on a sequence of tasks, older knowledge is overwritten more times than newer knowledge and thus is likely to be forgotten before newer knowledge. 
%The experiments in \cite{ewc} showed that the performance of a classifier on an older task is worse than that on a newer one.

%Because the capacity of the discriminator is finite, it has to forget old, unconsolidated knowledge to learn new things. 
%In order to avoid catastrophic forgetting, $\{ p_g^i \}_{i = 1}^T$ should evolve in a way such that a classifier that can separate $p_r$ and $p_g^t$ also can separate $p_r$ and $p_g^u, \forall \ u < t$. 

%In the next subsections, we demonstrate the catastrophic forgetting problem on a number of datasets and study its consequences.

\subsection{Catastrophic forgetting in GANs}
\label{sec:cfGAN}

%\begin{enumerate}
%\item Formulate the problem. Define catastrophic forgetting. Define the good/bad behavior: should not forget catastrophically, assign lower score to lower quality sample. If we assume that the quality of generated samples improve over time, the score of older samples should be lower than more recent ones. We consider that a discriminator has catastrophic forgetting if it assign higher score to lower quality samples.
%\item Demonstrate the problem on synthetic data and MNIST
%\item The similarity between distributions determines the level of forgetting. We want the discriminator to forget in the right way. We want the sequence of generator's distributions to converge to the target distribution, $d(p_g^t, p_r) > d(p_g^{t + 1}, p_r)$. Link to the claim in \cite{manypatheq}: GAN does not need to decrease a divergence at every step. This is plainly wrong if we want to learn the target distribution.
%\end{enumerate}

\begin{figure*}[ht!]
\centering
\subfloat[Iteration 3000]{\includegraphics[width=0.25\textwidth]{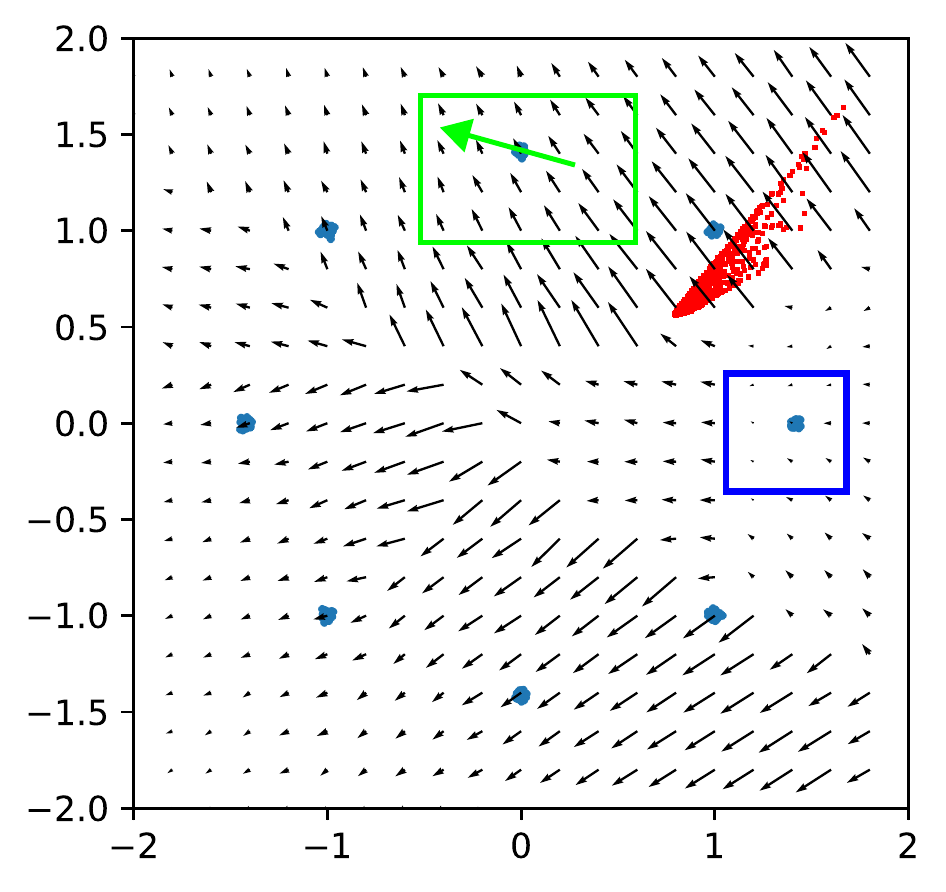}{\label{fig:8Gauss3000}}}
\subfloat[Iteration 3500]{\includegraphics[width=0.25\textwidth]{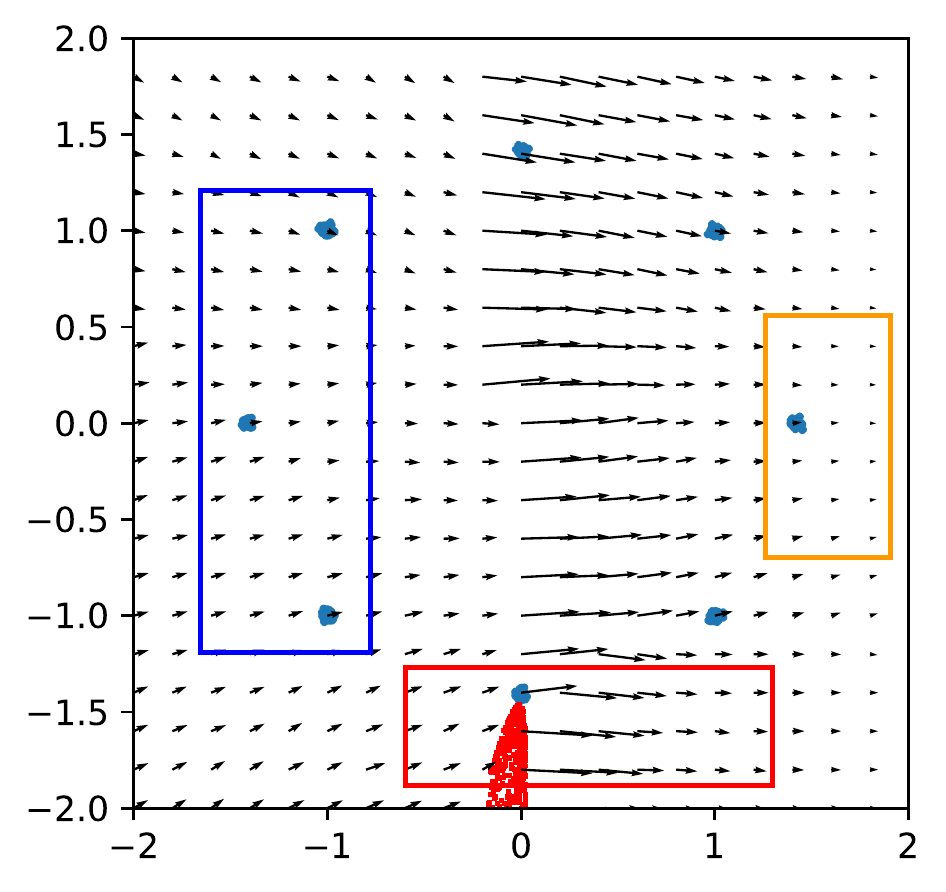}\label{fig:8Gauss3500}}
\subfloat[Iteration 3600]{\includegraphics[width=0.25\textwidth]{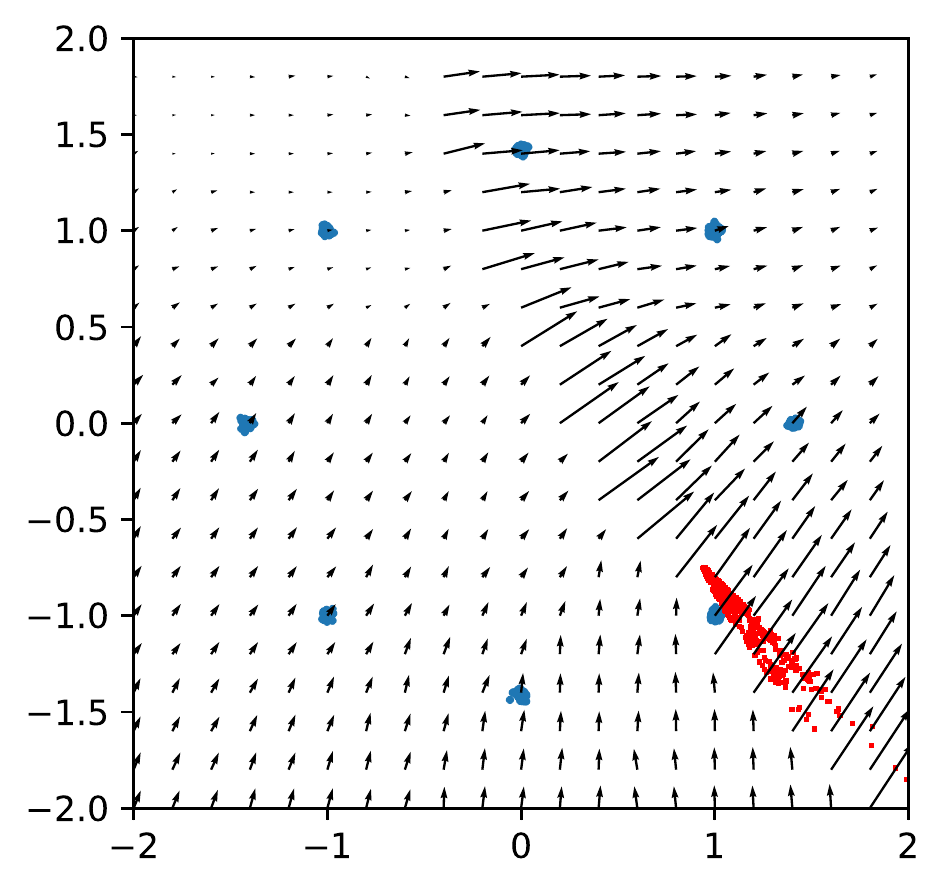}\label{fig:8Gauss3600}}
\subfloat[Iteration 20000]{\includegraphics[width=0.25\textwidth]{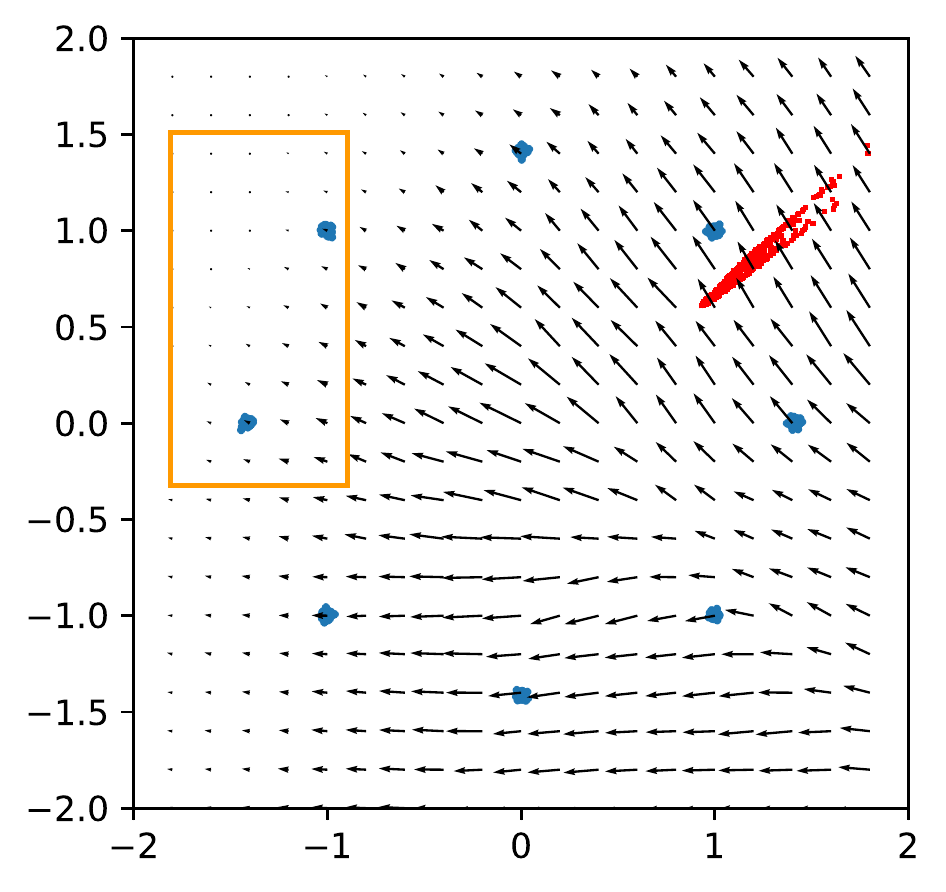}{\label{fig:8Gauss20000}}} 

%\subfloat[]{\includegraphics[width=0.25\textwidth]{figs/gan_8Gaussians_gradfield_center_1.00_alpha_None_lambda_0.00_lrg_0.00300_lrd_0.00300_nhidden_64_scale_2.00_optim_SGD_gnlayers_2_dnlayers_1_gradweight_0.1000_ncritic_1/fig_acc_03000.pdf}{\label{fig:8GaussAcc3000}}}
%\subfloat[]{\includegraphics[width=0.25\textwidth]{figs/gan_8Gaussians_gradfield_center_1.00_alpha_None_lambda_0.00_lrg_0.00300_lrd_0.00300_nhidden_64_scale_2.00_optim_SGD_gnlayers_2_dnlayers_1_gradweight_0.1000_ncritic_1/fig_acc_03500.pdf}{\label{fig:8GaussAcc3500}}}
%\subfloat[]{\includegraphics[width=0.25\textwidth]{figs/gan_8Gaussians_gradfield_center_1.00_alpha_None_lambda_0.00_lrg_0.00300_lrd_0.00300_nhidden_64_scale_2.00_optim_SGD_gnlayers_2_dnlayers_1_gradweight_0.1000_ncritic_1/fig_acc_20000_green.pdf}{\label{fig:8GaussAcc20000}}} \\
%\subfloat[]{\includegraphics[width=0.25\textwidth]{figs/continual_gan_8Gaussians_gradfield_center_0.00_alpha_None_lambda_0.00_lrg_0.00300_lrd_0.00300_nhidden_64_scale_2.00_optim_SGD_gnlayers_2_dnlayers_1_gradweight_0.0100_discount_0.9000_batch_1_ewcweight_10.00_ncritic_1/fig_05000.pdf}{\label{fig:8GaussCont5000}}}
%\subfloat[]{\includegraphics[width=0.25\textwidth]{figs/continual_gan_8Gaussians_gradfield_center_0.00_alpha_None_lambda_0.00_lrg_0.00300_lrd_0.00300_nhidden_64_scale_2.00_optim_SGD_gnlayers_2_dnlayers_1_gradweight_0.0100_discount_0.9000_batch_1_ewcweight_10.00_ncritic_1/fig_10000.pdf}{\label{fig:8GaussCont10000}}}

\subfloat[Iteration 1000]{\includegraphics[width=0.25\textwidth]{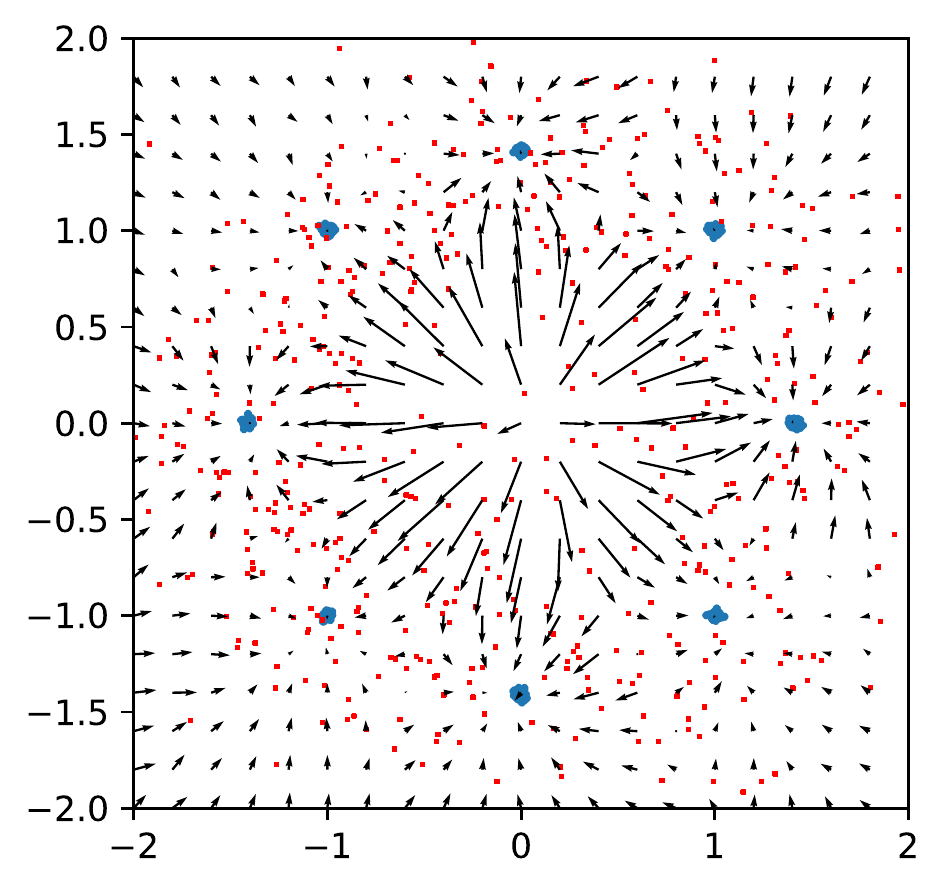}{\label{fig:8GaussR11000}}}
\subfloat[Iteration 2500]{\includegraphics[width=0.25\textwidth]{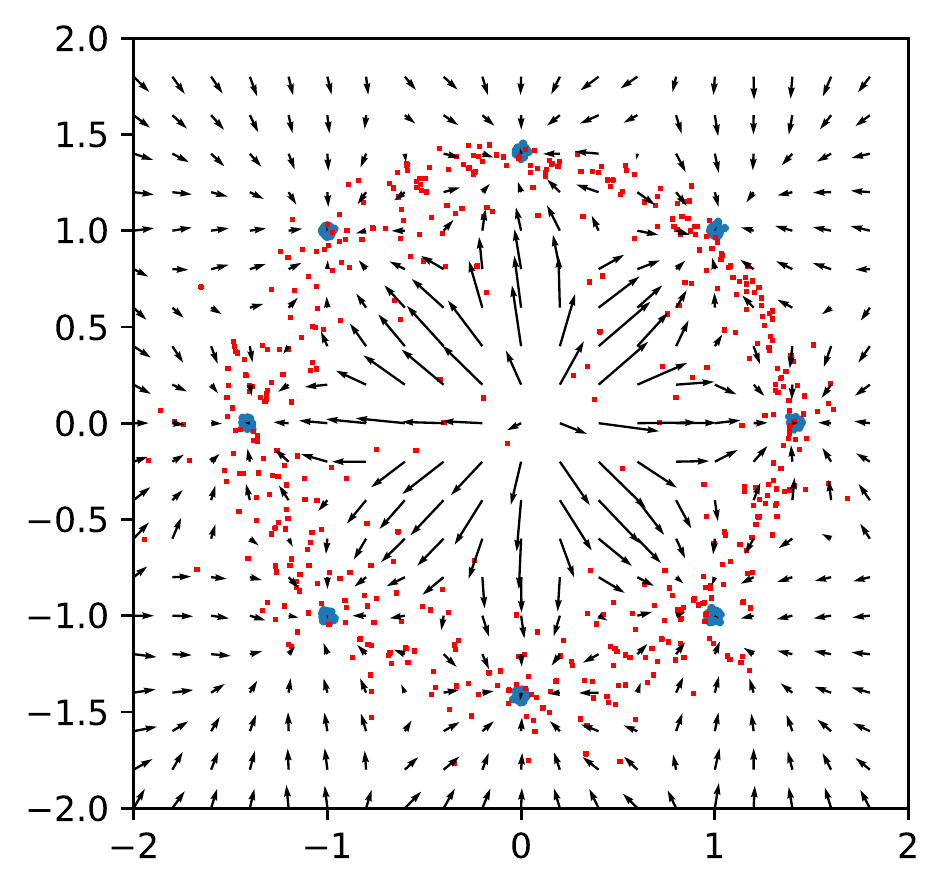}{\label{fig:8GaussR12500}}}
\subfloat[Iteration 5000]{\includegraphics[width=0.25\textwidth]{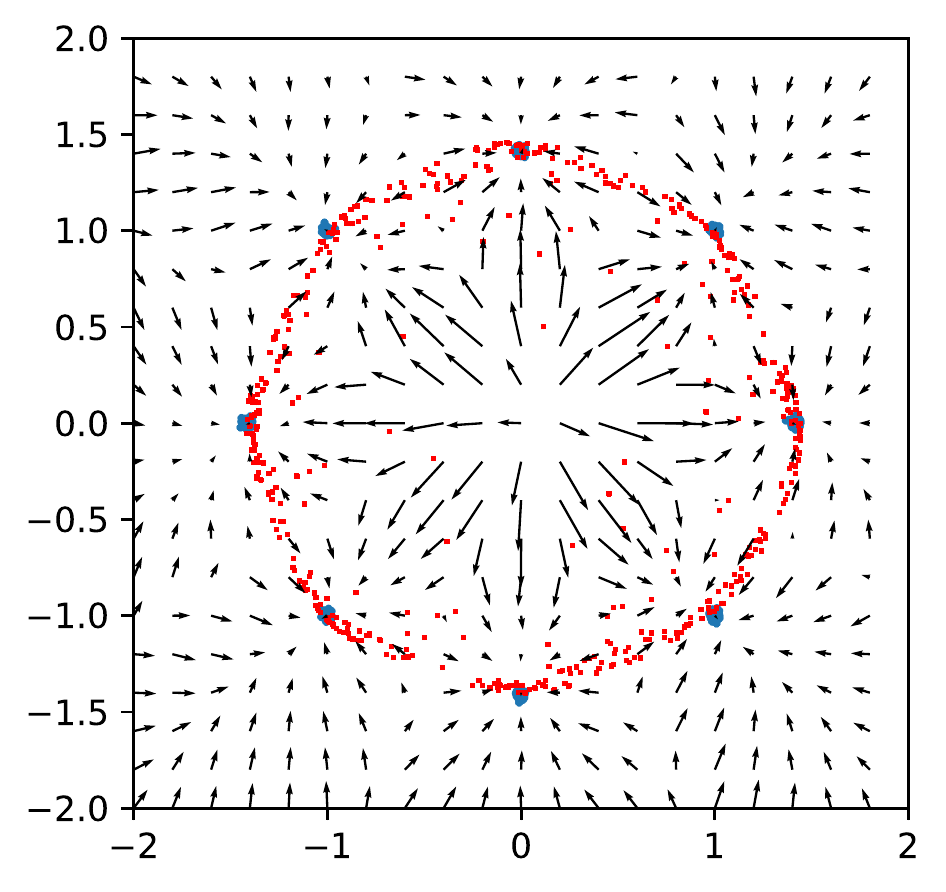}{\label{fig:8GaussR15000}}}
%\subfloat[Iteration 10000]{\includegraphics[width=0.25\textwidth]{/continual-gan-8Gaussians-gradfield-center-0.00-alpha-1.0-lambda-10.00-lrg-0.00300-lrd-0.00300-nhidden-512-scale-2.00-optim-SGD-gnlayers-2-dnlayers-2-gradweight-0.0100-discount-0.9900-batch-1-ewcweight-0.00-ncritic-1/fig-10000.pdf}{\label{fig:8GaussR110000}}} \\
%\subfloat[]{\includegraphics[width=0.25\textwidth]{figs/gan_8Gaussians_gradfield_center_0.00_alpha_None_lambda_0.00_lrg_0.00300_lrd_0.00300_nhidden_64_scale_2.00_optim_Adam_gnlayers_1_dnlayers_1_gradweight_0.0100_dweight_0.1000_gweight_0.0100_ncritic_1/fig_00400.pdf}{\label{fig:8GaussAdam400}}}
\subfloat[Adam. Iteration 1500]{\includegraphics[width=0.25\textwidth]{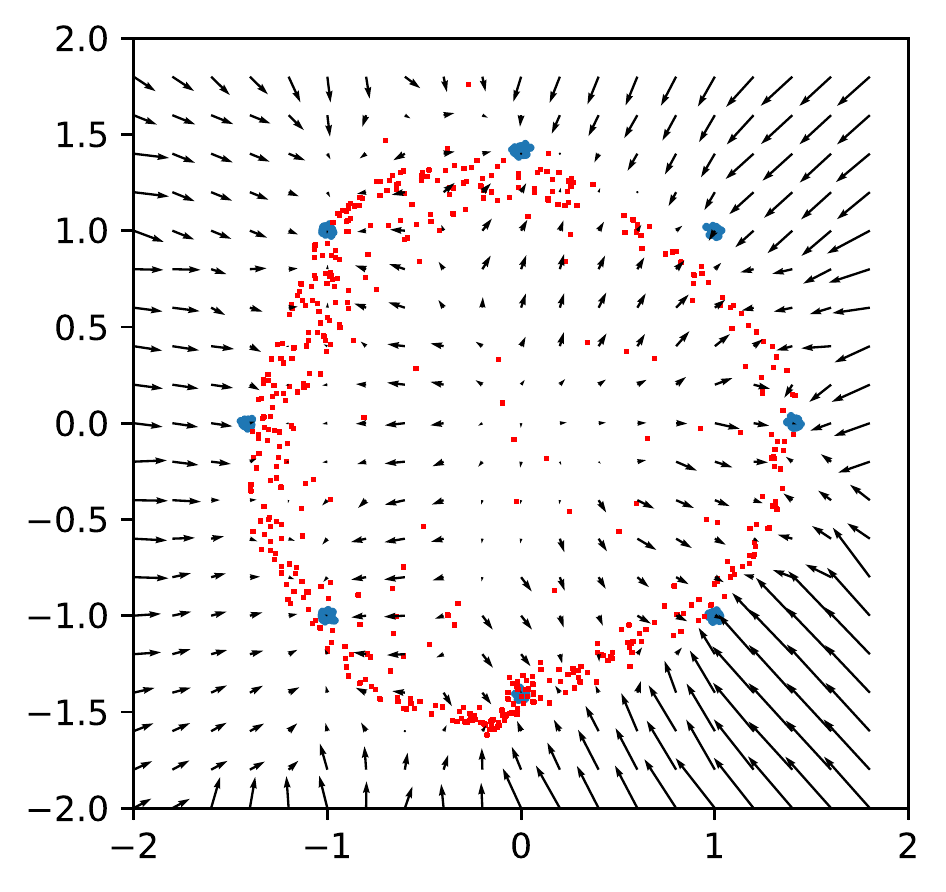}{\label{fig:8GaussAdam1500}}}
%\subfloat[]{\includegraphics[width=0.25\textwidth]{figs/gan_8Gaussians_gradfield_center_0.00_alpha_None_lambda_0.00_lrg_0.00300_lrd_0.00900_nhidden_512_scale_10.00_optim_SGD/fig_10000.pdf}{\label{fig:8GaussTTUR10000}}}
\caption{\protect\subref{fig:8Gauss3000} - \protect\subref{fig:8Gauss20000} catastrophic forgetting in GAN-NS trained on the 8 Gaussian dataset. 
%See Fig. \ref{figappx:8Gauss} in Appendix \ref{appx:ctSynthetic} for the full sequence.
\protect\subref{fig:8GaussR11000} - \protect\subref{fig:8GaussR15000} GAN-R1 with $\lambda = 10$. GAN-0GP and WGAN-GP exhibit similar behaviors on this dataset. \protect\subref{fig:8GaussAdam1500} GAN-NS trained with Adam. %See Fig. \ref{figappx:8GaussAdam} in Appendix \ref{appx:ctSynthetic} for the full sequence.
Viewing on computer is recommended.}
\label{fig:8Gauss}
\end{figure*}

%Because the capacity of a neural network is finite, it has to forget old, unconsolidated knowledge to learn new things. 
%In this section, we would like to differentiate between catastrophic forgetting and normal forgetting in GANs and derive some requirements for the sequence of model distributions.
% 

%In order to guide $p_g$ toward $p_r$, $D$ must assign higher scores to samples from $p_r$ and lower scores to samples from $p_g$.
%In practice, we only have access to a discrete set of real samples  $\mathcal{D}_r = \{ \bm x_1, ..., \bm x_N \}$ and $p_r$ is approximated by $\hat{p}_r$, the uniform distribution over $\mathcal{D}_r$. Any datapoint $\bm y \notin \mathcal{D}_r$ is labeled as fake.

\subsubsection{Catastrophic forgetting on synthetic dataset}
\label{sec:ct8Gauss}

We begin by analyzing the problem on the 8 Gaussian dataset, a dataset generated by a mixture of 8 Gaussians placed on a circle.
In Fig. \ref{fig:8Gauss}, red datapoints are generated samples, blue datapoints are real samples. 
The discriminator and generator are 2 hidden layer MLP with 64 hidden neurons.
ReLU activation function was used.
$p_z$ is a 2-dimensional standard normal distribution.
SGD with constant learning rate of $\alpha = 3 \times 10^{-3}$ was used for both networks.
The vector at a datapoint $\bm x$ shows the negative gradient $-\nicefrac{\partial\mathcal{L}_G}{\partial \bm x}$. 
The vector shows the direction in which $\mathcal{L}_G$ decreases the fastest. 
The length of the vector corresponds to the speed of change in $\mathcal{L}_G$. Because the gradient field is conservative, the the difference between the loss of two datapoints $\bm x_0$ and $\bm x_1$ is:
\begin{equation}
\mathcal{L}_G (\bm x_0) - \mathcal{L}_G (\bm x_1) = \int_{\mathcal{C}} \bm v \cdot d\bm s
\label{eqn:lineint}
\end{equation}
where $\bm v = -\nicefrac{\partial \mathcal{L}_G}{\partial \bm x}$ and $\mathcal{C}$ is a path from $\bm x_0$ to $\bm x_1$.
For the variants in Table \ref{tab:loss}, $\nicefrac{\partial\mathcal{L}_G}{\partial \bm x}$ only depends on $\bm x$ and $D$.
Because decreasing $\mathcal{L}_G$ in these GANs corresponds to increasing $D(\bm x)$, going in the direction of $-\nicefrac{\partial\mathcal{L}_G}{\partial \bm x}$ increases the score $D(\bm x)$.
Let $\bm y_0 = G(\bm z_0),\ \bm z_0 \sim p_z$ be a fake datapoint. 
%The gradient of $\mathcal{L}_G$ w.r.t. the generator's parameters $\bm \theta$ is:
%$-\nicefrac{\partial \mathcal{L}_G}{\partial \bm \theta} = -\nicefrac{\partial \bm y_0}{\partial \bm \theta} \times \nicefrac{\partial \mathcal{L}_G}{\partial \bm y_0}$.
Updating $\bm y_0$ with SGD with a small enough learning rate will move $\bm y_0$ in the direction of $-\nicefrac{\partial\mathcal{L}_G}{\partial \bm y_0}$ by a distance proportional to $\norm{-\nicefrac{\partial\mathcal{L}_G}{\partial \bm y_0}}$. %, the length of the vector associated with $\bm y_0$.
%\[ \bm y_0^{new} - \bm y_0 = G\left(\bm z_0; \bm \theta - \alpha \frac{\partial \mathcal{L}_G}{\partial \bm \theta} \right) - G(\bm z_0; \bm \theta) \]
%\[ \bm y_0^{new} \approx \bm y_0 - \alpha \frac{\partial\mathcal{L}_G}{\partial \bm y_0} \]
If the discriminator is fixed, then SGD updates will move $\bm y_0$ along its integral curve, in the direction of increasing $D(\bm y_0)$.\footnote{
In practice, gradient updates are not applied to $\bm y_0$ but to the generator's parameters. Because the generator also minimizes $\mathcal{L}_G$, gradient updates to the generator move $\bm y_0$ in a direction that approximates $-\nicefrac{\partial\mathcal{L}_G}{\partial \bm y_0}$. $-\nicefrac{\partial\mathcal{L}_G}{\partial \bm y_0}$ is a good approximation of the direction that $\bm y_0$ will move in the next iteration.}
%Because we want fake datapoints to converge to real datapoints, we want the integral curve from any fake datapoint to converge to a real datapoint.
 %Real datapoints should be attractors of the vector field.

Fig. \ref{fig:8Gauss3000} - \ref{fig:8Gauss20000} show the evolution of a GAN-NS on 8 Gaussian dataset. In Fig. \ref{fig:8Gauss3000} - \ref{fig:8Gauss3600}, the discriminator assigns higher score to datapoints that are further away from the fake datapoints, regardless of the true labels of these points.
This is shown by the gradient vectors pointing away from the fake datapoints.
The integral curves do not converge to any real datapoints.
If $D$ is fixed, updating $G$ with gradient descent makes $p_g$ diverges.
Because gradients w.r.t. different fake datapoints have the same direction, almost all of fake datapoints move in the same direction and do not spread out over the space. 
\textit{Because of CF, the generator is unable to break out of mode collapse.} %\textit{CF makes GANs prone to mode collapse.}

Inside the green box (Fig. \ref{fig:8Gauss3000}), gradients at all datapoints have {approximately the same direction}. 
The loss $\mathcal{L}_G$ decreases (the score $D(\cdot)$ increases) monotonically along the direction of the green vector $\bm u$, a random vector that points away from the fake datapoints.\footnote{
Graphically, we see that the angles between the green vector $\bm u$ and $\bm v = -\nicefrac{\partial \mathcal{L}_D}{\partial \bm x}$ are less than \ang{90} for all $\bm x$ in the box. Thus, the dot product $\bm v \cdot d\bm u$ is positive. The line integral in Eqn. \ref{eqn:lineint} is positive for $\bm x_0, \bm x_1$ in the box that satisfy $\bm x_1 = \bm x_0 + k\bm u,\ k > 0$. $\mathcal{L}_G$ monotonically decreases along the direction of $\bm u$. We say that $\mathcal{L}_G$ is monotonic in direction $\bm u$.}
We have the following observation: 
\begin{observation}
In a large neighborhood around a real datapoint, $\mathcal{L}_G$ (and therefore, $D(\cdot)$) is directionally monotonic.
\label{ob:monotonic}
\end{observation}
A theoretical explanation to this phenomenon is given in Sec. \ref{sec:ctEffect}.
Because fake samples in Fig. \ref{fig:8Gauss3000}-\ref{fig:8Gauss20000} are concentrated in a small region (i.e. mode collapse), $D$ can easily separate them from distant real samples and does not learn useful features of the real data. We say that $D$ \textit{catastrophically forgets} real samples that are far away from the current fake samples. 
\textit{Mode collapse and CF are interrelated, one problem makes the other more severe.}

%The same phenomenon is observed on MNIST (Fig. \ref{fig:gannsSGD}) and CIFAR-10 (Fig. \ref{fig:dcgannsReal} - \ref{fig:dcgannsFake20k}).

In Fig. \ref{fig:8Gauss3500}, fake datapoints on the right of the red box have higher scores than real datapoints on the left, although in Fig. \ref{fig:8Gauss3000}, these real datapoints have higher scores than these fake datapoints.
%The phenomenon suggests that $D$ overemphasizes on separating the current fake datapoints from near by real datapoints, ignoring older information.
Going from Fig. \ref{fig:8Gauss3000} to \ref{fig:8Gauss20000}, we observe that the vectors' directions change as soon as fake datapoints move.
The phenomenon suggests that \textit{information about previous model distributions is not preserved in the discriminator}. 
%In Fig. \ref{fig:8Gauss3000}, as the discriminator tries to separate $p_g^{3000}$ from $p_r$, it assigns lower scores to fake datapoints. 
%However, because the discriminator does not remember $p_g^{< t}$, it assigns higher score to datapoints that are further away from the fake datapoints.
%The discriminator continually directs the generator to move the model distribution to another region of the space, even if that region has been visited before. 
%That makes the model distribution rotates around the circle. 
%The discriminator and the generator fall into a learning loop and do not converge to an equilibrium.
%In the experiment in Fig. \ref{fig:8Gauss3000} - \ref{fig:8Gauss20000}, the loop continues for many cycles without breaking. The discriminator and generator at iteration 20000 are very similar to themselves at iteration 3000: the model distributions have the same shape and are located at the same location; the gradients around fake datapoints point in approximately the same direction.
As $D^t$ tries to separate $p_g^t$ from $p_r$, it assigns low scores to regions with fake samples and higher scores to other regions. Because $D^t$ does not 'remember' $p_g^{<t}$, it could assign high scores to regions previously occupied by $p_g^t$, i.e. $D^t$ could classify old fake samples as real. Fake samples at iteration 3000 (Fig. \ref{fig:8Gauss3000}) are classified as real by $D^{3500}$ (Fig. \ref{fig:8Gauss3500}). Similar behaviors are observed on MNIST (Fig. \ref{fig:gannsScore}). Because of forgetting, $D$ could direct $G$ to move to a region which $G$ has visited before. That could cause $G$ and $D$ to fall in a learning loop and do not converge to an equilibrium. In Fig. \ref{fig:8Gauss3000} - \ref{fig:8Gauss20000}, the model distribution rotates around the circle indefinitely. \textit{CF is a cause of non-convergence.}

\subsubsection{Catastrophic forgetting on image datasets}
\label{sec:ctMNIST}
We performed experiments on real world datasets to confirm the existence of CF in GANs. We visualize the landscape around a real datapoint $\bm x$ by plotting the output of the discriminator along a random line through $\bm x$. We choose a random unit vector $\hat{\bm u} \in \mathbb{R}^d, \norm{\hat{\bm u}} = 1$ and plot the value of the function 
\begin{equation}
f(k) = D(\bm x + k\hat{\bm u})
\label{eqn:fk}
\end{equation}
for $k \in [-100, 100]$. 
We use the same $\hat{\bm u}$ for all images in Fig. \ref{fig:gannsSGD}.
We choose to visualize $D(\cdot)$ instead of $\mathcal{L}_G$ because $\mathcal{L}_G$ explodes if $D(\cdot) \ll 1$.
The quality of the image $\bm x + k \hat{\bm u}$ decreases as $\abs{k}$ increases. % (Fig. \ref{figappx:realnoise} in Appendix \ref{appx:landscapes}). 
A good discriminator $D^*$ should assign lower scores to samples with lower quality. $D^*(\bm x)$ should be higher than $D^*(\bm x + k \hat{\bm u}),\ k > 0$, i.e. $\bm x$ is a local maximum of $D^*$.
If $\bm x$ is a local maximum of $D^*$, $f^*(k)$ must have a local maximum at $k = 0$ (the center of each subplot).
The result reported below was observed in all 10 different runs of the experiment. 

Fig. \ref{fig:gannsSGD} demonstrates the problem on MNIST. The generator and discriminator are 3 hidden layer MLPs with 512 hidden neurons. SGD with constant learning rate $\alpha = 3\times 10^{-4}$ was ued in training.
%Detailed configuration is given in Appendix \ref{appx:landscapes}.
%For a datapoint $\bm x$, we choose to visualize the value $D(\bm x)$ instead of the loss $\mathcal{L}_G$ because the non-saturating loss $\mathcal{L}_{G, NS} = \log(D(\bm x))$ explodes as $D(\bm x)$ approach $0$.

As shown in Fig. \ref{fig:gannsSGD}, the generated images keep changing from one shape to another, implying that the game does not converge to an equilibrium. 
In a large neighborhood around every real image, the discriminator's output is monotonic in the sampled direction. 
At iteration 100000, for every image, $f$ is a decreasing function (Fig. \ref{fig:gannsSGDFake100000}), while at iteration 200000, $f$ is an increasing function (Fig. \ref{fig:gannsSGDFake200000}).
More conretely, let $\nabla_{\hat{\bm u}}D^t(\bm x_0)$ be the discriminator's directional derivative along direction $\hat{\bm u}$ at $\bm x_0$ at iteration $t$. Then Fig. \ref{fig:gannsSGDFake100000} and \ref{fig:gannsSGDFake200000} shows that $\nabla_{\hat{\bm u}}D^{100000}(\bm x_0)$ and $\nabla_{\hat{\bm u}}D^{200000}(\bm x_0)$ for some $\bm x_0$ near the real datapoint $\bm x$, have opposite directions.
The knowledge of $D^{200000}$ (what $D^{200000}$ learned on $\{ p_g^{200000}, p_r \}$)  is not usable for $\{p_g^{100000}, p_r\}$.

We trained DCGAN \cite{dcgan} on CelebA \cite{celeba} and CIFAR-10 \cite{cifar10} to study the effect of network architecture and dataset complexity on the level of forgetting. 
Network architecture and hyper parameters are given in Table \ref{tab:dcganArch}. 

On CelebA, Fig. \ref{fig:celebaNSReal} - \ref{fig:celebaNSFake20k} show that CNN suffers less from CF than MLP. The discriminator in DCGAN-NS is not directional monotonic and it successfully makes many real datapoints its local maxima (see Sec. \ref{sec:landscape} for more). 
The discriminator can effectively discriminate real images from neighboring noisy images. The generator moves fake datapoints toward these local maxima and produces recognizable faces.

On CIFAR-10 (Fig. \ref{fig:dcgannsReal} - \ref{fig:dcgannsFake20k}), the discriminator cannot discriminate real images from noisy images. The function $f(k)$ in Fig. \ref{fig:dcgannsLand5k} is almost an increasing function while in Fig. \ref{fig:dcgannsLand20k} it is almost a decreasing function. 
The training does not converge as fake images change significantly as the learning progresses.

%Class labels in conditional GANs \cite{conditionalGAN} helps the discriminator by providing additional supervisory signals.

\textit{Conclusion}: GAN-NS trained on high dimensional datasets exhibits the same catastrophic forgetting behaviors as on toy datasets: 
(1) real datapoints are not local maxima of the discriminator or in more extreme cases, the discriminator is directionally monotonic in the neighborhoods of real datapoints; (2) the gradients w.r.t. datapoints in the neighborhood of a real datapoint change their directions significantly as fake datapoints move.
%Section \ref{sec:landscape} discusses the phenomenon in greater detail.

%When trained with Adam, an optimizer with momentum, the generator can generate high quality images (Fig. \ref{fig:ganns}). However, the generated images still change from one mode to another.

\begin{figure*}
\begin{flushright}

\subfloat[Real]{
\adjincludegraphics[width=0.22\textwidth, trim={0 {0.875\width} {0.5\width} {0\width}}, clip]{gan-mlp-mnist-nn512-nr3-ld0.0003-lg0.0003-nd1-ng1-gt0.0-gr0.0-grNone-orsgd-mm0.0-b10.5-b20.99-dsNone-ntGaussian-nm50-ns200001-be64-sv0.02-se1-ip0.01-idslerp-ne100.0-np1.0-ntTrue-ieFalse-/2019-10-06-00:33:33.098380/real}\label{fig:gannsSGDReal}}
\subfloat[Landscape 50000]{
\adjincludegraphics[width=0.22\textwidth, trim={0 {0.875\width} {0.5\width} {0\width}}, clip]{gan-mlp-mnist-nn512-nr3-ld0.0003-lg0.0003-nd1-ng1-gt0.0-gr0.0-grNone-orsgd-mm0.0-b10.5-b20.99-dsNone-ntGaussian-nm50-ns200001-be64-sv0.02-se1-ip0.01-idslerp-ne100.0-np1.0-ntTrue-ieFalse-/2019-10-06-00:33:33.098380/extrema-50000}
\label{fig:gannsSGDLandscape50000}} 
\subfloat[Landscape 100000]{
\adjincludegraphics[width=0.22\textwidth, trim={0 {0.875\width} {0.5\width} {0\width}}, clip]{gan-mlp-mnist-nn512-nr3-ld0.0003-lg0.0003-nd1-ng1-gt0.0-gr0.0-grNone-orsgd-mm0.0-b10.5-b20.99-dsNone-ntGaussian-nm50-ns200001-be64-sv0.02-se1-ip0.01-idslerp-ne100.0-np1.0-ntTrue-ieFalse-/2019-10-06-00:33:33.098380/extrema-100000}
\label{fig:gannsSGDLandscape100000}}
\subfloat[Landscape 200000]{
\adjincludegraphics[width=0.22\textwidth, trim={0 {0.875\width} {0.5\width} {0\width}}, clip]{gan-mlp-mnist-nn512-nr3-ld0.0003-lg0.0003-nd1-ng1-gt0.0-gr0.0-grNone-orsgd-mm0.0-b10.5-b20.99-dsNone-ntGaussian-nm50-ns200001-be64-sv0.02-se1-ip0.01-idslerp-ne100.0-np1.0-ntTrue-ieFalse-/2019-10-06-00:33:33.098380/extrema-200000}
\label{fig:gannsSGDLandscape200000}} \\

\subfloat[Generated 50000]{
\adjincludegraphics[width=0.22\textwidth, trim={0 {0.875\width} {0.5\width} {0\width}}, clip]{gan-mlp-mnist-nn512-nr3-ld0.0003-lg0.0003-nd1-ng1-gt0.0-gr0.0-grNone-orsgd-mm0.0-b10.5-b20.99-dsNone-ntGaussian-nm50-ns200001-be64-sv0.02-se1-ip0.01-idslerp-ne100.0-np1.0-ntTrue-ieFalse-/2019-10-06-00:33:33.098380/fake-50000}
\label{fig:gannsSGDFake50000}}
\subfloat[Generated 100000]{
\adjincludegraphics[width=0.22\textwidth, trim={0 {0.875\width} {0.5\width} {0\width}}, clip]{gan-mlp-mnist-nn512-nr3-ld0.0003-lg0.0003-nd1-ng1-gt0.0-gr0.0-grNone-orsgd-mm0.0-b10.5-b20.99-dsNone-ntGaussian-nm50-ns200001-be64-sv0.02-se1-ip0.01-idslerp-ne100.0-np1.0-ntTrue-ieFalse-/2019-10-06-00:33:33.098380/fake-100000}
\label{fig:gannsSGDFake100000}}
\subfloat[Generated 200000]{
\adjincludegraphics[width=0.22\textwidth, trim={0 {0.875\width} {0.5\width} {0\width}}, clip]{gan-mlp-mnist-nn512-nr3-ld0.0003-lg0.0003-nd1-ng1-gt0.0-gr0.0-grNone-orsgd-mm0.0-b10.5-b20.99-dsNone-ntGaussian-nm50-ns200001-be64-sv0.02-se1-ip0.01-idslerp-ne100.0-np1.0-ntTrue-ieFalse-/2019-10-06-00:33:33.098380/fake-200000}
\label{fig:gannsSGDFake200000}}
\end{flushright}
\centering
\caption{Catastrophic forgetting problem in GAN-NS trained with SGD.  \protect\subref{fig:gannsSGDReal} real datapoints from MNIST dataset. \protect\subref{fig:gannsSGDLandscape50000} - \protect\subref{fig:gannsSGDLandscape200000} the landscape around these real datapoints at different training iterations. In each subplot, the $X$-axis represent $k$, the $Y$-axis represent $D(\cdot)$. \protect\subref{fig:gannsSGDFake50000} - \protect\subref{fig:gannsSGDFake200000} generated data at different iterations. The same noise inputs were used for all iterations.}% See Fig. \ref{figappx:gannsSGD} in Appendix \ref{appx:landscapes} for the full figure.}
\label{fig:gannsSGD}
\end{figure*}

\begin{figure*}
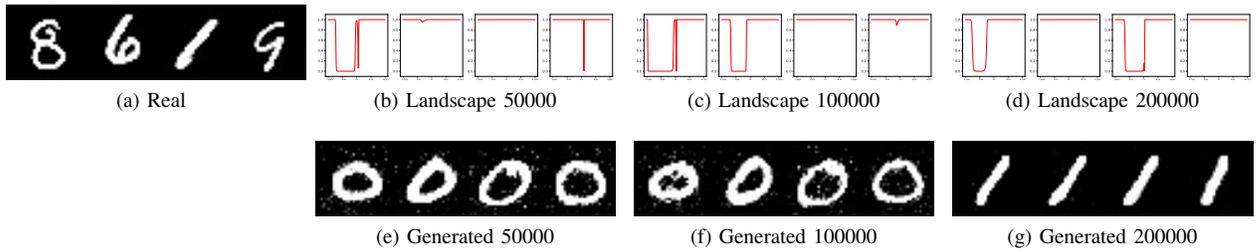

\begin{flushright}

\subfloat[Real]{
\adjincludegraphics[width=0.22\textwidth, trim={0 {0.875\width} {0.5\width} {0\width}}, clip]{gan-mlp-mnist-nn512-nr3-ld0.0003-lg0.0003-nd1-ng1-gt0.0-gr0.0-grNone-oradam-mm0.0-b10.5-b20.99-dsNone-ntGaussian-nm50-ns200001-be64-sv0.02-se1-ip0.01-idslerp-ne100.0-np1.0-ntTrue-ieFalse-/2019-10-06-00:11:15.651012/real}\label{fig:gannsAdamReal}}
\subfloat[Landscape 50000]{
\adjincludegraphics[width=0.22\textwidth, trim={0 {0.875\width} {0.5\width} {0\width}}, clip]{gan-mlp-mnist-nn512-nr3-ld0.0003-lg0.0003-nd1-ng1-gt0.0-gr0.0-grNone-oradam-mm0.0-b10.5-b20.99-dsNone-ntGaussian-nm50-ns200001-be64-sv0.02-se1-ip0.01-idslerp-ne100.0-np1.0-ntTrue-ieFalse-/2019-10-06-00:11:15.651012/extrema-50000}
\label{fig:gannsAdamLandscape50000}} 
\subfloat[Landscape 100000]{
\adjincludegraphics[width=0.22\textwidth, trim={0 {0.875\width} {0.5\width} {0\width}}, clip]{gan-mlp-mnist-nn512-nr3-ld0.0003-lg0.0003-nd1-ng1-gt0.0-gr0.0-grNone-oradam-mm0.0-b10.5-b20.99-dsNone-ntGaussian-nm50-ns200001-be64-sv0.02-se1-ip0.01-idslerp-ne100.0-np1.0-ntTrue-ieFalse-/2019-10-06-00:11:15.651012/extrema-100000}
\label{fig:gannsAdamLandscape100000}}
\subfloat[Landscape 200000]{
\adjincludegraphics[width=0.22\textwidth, trim={0 {0.875\width} {0.5\width} {0\width}}, clip]{gan-mlp-mnist-nn512-nr3-ld0.0003-lg0.0003-nd1-ng1-gt0.0-gr0.0-grNone-oradam-mm0.0-b10.5-b20.99-dsNone-ntGaussian-nm50-ns200001-be64-sv0.02-se1-ip0.01-idslerp-ne100.0-np1.0-ntTrue-ieFalse-/2019-10-06-00:11:15.651012/extrema-200000}
\label{fig:gannsAdamLandscape200000}} \\

\subfloat[Generated 50000]{
\adjincludegraphics[width=0.22\textwidth, trim={0 {0.875\width} {0.5\width} {0\width}}, clip]{gan-mlp-mnist-nn512-nr3-ld0.0003-lg0.0003-nd1-ng1-gt0.0-gr0.0-grNone-oradam-mm0.0-b10.5-b20.99-dsNone-ntGaussian-nm50-ns200001-be64-sv0.02-se1-ip0.01-idslerp-ne100.0-np1.0-ntTrue-ieFalse-/2019-10-06-00:11:15.651012/fake-50000}
\label{fig:gannsAdamFake50000}}
\subfloat[Generated 100000]{
\adjincludegraphics[width=0.22\textwidth, trim={0 {0.875\width} {0.5\width} {0\width}}, clip]{gan-mlp-mnist-nn512-nr3-ld0.0003-lg0.0003-nd1-ng1-gt0.0-gr0.0-grNone-oradam-mm0.0-b10.5-b20.99-dsNone-ntGaussian-nm50-ns200001-be64-sv0.02-se1-ip0.01-idslerp-ne100.0-np1.0-ntTrue-ieFalse-/2019-10-06-00:11:15.651012/fake-100000}
\label{fig:gannsAdamFake100000}}
\subfloat[Generated 200000]{
\adjincludegraphics[width=0.22\textwidth, trim={0 {0.875\width} {0.5\width} {0\width}}, clip]{gan-mlp-mnist-nn512-nr3-ld0.0003-lg0.0003-nd1-ng1-gt0.0-gr0.0-grNone-oradam-mm0.0-b10.5-b20.99-dsNone-ntGaussian-nm50-ns200001-be64-sv0.02-se1-ip0.01-idslerp-ne100.0-np1.0-ntTrue-ieFalse-/2019-10-06-00:11:15.651012/fake-200000}
\label{fig:gannsAdamFake200000}}
\end{flushright}
\centering
\caption{Output landscape and generated samples from GAN-NS + Adam.}
\label{fig:gannsAdam}
\end{figure*}

\begin{figure*}
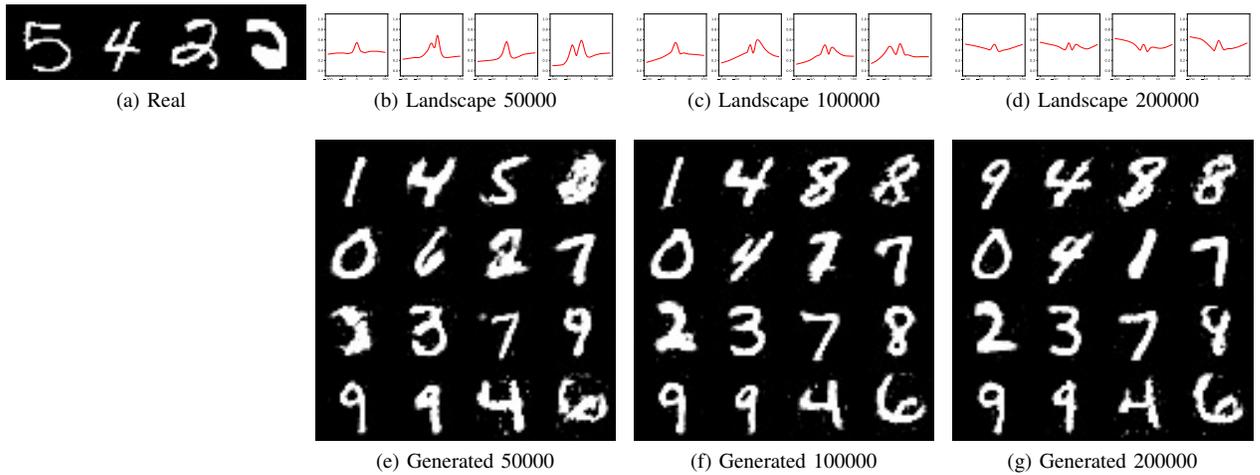

\begin{flushright}

\subfloat[Real]{
\adjincludegraphics[width=0.22\textwidth, trim={0 {0.875\width} {0.5\width} {0\width}}, clip]{gan-mlp-mnist-nn512-nr3-ld0.0003-lg0.0003-nd1-ng1-gt100.0-gr0.0-grNone-oradam-mm0.0-b10.5-b20.99-dsNone-ntGaussian-nm50-ns200001-be64-sv0.02-se1-ip0.01-idslerp-ne100.0-np1.0-ntTrue-ieFalse-/2019-10-06-10:56:30.392841/real}\label{fig:gan0gpReal}}
\subfloat[Landscape 50000]{
\adjincludegraphics[width=0.22\textwidth, trim={0 {0.875\width} {0.5\width} {0\width}}, clip]{gan-mlp-mnist-nn512-nr3-ld0.0003-lg0.0003-nd1-ng1-gt100.0-gr0.0-grNone-oradam-mm0.0-b10.5-b20.99-dsNone-ntGaussian-nm50-ns200001-be64-sv0.02-se1-ip0.01-idslerp-ne100.0-np1.0-ntTrue-ieFalse-/2019-10-06-10:56:30.392841/extrema-50000}
\label{fig:gan0gpLandscape50000}} 
\subfloat[Landscape 100000]{
\adjincludegraphics[width=0.22\textwidth, trim={0 {0.875\width} {0.5\width} {0\width}}, clip]{gan-mlp-mnist-nn512-nr3-ld0.0003-lg0.0003-nd1-ng1-gt100.0-gr0.0-grNone-oradam-mm0.0-b10.5-b20.99-dsNone-ntGaussian-nm50-ns200001-be64-sv0.02-se1-ip0.01-idslerp-ne100.0-np1.0-ntTrue-ieFalse-/2019-10-06-10:56:30.392841/extrema-100000}
\label{fig:gan0gpLandscape100000}}
\subfloat[Landscape 200000]{
\adjincludegraphics[width=0.22\textwidth, trim={0 {0.875\width} {0.5\width} {0\width}}, clip]{gan-mlp-mnist-nn512-nr3-ld0.0003-lg0.0003-nd1-ng1-gt100.0-gr0.0-grNone-oradam-mm0.0-b10.5-b20.99-dsNone-ntGaussian-nm50-ns200001-be64-sv0.02-se1-ip0.01-idslerp-ne100.0-np1.0-ntTrue-ieFalse-/2019-10-06-10:56:30.392841/extrema-200000}
\label{fig:gan0gpLandscape200000}} \\

\subfloat[Generated 50000]{
\adjincludegraphics[width=0.22\textwidth, trim={0 {0.5\width} {0.5\width} {0\width}}, clip]{gan-mlp-mnist-nn512-nr3-ld0.0003-lg0.0003-nd1-ng1-gt100.0-gr0.0-grNone-oradam-mm0.0-b10.5-b20.99-dsNone-ntGaussian-nm50-ns200001-be64-sv0.02-se1-ip0.01-idslerp-ne100.0-np1.0-ntTrue-ieFalse-/2019-10-06-10:56:30.392841/fake-50000}
\label{fig:gan0gpFake50000}}
\subfloat[Generated 100000]{
\adjincludegraphics[width=0.22\textwidth, trim={0 {0.5\width} {0.5\width} {0\width}}, clip]{gan-mlp-mnist-nn512-nr3-ld0.0003-lg0.0003-nd1-ng1-gt100.0-gr0.0-grNone-oradam-mm0.0-b10.5-b20.99-dsNone-ntGaussian-nm50-ns200001-be64-sv0.02-se1-ip0.01-idslerp-ne100.0-np1.0-ntTrue-ieFalse-/2019-10-06-10:56:30.392841/fake-100000}
\label{fig:gan0gpFake100000}}
\subfloat[Generated 200000]{
\adjincludegraphics[width=0.22\textwidth, trim={0 {0.5\width} {0.5\width} {0\width}}, clip]{gan-mlp-mnist-nn512-nr3-ld0.0003-lg0.0003-nd1-ng1-gt100.0-gr0.0-grNone-oradam-mm0.0-b10.5-b20.99-dsNone-ntGaussian-nm50-ns200001-be64-sv0.02-se1-ip0.01-idslerp-ne100.0-np1.0-ntTrue-ieFalse-/2019-10-06-10:56:30.392841/fake-200000}
\label{fig:gan0gpFake200000}}
\end{flushright}
\centering
\caption{Output landscape and generated samples from GAN-0GP with $\lambda = 100$.}% See Fig. \ref{figappx:gan0gp} in Appendix \ref{appx:landscapes} for the full figure.}
\label{fig:gan0gp}
\end{figure*}

\begin{figure*}
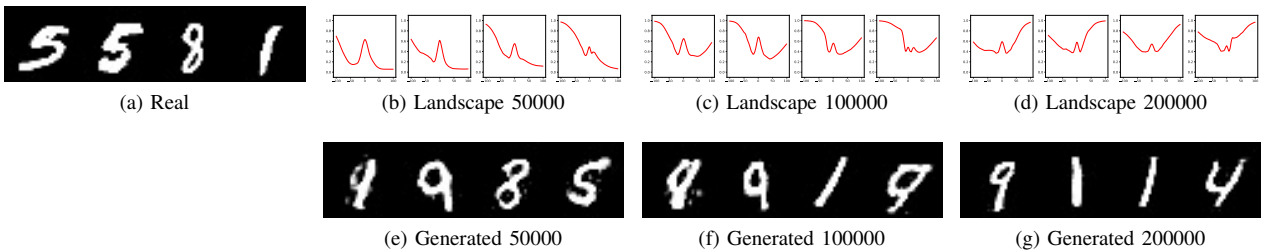

\begin{flushright}
\subfloat[Real]{
\adjincludegraphics[width=0.22\textwidth, trim={0 {0.875\width} {0.5\width} {0\width}}, clip]{gan-mlp-mnist-nn512-nr3-ld0.0003-lg0.0003-nd1-ng1-gt100.0-gr0.0-gr1.0-oradam-mm0.0-b10.5-b20.99-dsNone-ntGaussian-nm50-ns200001-be64-sv0.02-se1-ip0.01-idslerp-ne100.0-np1.0-ntTrue-ieFalse-/2019-10-05-23:14:33.524138/real}
}
\subfloat[Landscape 50000]{
\adjincludegraphics[width=0.22\textwidth, trim={0 {0.874\width} {0.5\width} {0\width}}, clip]{gan-mlp-mnist-nn512-nr3-ld0.0003-lg0.0003-nd1-ng1-gt100.0-gr0.0-gr1.0-oradam-mm0.0-b10.5-b20.99-dsNone-ntGaussian-nm50-ns200001-be64-sv0.02-se1-ip0.01-idslerp-ne100.0-np1.0-ntTrue-ieFalse-/2019-10-05-23:14:33.524138/extrema-50000}
}
\subfloat[Landscape 100000]{
\adjincludegraphics[width=0.22\textwidth, trim={0 {0.874\width} {0.5\width} {0\width}}, clip]{gan-mlp-mnist-nn512-nr3-ld0.0003-lg0.0003-nd1-ng1-gt100.0-gr0.0-gr1.0-oradam-mm0.0-b10.5-b20.99-dsNone-ntGaussian-nm50-ns200001-be64-sv0.02-se1-ip0.01-idslerp-ne100.0-np1.0-ntTrue-ieFalse-/2019-10-05-23:14:33.524138/extrema-100000}
}
\subfloat[Landscape 200000]{
\adjincludegraphics[width=0.22\textwidth, trim={0 {0.874\width} {0.5\width} {0\width}}, clip]{gan-mlp-mnist-nn512-nr3-ld0.0003-lg0.0003-nd1-ng1-gt100.0-gr0.0-gr1.0-oradam-mm0.0-b10.5-b20.99-dsNone-ntGaussian-nm50-ns200001-be64-sv0.02-se1-ip0.01-idslerp-ne100.0-np1.0-ntTrue-ieFalse-/2019-10-05-23:14:33.524138/extrema-200000}
}

\subfloat[Generated 50000]{
\adjincludegraphics[width=0.22\textwidth, trim={0 {0.875\width} {0.5\width} {0\width}}, clip]{gan-mlp-mnist-nn512-nr3-ld0.0003-lg0.0003-nd1-ng1-gt100.0-gr0.0-gr1.0-oradam-mm0.0-b10.5-b20.99-dsNone-ntGaussian-nm50-ns200001-be64-sv0.02-se1-ip0.01-idslerp-ne100.0-np1.0-ntTrue-ieFalse-/2019-10-05-23:14:33.524138/fake-50000}
}
\subfloat[Generated 100000]{
\adjincludegraphics[width=0.22\textwidth, trim={0 {0.875\width} {0.5\width} {0\width}}, clip]{gan-mlp-mnist-nn512-nr3-ld0.0003-lg0.0003-nd1-ng1-gt100.0-gr0.0-gr1.0-oradam-mm0.0-b10.5-b20.99-dsNone-ntGaussian-nm50-ns200001-be64-sv0.02-se1-ip0.01-idslerp-ne100.0-np1.0-ntTrue-ieFalse-/2019-10-05-23:14:33.524138/fake-100000}
}
\subfloat[Generated 200000]{
\adjincludegraphics[width=0.22\textwidth, trim={0 {0.875\width} {0.5\width} {0\width}}, clip]{gan-mlp-mnist-nn512-nr3-ld0.0003-lg0.0003-nd1-ng1-gt100.0-gr0.0-gr1.0-oradam-mm0.0-b10.5-b20.99-dsNone-ntGaussian-nm50-ns200001-be64-sv0.02-se1-ip0.01-idslerp-ne100.0-np1.0-ntTrue-ieFalse-/2019-10-05-23:14:33.524138/fake-200000}
}
\end{flushright}

\centering
\caption{Output landscape and generated samples from GAN-R1, $\lambda=100$.}% See Fig. \ref{figappx:ganr1} in Appendix \ref{appx:landscapes} for the full figure.}
\label{fig:ganr1}
\end{figure*}

\begin{figure*}
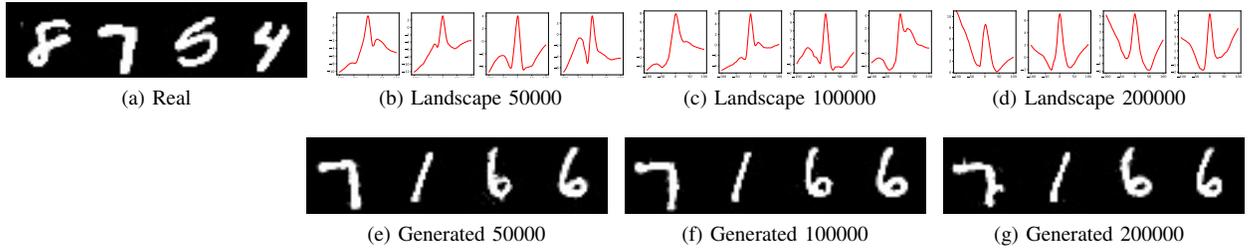

\centering
\begin{flushright}
\subfloat[Real]{
\adjincludegraphics[width=0.22\textwidth, trim={0 {0.875\width} {0.5\width} {0\width}}, clip]{wgan-mlp-mnist-nn512-nr3-ld0.0003-lg0.0003-nd5-ng1-gt10.0-gr1.0-grNone-oradam-mm0.0-b10.5-b20.99-dsNone-ntGaussian-nm50-ns200001-be64-sv0.02-se1-ip0.01-idslerp-ne100.0-np1.0-ntTrue-ieFalse-/2019-10-06-00:52:29.276259/real}
}
\subfloat[Landscape 50000]{
\adjincludegraphics[width=0.22\textwidth, trim={0 {0.874\width} {0.5\width} {0\width}}, clip]{wgan-mlp-mnist-nn512-nr3-ld0.0003-lg0.0003-nd5-ng1-gt10.0-gr1.0-grNone-oradam-mm0.0-b10.5-b20.99-dsNone-ntGaussian-nm50-ns200001-be64-sv0.02-se1-ip0.01-idslerp-ne100.0-np1.0-ntTrue-ieFalse-/2019-10-06-00:52:29.276259/extrema-50000.pdf}}
\subfloat[Landscape 100000]{
\adjincludegraphics[width=0.22\textwidth, trim={0 {0.874\width} {0.5\width} {0\width}}, clip]{wgan-mlp-mnist-nn512-nr3-ld0.0003-lg0.0003-nd5-ng1-gt10.0-gr1.0-grNone-oradam-mm0.0-b10.5-b20.99-dsNone-ntGaussian-nm50-ns200001-be64-sv0.02-se1-ip0.01-idslerp-ne100.0-np1.0-ntTrue-ieFalse-/2019-10-06-00:52:29.276259/extrema-100000.pdf}}
\subfloat[Landscape 200000]{
\adjincludegraphics[width=0.22\textwidth, trim={0 {0.874\width} {0.5\width} {0\width}}, clip]{wgan-mlp-mnist-nn512-nr3-ld0.0003-lg0.0003-nd5-ng1-gt10.0-gr1.0-grNone-oradam-mm0.0-b10.5-b20.99-dsNone-ntGaussian-nm50-ns200001-be64-sv0.02-se1-ip0.01-idslerp-ne100.0-np1.0-ntTrue-ieFalse-/2019-10-06-00:52:29.276259/extrema-200000.pdf}}

\subfloat[Generated 50000]{
\adjincludegraphics[width=0.22\textwidth, trim={0 {0.875\width} {0.5\width} {0\width}}, clip]{wgan-mlp-mnist-nn512-nr3-ld0.0003-lg0.0003-nd5-ng1-gt10.0-gr1.0-grNone-oradam-mm0.0-b10.5-b20.99-dsNone-ntGaussian-nm50-ns200001-be64-sv0.02-se1-ip0.01-idslerp-ne100.0-np1.0-ntTrue-ieFalse-/2019-10-06-00:52:29.276259/fake-50000}
}
\subfloat[Generated 100000]{
\adjincludegraphics[width=0.22\textwidth, trim={0 {0.875\width} {0.5\width} {0\width}}, clip]{wgan-mlp-mnist-nn512-nr3-ld0.0003-lg0.0003-nd5-ng1-gt10.0-gr1.0-grNone-oradam-mm0.0-b10.5-b20.99-dsNone-ntGaussian-nm50-ns200001-be64-sv0.02-se1-ip0.01-idslerp-ne100.0-np1.0-ntTrue-ieFalse-/2019-10-06-00:52:29.276259/fake-100000}
}
\subfloat[Generated 200000]{
\adjincludegraphics[width=0.22\textwidth, trim={0 {0.875\width} {0.5\width} {0\width}}, clip]{wgan-mlp-mnist-nn512-nr3-ld0.0003-lg0.0003-nd5-ng1-gt10.0-gr1.0-grNone-oradam-mm0.0-b10.5-b20.99-dsNone-ntGaussian-nm50-ns200001-be64-sv0.02-se1-ip0.01-idslerp-ne100.0-np1.0-ntTrue-ieFalse-/2019-10-06-00:52:29.276259/fake-200000}}
\end{flushright}
\caption{Output landscape and generated samples from WGAN-GP, $\lambda=10$, 5 discriminator updates per 1 generator update.}%See Fig. \ref{figappx:wgangp} in Appendix \ref{appx:landscapes} for the full figure.}
\label{fig:wgangp}
\end{figure*}

\subsubsection{The causes of Catastrophic Forgetting}
Based on the above experiments, we identified two reasons for CF:
\begin{enumerate}
\item \textit{Information from previous tasks is not carried to/used for the current task.} SGD does not use information from previous model distributions, $p_g^{<t}$. At iteration $t$, SGD update for the discriminator is computed from samples from $p_g^t$ and $p_r$ only. Because information from $p_g^{<t}$ is not used in training, the discriminator forgets $p_g^{< t}$, i.e. it does not assign low score to samples from $p_g^{<t}$. 
\item \textit{The current task is significantly different from previous tasks so the knowledge of the current task cannot be used for previous tasks and vice versa.} 
As old knowledge is overwritten by new knowledge, optimizing the discriminator on the current task will degrade its performance on older tasks. 
%In Fig. \ref{fig:8Gauss3000}, $D^{3000}$ can separate $p_g^{3000}$ from $p_r$ by assigning lower score to samples from $p_g^{3000}$ and higher scores to samples from $p_r$. However, $D^{3000}$ does not perform well on the pair $\{p_g^{3500}, p_r\}$. The reverse is true for $D^{3500}$ in Fig. \ref{fig:8Gauss3500}. %Optimizing the performance of the discriminator on $\mathcal{T}^{3500} = \{p_g^{3500}, p_r\}$ will damage the performance on $\mathcal{T}^{3000} = \{p_g^{3000}, p_r\}$.
\end{enumerate}
Methods for preventing CF is studied in Section \ref{sec:methods}.

\section{The output landscape}
\label{sec:landscape}

\subsection{The evolution of the landscape}

We apply the visualization technique in Section \ref{sec:ctMNIST} to other variants of GAN. 
We reuse the network architecture and learning rate from the experiment in Fig. \ref{fig:gannsSGD}. We replace SGD with Adam with $\beta_1 = 0.5, \beta_2 = 0.99$. 
%Details about the configuration are given in Appendix \ref{appx:landscapes}. 
We run each experiment 10 times with different random seeds and report results that are consistent between different runs. 
The evolution of the landscape and generated samples of GAN-NS, GAN-0GP with $\lambda = 100$, GAN-R1 with $\lambda = 100$, and WGAN-GP with $\lambda = 10$ are shown in Fig. \ref{fig:gannsAdam}, \ref{fig:gan0gp}, \ref{fig:ganr1}, and \ref{fig:wgangp} respectively. %The result for GAN-NS trained with Adam is shown in Fig. \ref{figappx:ganns} in Appendix \ref{appx:landscapes}.

GAN-0GP, GAN-R1, and WGAN-GP have significantly better sample quality and diversity than GAN-NS. GAN-NS does not exhibit good convergence behavior: the digit in a image changes from one digit to another as the training progresses (Fig. \ref{fig:gannsAdam}).\footnote{Note that this does not contradict the statement in \cite{ganStable} that GAN-NS converge to an equilibrium. Many of the assumptions in that paper is not satisfied in practice, e.g. the learning rate is not decayed toward 0.}  
GAN-0GP, GAN-R1, and WGAN-GP exhibit better convergence behaviors: for many images, the digits stay the same during training.  

We observe that throughout the training process of GAN-0GP, GAN-R1, and WGAN-GP, for every real datapoint, the function $f(k)$ always has a local maximum at $k = 0$, implying that real datapoints are local maxima of the discriminator. 
This can also be seen in GAN-R1 trained on the 8 Gaussian dataset (Fig. \ref{fig:8GaussR11000} - \ref{fig:8GaussR15000}): the gradients w.r.t. datapoints in the neighborhood of a real datapoint point toward that real datapoint  (GAN-0GP and WGANGP exhibit the same behaviors).  
%In GAN-0GP, GAN-R1, and WGAN-GP, real datapoints are attractors of the vector field.
If a fake datapoint is in the basin of attraction of a real datapoint and gradient updates are  applied directly on the fake datapoint, it will be attracted toward the real datapoint.
%Mode collapse will be alleviated if different fake datapoints are attracted toward different real datapoints.
Different attractors (local maxima) at different regions of the data space attract different fake datapoints toward different directions, spreading fake datapoints over the space, effectively reducing mode collapse. 

Fig. \ref{fig:gan0gp10} shows that GAN-0GP with $\lambda = 10$ suffers from mild mode collapse.\footnote{This is consistent with the analysis by the authors of GAN-0GP. \citeauthor{improveGeneralization} claimed that larger $\lambda$ leads to better generalization but may slow down the training.} 
The maxima in Fig. \ref{fig:gan0gp10} are much sharper than those in Fig. \ref{fig:wgangp}.
The discriminator overfits to the real training datapoints and forces the scores of near by datapoints to be close to 0. 
That creates many flat regions where the gradients of the discriminator w.r.t. datapoints in these regions are vanishingly small. 
A fake datapoint located in a flat region cannot move toward the real datapoint because the gradient is vanishingly small.
Real datapoints in Fig. \ref{fig:gan0gp10} have small basin of attraction and cannot effectively spread fake samples over the space.
%It is very hard for a gradient based generator to get pass these flat regions to discover the real datapoints.
The diversity of generated samples is thus reduced, making mode collapse visible. 
In order to attract fake datapoints toward different directions, \textit{local maxima should be wide}, i.e. they should have large basin of attraction.

The landscapes of GAN-NS  in Fig. \ref{fig:gannsSGD} and \ref{fig:gannsAdam} contain many flat regions where the scores $D(\cdot)$ are very close to 1 or 0. 
The same problem is seen on the 8 Gaussian dataset (datapoints in the orange and blue boxes in Fig. \ref{fig:8Gauss3000}-\ref{fig:8Gauss20000} have scores close to 1 and 0, respectively). 
%Datapoints in the orange boxes in Fig. \ref{fig:8Gauss3500} and \ref{fig:8Gauss20000} have scores close to 1. 
%Datapoints in the blue boxes in Fig. \ref{fig:8Gauss3000} and \ref{fig:8Gauss3500} have scores close to 0. 
However, unlike Fig. \ref{fig:gan0gp10}, the real datapoints in Fig. \ref{fig:8Gauss3000} - \ref{fig:8Gauss20000}, \ref{fig:gannsSGD}, and \ref{fig:gannsAdam} are not local maxima.
The discriminator in GAN-NS underfits the data.
%In the next subsection, we explain the formation of these flat regions as a consequence of catastrophic forgetting.

CNN based discriminators do not create flat regions in the output landscape (Fig. \ref{fig:celebaNSLand5k}-\ref{fig:celebaNSLand20k} and  \ref{fig:dcgannsLand5k}-\ref{fig:dcgannsLand20k}). 
However, when the dataset is more complicated, DCGAN-NS discriminator fails to make real datapoints local maxima and the training does not converge (Fig. \ref{fig:dcgannsReal}-\ref{fig:dcgannsFake20k}).
The discriminator underfits the data because it is not powerful enough to learn features that separate real and fake/noisy samples. 
More powerful discriminators based on ResNet \cite{resnet} significantly improve the quality of GANs (e.g. \cite{progressiveGAN}).
We make the following observation: 
\begin{observation}
For a GAN to converge to a good local equilibrium, real datapoints should be wide local maxima of the discriminator.
\label{ob:wideMaxima}
\end{observation}

\begin{figure}[ht!]
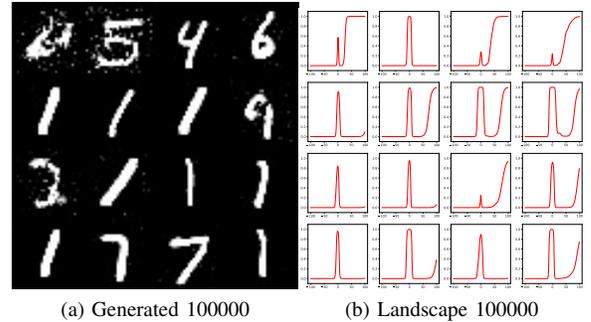

\centering
\subfloat[Generated 100000]{\adjincludegraphics[width=0.21\textwidth, trim={{0.5\width} {0.5\width} {0\width} {0\width}}, clip]{gan-mlp-mnist-nn512-nr3-ld0.0003-lg0.0003-nd1-ng1-gt10.0-gr0.0-grNone-oradam-mm0.0-b10.5-b20.99-dsNone-ntGaussian-nm50-ns200001-be64-sv0.02-se1-ip0.01-idslerp-ne100.0-np1.0-ntTrue-ieFalse-/2019-10-06-12:02:13.865053/fake-100000}}
\subfloat[Landscape 100000]{\adjincludegraphics[width=0.21\textwidth, trim={{0.5\width} {0.5\width} {0\width} {0\width}}, clip]{gan-mlp-mnist-nn512-nr3-ld0.0003-lg0.0003-nd1-ng1-gt10.0-gr0.0-grNone-oradam-mm0.0-b10.5-b20.99-dsNone-ntGaussian-nm50-ns200001-be64-sv0.02-se1-ip0.01-idslerp-ne100.0-np1.0-ntTrue-ieFalse-/2019-10-06-12:02:13.865053/extrema-100000.pdf}}
\caption{Mode collapse without CF in GAN-0GP, $\lambda = 10$.}
\label{fig:gan0gp10}
\end{figure}

\subsection{The effect of catastrophic forgetting on the landscape}
\label{sec:ctEffect}
\begin{figure}[ht!]
\centering
\subfloat[Iter. 0]{\includegraphics[width=0.08\textwidth]{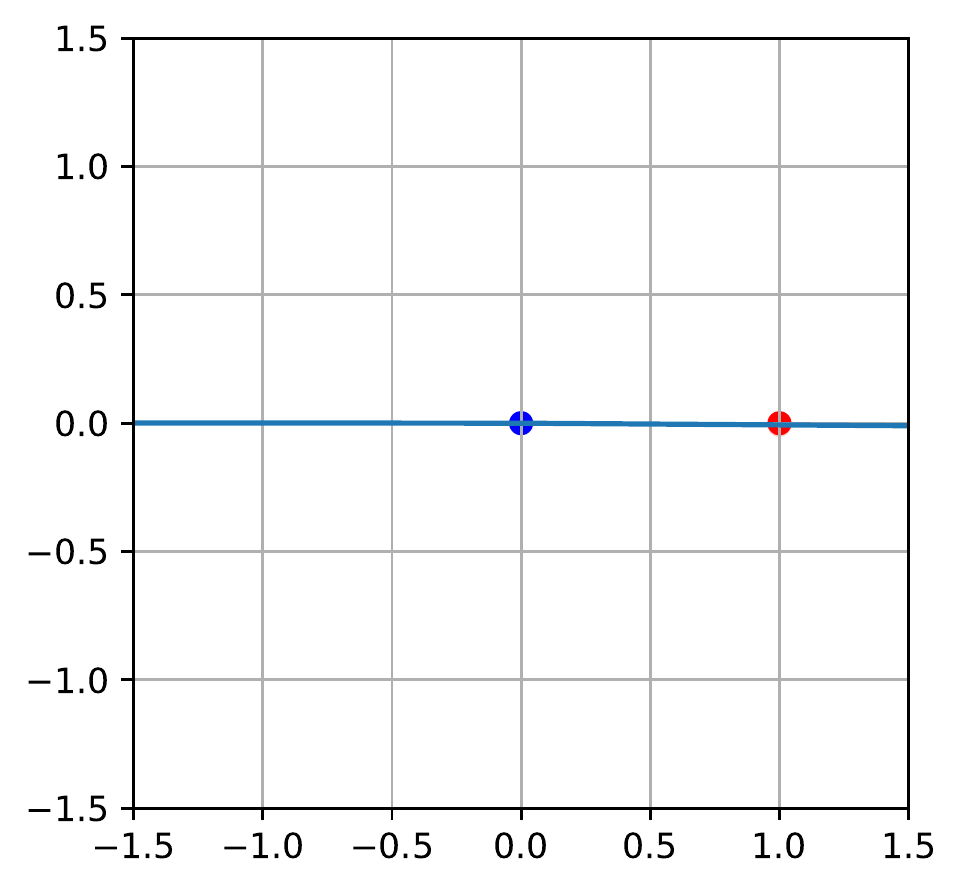}\label{fig:dirac1SampleLow0}}
\subfloat[Iter. 10]{\includegraphics[width=0.08\textwidth]{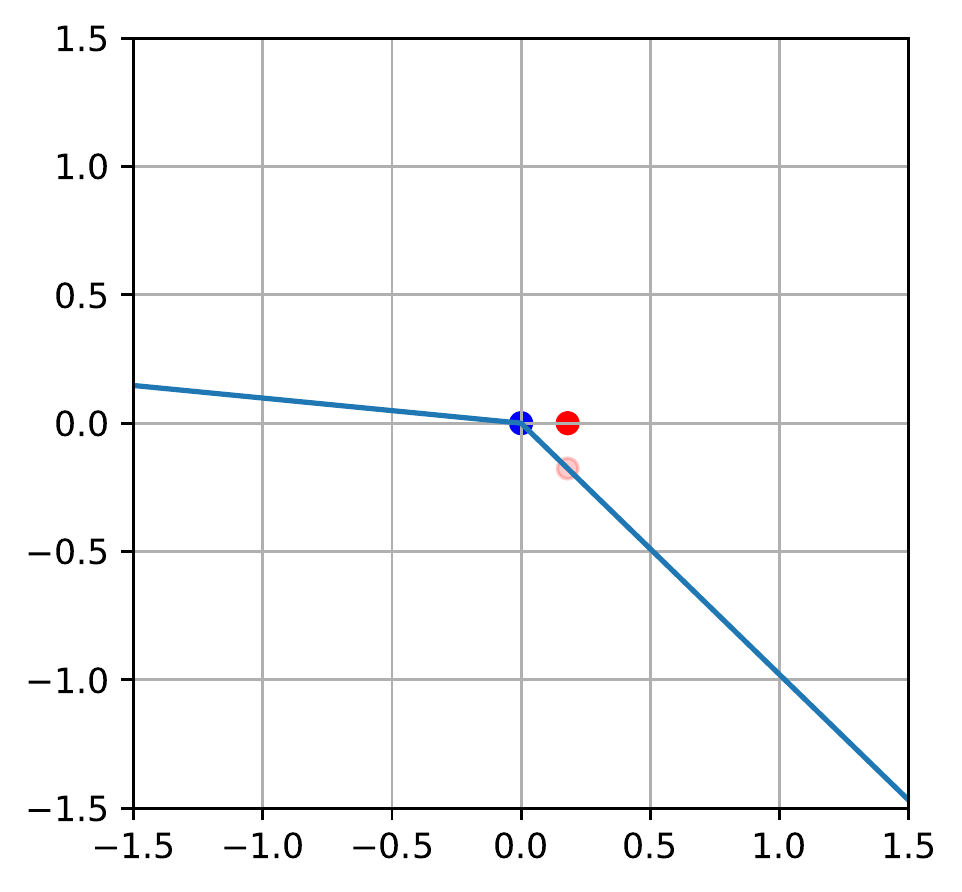}\label{fig:dirac1SampleLow10}}
\subfloat[Iter. 100]{\includegraphics[width=0.08\textwidth]{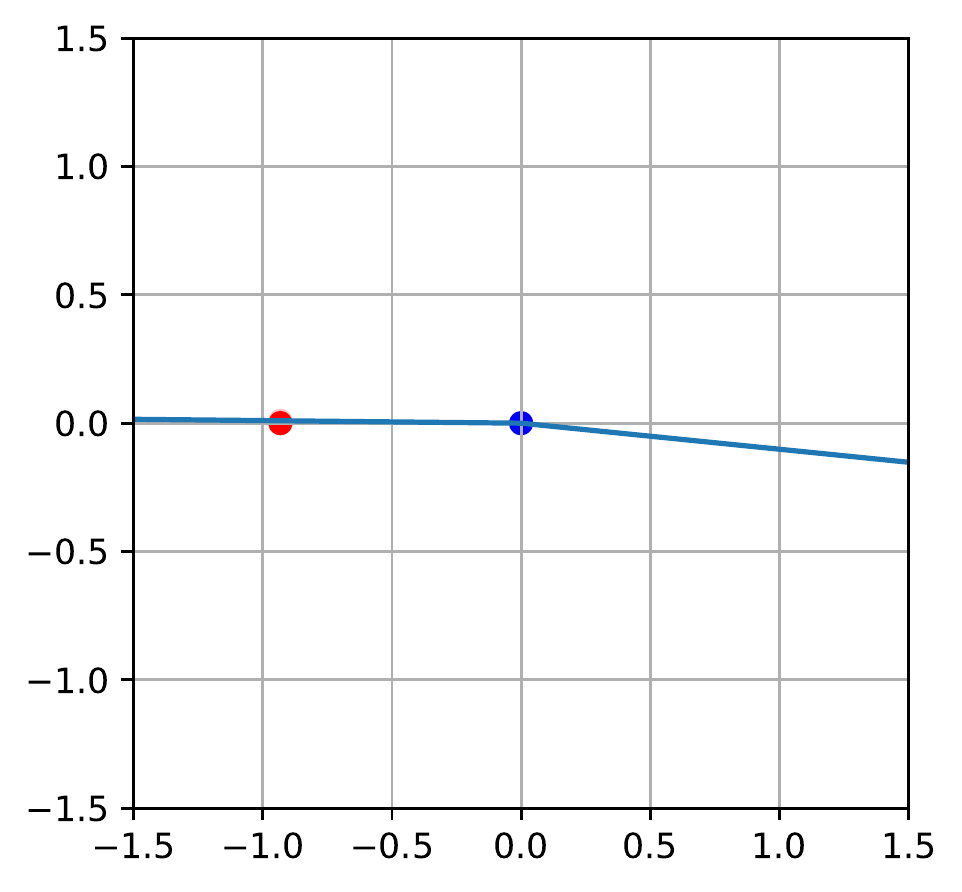}\label{fig:dirac1SampleLow100}}
\subfloat[Iter. 200]{\includegraphics[width=0.08\textwidth]{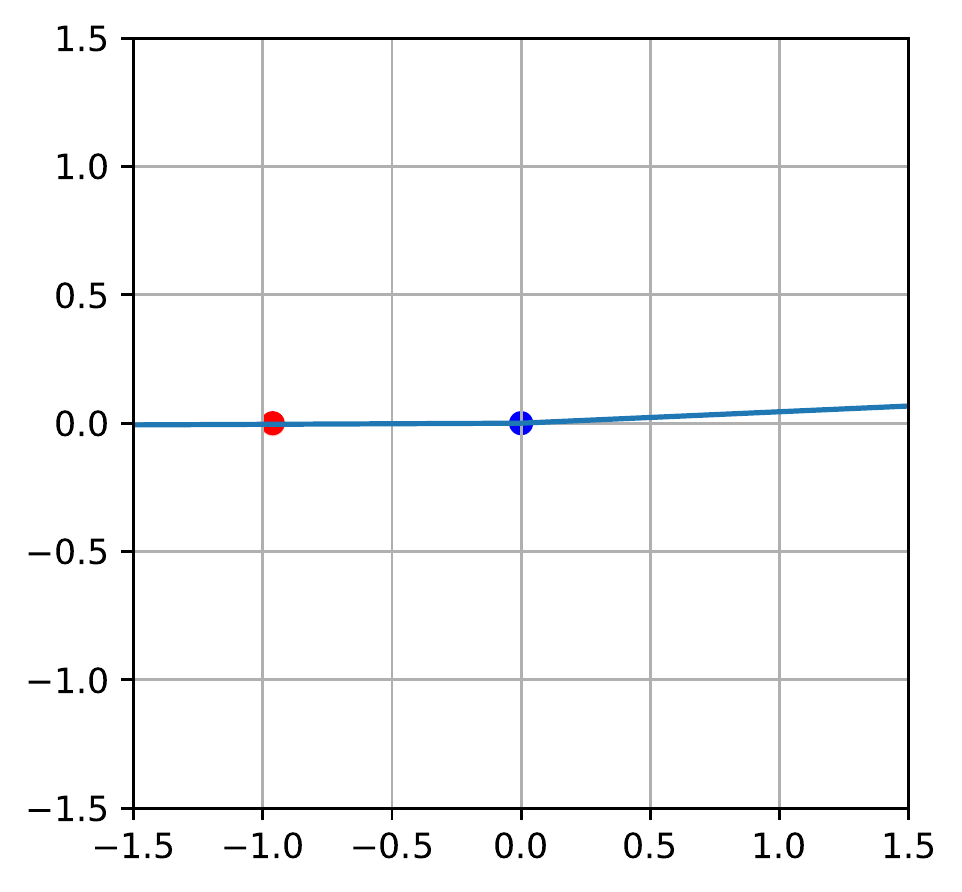}\label{fig:dirac1SampleLow200}}
\subfloat[Iter. 300]{\includegraphics[width=0.08\textwidth]{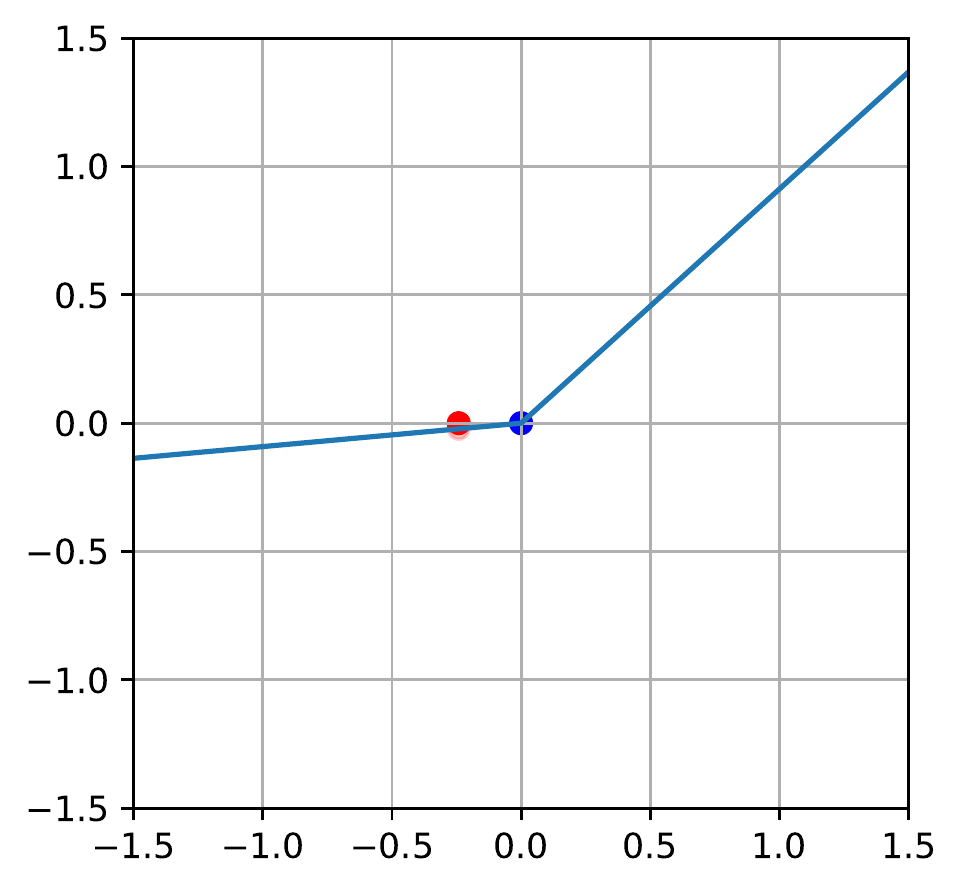}\label{fig:dirac1SampleLow300}} \\

%\subfloat[Iter. 0]{\includegraphics[width=0.08\textwidth]{figs/nhidden_2_ndata_1_fixg_0/dirac_0000.pdf}\label{figappx:dirac1SampleHigh0}}
%\subfloat[Iter. 10]{\includegraphics[width=0.08\textwidth]{figs/nhidden_2_ndata_1_fixg_0/dirac_0010.pdf}\label{figappx:dirac1SampleHigh10}}
%\subfloat[Iter. 70]{\includegraphics[width=0.08\textwidth]{figs/nhidden_2_ndata_1_fixg_0/dirac_0070.pdf}\label{figappx:dirac1SampleHigh70}}
%\subfloat[Iter. 78]{\includegraphics[width=0.08\textwidth]{figs/nhidden_2_ndata_1_fixg_0/dirac_0078.pdf}\label{figappx:dirac1SampleHigh78}}
%
%\subfloat[Iter. 87]{\includegraphics[width=0.08\textwidth]{figs/nhidden_2_ndata_1_fixg_0/dirac_0087.pdf}\label{figappx:dirac1SampleHigh87}}
%\subfloat[Iter. 150]{\includegraphics[width=0.08\textwidth]{figs/nhidden_2_ndata_1_fixg_0/dirac_0150.pdf}\label{figappx:dirac1SampleHigh150}}
%\subfloat[Iter. 175]{\includegraphics[width=0.08\textwidth]{figs/nhidden_2_ndata_1_fixg_0/dirac_0175.pdf}\label{figappx:dirac1SampleHigh175}}
%\subfloat[Iter. 250]{\includegraphics[width=0.08\textwidth]{figs/nhidden_2_ndata_1_fixg_0/dirac_0250.pdf}\label{figappx:dirac1SampleHigh250}}

\subfloat[Optimal]{\includegraphics[width=0.08\textwidth]{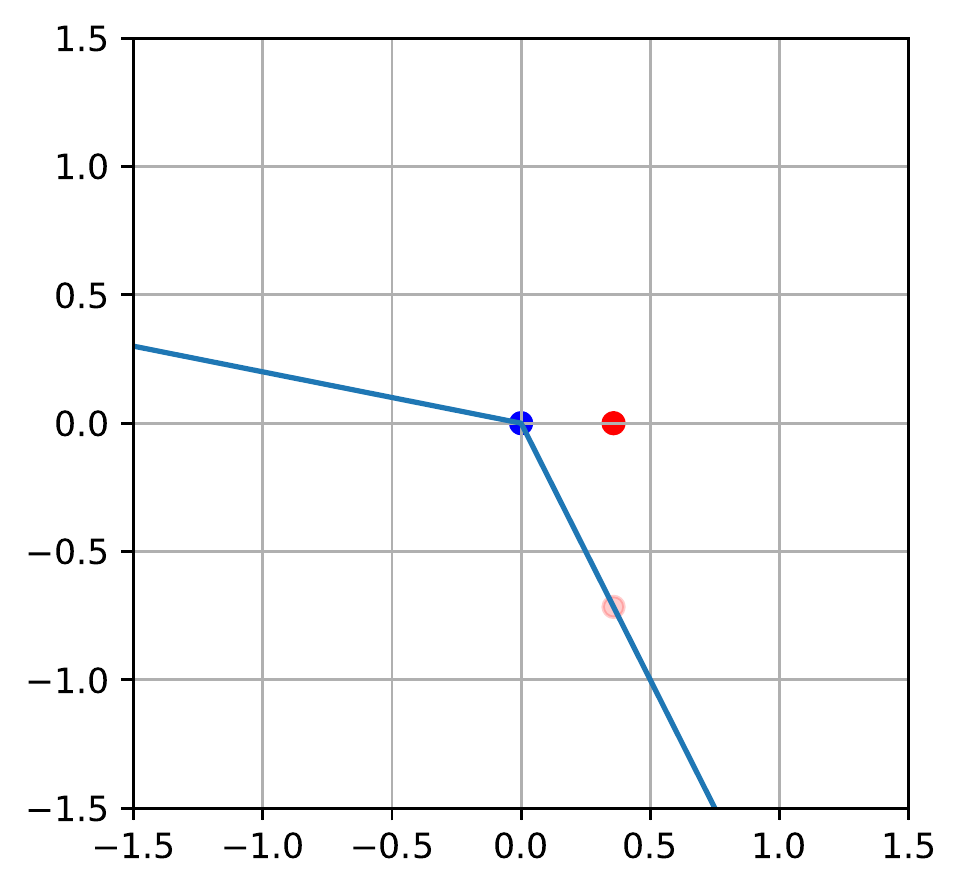}\label{fig:diracOptim}}
\subfloat[Iter. 0]{\includegraphics[width=0.08\textwidth]{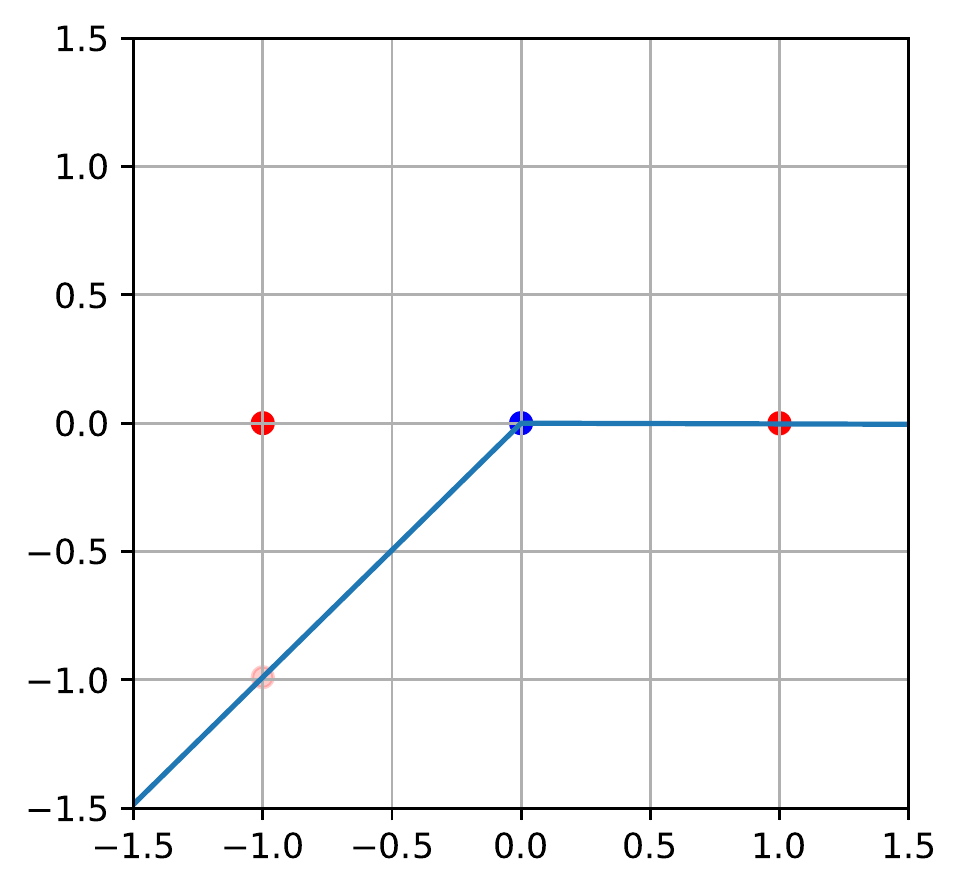}\label{fig:dirac2Sample0}}
\subfloat[Iter. 10]{\includegraphics[width=0.08\textwidth]{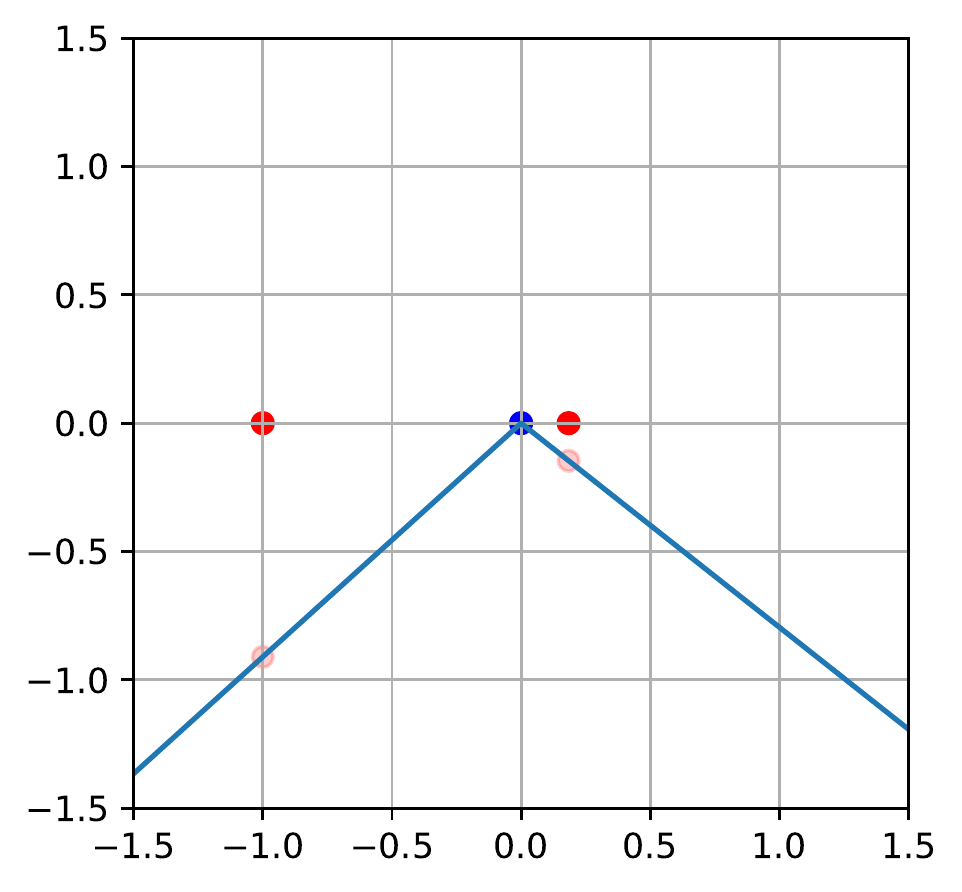}\label{fig:dirac2Sample10}}
\subfloat[Iter. 125]{\includegraphics[width=0.08\textwidth]{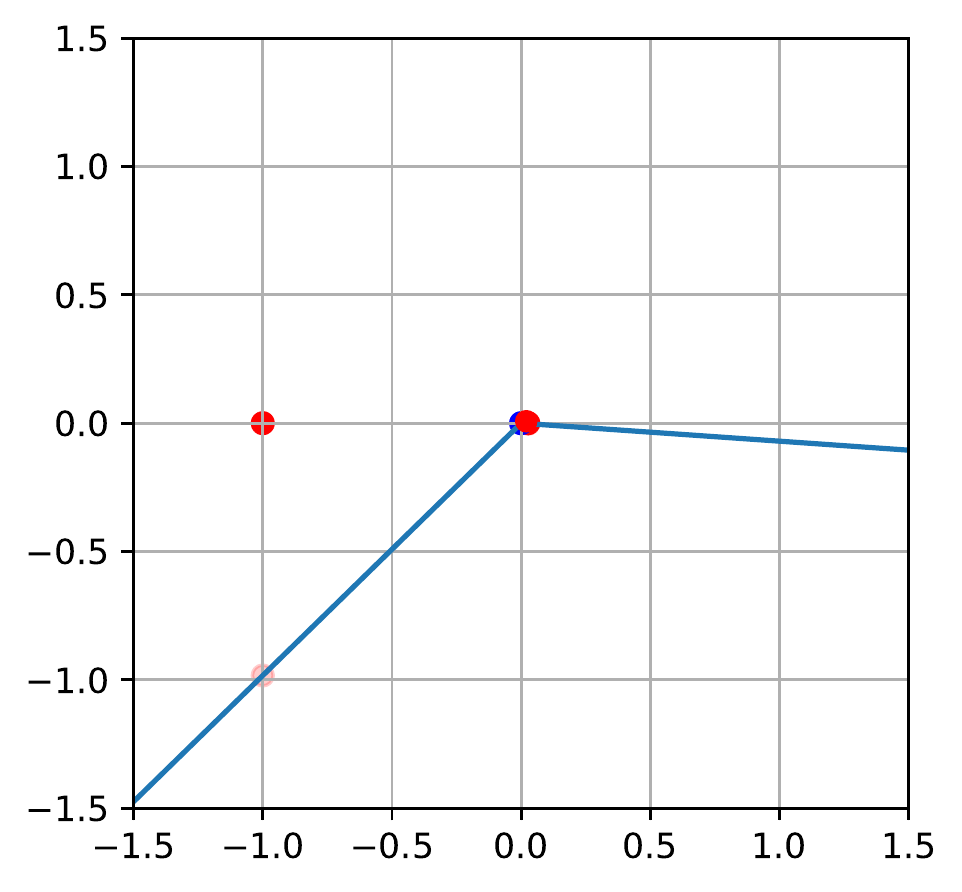}\label{fig:dirac2Sample100}} 
\subfloat[Iter. 250]{\includegraphics[width=0.08\textwidth]{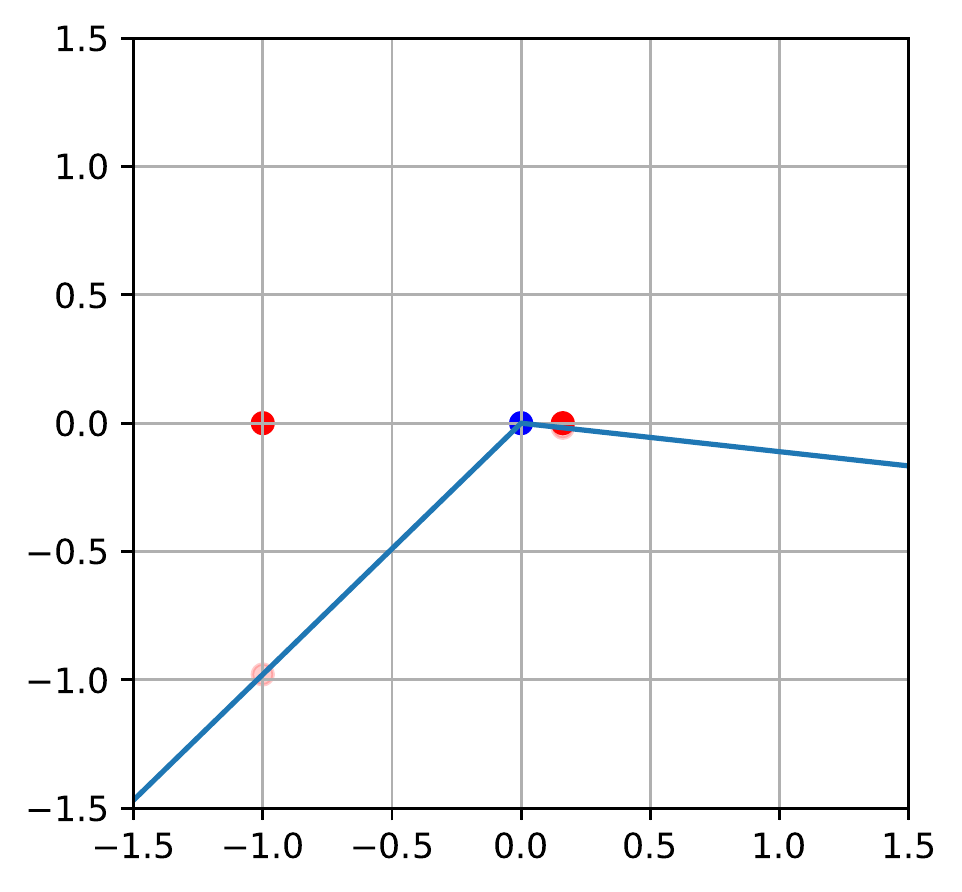}\label{fig:dirac2Sample250}}
%\subfloat[Iter. 500]{\includegraphics[width=0.08\textwidth]{figs/nhidden_2_ndata_2_fixg_0/fixed_fake_0/dirac_0499.pdf}\label{fig:dirac2Sample500}}
%\subfloat[Iter. 0]{\includegraphics[width=0.08\textwidth]{figs/nhidden_2_ndata_2_fixg_0/no_maximum/dirac_0000.pdf}}
%\subfloat[Iter. 10]{\includegraphics[width=0.08\textwidth]{figs/nhidden_2_ndata_2_fixg_0/no_maximum/dirac_0010.pdf}}
%\subfloat[Iter. 30]{\includegraphics[width=0.08\textwidth]{figs/nhidden_2_ndata_2_fixg_0/no_maximum/dirac_0030.pdf}}
%\subfloat[Iter. 500]{\includegraphics[width=0.08\textwidth]{figs/nhidden_2_ndata_2_fixg_0/no_maximum/dirac_0499.pdf}} 
%\subfloat[Iter. 2]{\includegraphics[width=0.08\textwidth]{figs/nhidden_2_ndata_1_fixg_0_diter_10/dirac_0003.pdf}\label{fig:diracOptim20}}
%\subfloat[Iter. 3]{\includegraphics[width=0.08\textwidth]{figs/nhidden_2_ndata_1_fixg_0_diter_10/dirac_0004.pdf}\label{fig:diracOptim30}}
%\caption{Optimal high capacity Dirac discriminator. The discriminator is a 1 hidden layer neural network with Leaky ReLU activation function and 2 hidden neurons. It is trained for 10 iterations every generator's iteration. To make the function $k$-Lipschitz, the weights are clamped to the $[-1, 1]$ range.}

\caption{
High capacity Dirac GAN with $n = 2$. The blue line represents the discriminator's function. The real and fake datapoints are shown by the blue and red dots, respectively. 
%The values $D(x)$ and $D(y)$ are shown in light colors. %The discriminator is $D(x) = \bm \Psi_1^{\top}\sigma(\bm \Psi_0 x)$ where $\bm \Psi_0, \bm \Psi_1 \in \mathbb{R}^{2 \times 1}$. 
\protect\subref{fig:dirac1SampleLow0} - \protect\subref{fig:dirac1SampleLow300}: Dirac GAN trained on the current fake example only. %Refer to Fig. \ref{figappx:dirac1SampleHigh} in Appendix \ref{appx:ctSynthetic} for the full sequence.
\protect\subref{fig:diracOptim}: empirically optimal Dirac discriminator trained on the current fake example only.
\protect\subref{fig:dirac2Sample0} - \protect\subref{fig:dirac2Sample250}: Dirac GAN trained on two fake examples: old fake example on the left and current fake example on the right. 
}
\label{fig:dirac}
\end{figure}

We investigate the effect of CF on Dirac GAN \cite{whichGANConverge}, a GAN that learns a 1 dimensional Dirac distribution located at the origin, $p_r = \delta_0$.
%While Dirac GAN is simple to analyze, the results on Dirac GAN generalize well to more complicated GANs. 
In the original Dirac GAN, the discriminator is a linear function with 1 parameter, $D(x) = \psi x,\ \psi \in [-1, 1]$ and the model distribution is a Dirac distribution located at $\theta$, $p_g = \delta_\theta$. $\theta$ is the generator's parameter. Initially, $\theta \neq 0$.
At each iteration, the training dataset of Dirac GAN contains two training examples: a real training example $x_0 = 0$, and a fake training example $y_0 = \theta$. Gradient updates are applied directly on the fake training example.

\begin{equation}
-\mathcal{L}_G^{dirac} = \mathcal{L}_D^{dirac} = -D(0) + D(x)
\label{eqn:driac}
\end{equation}
The unique equilibrium is $\psi = \theta = 0$. \citeauthor{whichGANConverge} showed that the players in Dirac GAN do not converge to an equilibrium (see Fig. 1 in \cite{whichGANConverge}).
To make the game converge to the above equilibrium, the authors proposed R1 gradient penalty which pushes the gradient w.r.t. the real datapoint to $\bm 0$ (Table \ref{tab:loss}).
%Training a Dirac GAN is equivalent to training a GAN on a dataset with 1 real training example and 1 fake example, and gradient updates are applied directly to the fake example.
A high dimensional GAN can be narrowed to a Dirac GAN by considering a pair of real and fake sample and the discriminator's output along the line connecting these samples (similar to the landscape in Fig. \ref{fig:gannsSGD}-\ref{fig:wgangp}).

%Different from the landscape of GAN-R1 in Fig. \ref{fig:ganr1}, the landscape of Dirac GAN-R1 is always a monotonic function without any local maxima. 
Because the discriminator in the original Dirac GAN is a linear function with a single parameter, the output of Dirac discriminator is always a monotonic function. 
We consider a generic discriminator which is a 1 hidden layer neural network: $\hat{D}(x) = \bm \Psi_1^{\top}\sigma(\bm \Psi_0 x)$ where $\bm \Psi_0, \bm \Psi_1 \in [-1, 1]^{n \times 1}$, and $\sigma$ is a monotonically increasing activation function such as Leaky ReLU (Fig. \ref{fig:dirac}).
At equilibrium, $\theta = 0$ and $\hat{D}(x)$ is any function with a global maximum at $x = 0$.
Although $\hat{D}$ can have global maxima (see Fig. \ref{fig:dirac2Sample10}), optimizing $\hat{D}$ only on the current task makes $\hat{D}$ a monotonic function (Fig. \ref{fig:diracOptim}).

\begin{proposition}
The optimal Dirac discriminator $\hat{D}^*(x)$ that minimizes $\mathcal{L}_D^{dirac}$ in Eqn. \ref{eqn:driac} is a monotonic function. 
\label{prop:diracMonotonic}
\end{proposition}

\begin{proof}

Let $\hat{D}(x) = \bm \Psi_1^{\top}\sigma(\bm \Psi_0 x)$ where $\bm \Psi_0, \bm \Psi_1 \in [-1, 1]^{n \times 1}$ be the discriminator and $\sigma$ be a non-decreasing activation function such as ReLU, Leaky ReLU, Sigmoid, or Tanh. Let $x_0 = 0$ be the real datapoint, $y_0 = \theta \neq 0$ be the fake datapoint. The empirically optimal discriminator $D^*$ must maximize the difference $D^*(x_0) - D^*(y_0)$.
\begin{eqnarray*}
 \hat{D}(x_0)& = & \bm \Psi_1^{\top}\sigma(\bm \Psi_0 \times 0) \\
 & = & \bm \Psi_1^{\top} \sigma(\bm 0) \\
 & = & \sum_{i = 1}^n \Psi_{1, i} \sigma(0) \\
 \hat{D}(y_0)& = & \bm \Psi_1^{\top}\sigma(\bm \Psi_0 \times y_0) \\
 & = & \sum_{i = 1}^n {\Psi_{1, i} \sigma(\Psi_{0, i} y_0)} \\
 \hat{D}(x_0) - \hat{D}(y_0) & = & \sum_{i = 1}^n \Psi_{1, i} \times (\sigma(0) - \sigma(\Psi_{0, i} y_0))
\end{eqnarray*}

Because
\begin{eqnarray*}
\Psi_{0, i} y_0 & \le & \abs{y_0} \nonumber 
\end{eqnarray*}
and $\sigma$ is non-decreasing
\begin{equation*}
\sigma(0) - \sigma(-\abs{y_0}) \ge \sigma(0) - \sigma(\Psi_{0, i} y_0) \ge \sigma(0) - \sigma(\abs{y_0})
\end{equation*}
If $\sigma$ is ReLU or Leaky ReLU or Tanh, then 
$\sigma(0) = 0$, 
$\abs{\sigma(\abs{y_0})} \ge \abs{\sigma(-\abs{y_0})}$,
thus 
\[\abs{\sigma(0) - \sigma(\abs{y_0})} > \abs{\sigma(0) - \sigma(-\abs{y_0})}\]
If $\sigma$ is Sigmoid, then 
$\sigma(0) = 0.5$
and 
$\abs{\sigma(0) - \sigma(\abs{y_0})} = \abs{\sigma(0) - \sigma(-\abs{y_0})}$.
For both cases, we have
\begin{equation}
\abs{\sigma(0) - \sigma(\Psi_{0, i} y_0)} \le \abs{\sigma(0) - \sigma(\abs{y_0})}
\end{equation}
Thus 
\begin{equation}
{\Psi_{1, i} (\sigma(0) - \sigma(\Psi_{0, i} y_0))} \le 1 \times \abs{\sigma(0) -\sigma(\abs{y_0})}
\end{equation}

The equality for both Eqn. 1 and 2 is achieved for all cases when $\Psi_{1, i} = -1$ and $\sigma({\Psi_{0, i} y_0}) = \sigma(\abs{y_0}) \Rightarrow \Psi_{0, i} y_0 = \abs{y_0} \Rightarrow \Psi_{0, i} = \text{sign}(y_0)$. The optimal discriminator's parameters are $\bm \Psi_0^* = \text{sign}(y_0) \times \bm 1, \bm \Psi_1^* = -\bm 1$. %Flipping the sign of both $\bm \Psi_0^*$ and $\bm \Psi_1^*$ result in another optimal set of parameters. 
\[ D(x) = -\bm 1^\top \sigma(x \times \text{sign}(y_0) \times \bm 1) \]
Without loss of generality, assume $\text{sign}(y_0) = 1$.
\[ D(x) = -\bm 1^\top \sigma(x \times \bm 1) = -n\sigma(x) \]
Because $\sigma$ is monotonic, $D(x)$ is monotonic.
\end{proof}

Optimizing the performance of $\hat{D}$ pushes it toward $\hat{D}^*$, making $\hat{D}$ monotonic (Fig. \ref{fig:dirac1SampleLow0} - \ref{fig:dirac1SampleLow300}). This explains the directional monotonicity of discriminators in Fig. \ref{fig:8Gauss3000}-\ref{fig:8Gauss20000},  \ref{fig:gannsSGD}.

Although the discriminator in Fig. \ref{fig:diracOptim} minimizes the score of the current fake datapoint, it assigns high scores to (old) fake datapoints on the left of the real datapoint, i.e. it forgets these datapoints.
If the discriminator is fixed, then minimizing $\mathcal{L}_G^{dirac}$ corresponds to moving $\theta$ to $-\infty$.
\textit{Dirac GAN with a monotonic discriminator does not converge.}
When the generator and discriminator are trained with alternating SGD, the two players oscillate around the equilibrium (Fig. \ref{fig:dirac1SampleLow0} - \ref{fig:dirac1SampleLow300}).

The problem can be alleviated if one old fake datapoint is added to the training dataset.
Fig. \ref{fig:dirac2Sample0} - \ref{fig:dirac2Sample250} shows that when old fake example is added, Dirac GAN has better convergence behavior (the small fluctuation is due to the large constant learning rate of 0.1).
The discriminator at iteration 10 has a global maximum at the origin. 
If the discriminator is fixed, then $\theta$ will converge to 0.
The experiment suggests that information about previous model distributions helps GANs converge.
\cite{shrivastava2017learning} used a buffer of recent old fake samples to refine reasonably good fake samples. Recent old fake samples reduce the oscillation around the equilibrium, helping GANs to converge faster and produce sharper images.
However, because the number of samples needed to capture the statistics of a distribution grows exponentially with it dimensionality, storing old fake datapoints is not efficient for high dimensional data. 
In the next section, we study more efficient methods for preserving information about old distributions.

\begin{figure}
\centering
\begin{flushright}
\subfloat[Real]{
\adjincludegraphics[width=0.11\textwidth, trim={0 {0.75\width} {0.5\width} {0\width}}, clip]{dcgan-celeba/gan-dcgan-celeba-nn512-nr3-nf64-nf64-ld0.0002-lg0.0002-nd1-ng1-gt0.0-gr0.0-grNone-oradam-mm0.0-b10.5-b20.99-dsNone-ntGaussian-nm100-ns30000-be64-sv0.02-se1-ip0.01-idslerp-ne100.0-np1.0-ntTrue-ieFalse-/2019-11-28-11:39:43.776340/real}\label{fig:celebaNSReal}}
\subfloat[Land. 5000]{
\adjincludegraphics[width=0.11\textwidth, trim={0 {0.74\width} {0.5\width} {0\width}}, clip]{dcgan-celeba/gan-dcgan-celeba-nn512-nr3-nf64-nf64-ld0.0002-lg0.0002-nd1-ng1-gt0.0-gr0.0-grNone-oradam-mm0.0-b10.5-b20.99-dsNone-ntGaussian-nm100-ns30000-be64-sv0.02-se1-ip0.01-idslerp-ne100.0-np1.0-ntTrue-ieFalse-/2019-11-28-11:39:43.776340/extrema-05000.pdf}\label{fig:celebaNSLand5k}}
\subfloat[Land. 10000]{
\adjincludegraphics[width=0.11\textwidth, trim={0 {0.74\width} {0.5\width} {0\width}}, clip]{dcgan-celeba/gan-dcgan-celeba-nn512-nr3-nf64-nf64-ld0.0002-lg0.0002-nd1-ng1-gt0.0-gr0.0-grNone-oradam-mm0.0-b10.5-b20.99-dsNone-ntGaussian-nm100-ns30000-be64-sv0.02-se1-ip0.01-idslerp-ne100.0-np1.0-ntTrue-ieFalse-/2019-11-28-11:39:43.776340/extrema-10000.pdf}}
\subfloat[Land. 20000]{
\adjincludegraphics[width=0.11\textwidth, trim={0 {0.74\width} {0.5\width} {0\width}}, clip]{dcgan-celeba/gan-dcgan-celeba-nn512-nr3-nf64-nf64-ld0.0002-lg0.0002-nd1-ng1-gt0.0-gr0.0-grNone-oradam-mm0.0-b10.5-b20.99-dsNone-ntGaussian-nm100-ns30000-be64-sv0.02-se1-ip0.01-idslerp-ne100.0-np1.0-ntTrue-ieFalse-/2019-11-28-11:39:43.776340/extrema-20000.pdf}\label{fig:celebaNSLand20k}}

\subfloat[Gen. 5000]{
\adjincludegraphics[width=0.11\textwidth, trim={0 {0.875\width} {0.5\width} {0\width}}, clip]{dcgan-celeba/gan-dcgan-celeba-nn512-nr3-nf64-nf64-ld0.0002-lg0.0002-nd1-ng1-gt0.0-gr0.0-grNone-oradam-mm0.0-b10.5-b20.99-dsNone-ntGaussian-nm100-ns30000-be64-sv0.02-se1-ip0.01-idslerp-ne100.0-np1.0-ntTrue-ieFalse-/2019-11-28-11:39:43.776340/fake-05000}
}
\subfloat[Gen. 10000]{
\adjincludegraphics[width=0.11\textwidth, trim={0 {0.875\width} {0.5\width} {0\width}}, clip]{dcgan-celeba/gan-dcgan-celeba-nn512-nr3-nf64-nf64-ld0.0002-lg0.0002-nd1-ng1-gt0.0-gr0.0-grNone-oradam-mm0.0-b10.5-b20.99-dsNone-ntGaussian-nm100-ns30000-be64-sv0.02-se1-ip0.01-idslerp-ne100.0-np1.0-ntTrue-ieFalse-/2019-11-28-11:39:43.776340/fake-10000}
}
\subfloat[Gen. 20000]{
\adjincludegraphics[width=0.11\textwidth, trim={0 {0.875\width} {0.5\width} {0\width}}, clip]{dcgan-celeba/gan-dcgan-celeba-nn512-nr3-nf64-nf64-ld0.0002-lg0.0002-nd1-ng1-gt0.0-gr0.0-grNone-oradam-mm0.0-b10.5-b20.99-dsNone-ntGaussian-nm100-ns30000-be64-sv0.02-se1-ip0.01-idslerp-ne100.0-np1.0-ntTrue-ieFalse-/2019-11-28-11:39:43.776340/fake-20000}\label{fig:celebaNSFake20k}}

\subfloat[Real]{
\adjincludegraphics[width=0.11\textwidth, trim={0 {0.75\width} {0.5\width} {0\width}}, clip]{dcgan-celeba/gan-dcgan-celeba-nn512-nr3-nf64-nf64-ld0.0002-lg0.0002-nd1-ng1-gt10.0-gr0.0-grNone-oradam-mm0.0-b10.5-b20.99-dsNone-ntGaussian-nm100-ns30000-be64-sv0.02-se1-ip0.01-idslerp-ne100.0-np1.0-ntTrue-ieFalse-/2019-11-28-11:47:39.422148/real}\label{fig:celeba0GPReal}}
\subfloat[Land. 5000]{
\adjincludegraphics[width=0.11\textwidth, trim={0 {0.74\width} {0.5\width} {0\width}}, clip]{dcgan-celeba/gan-dcgan-celeba-nn512-nr3-nf64-nf64-ld0.0002-lg0.0002-nd1-ng1-gt10.0-gr0.0-grNone-oradam-mm0.0-b10.5-b20.99-dsNone-ntGaussian-nm100-ns30000-be64-sv0.02-se1-ip0.01-idslerp-ne100.0-np1.0-ntTrue-ieFalse-/2019-11-28-11:47:39.422148/extrema-05000.pdf}}
\subfloat[Land. 10000]{
\adjincludegraphics[width=0.11\textwidth, trim={0 {0.74\width} {0.5\width} {0\width}}, clip]{dcgan-celeba/gan-dcgan-celeba-nn512-nr3-nf64-nf64-ld0.0002-lg0.0002-nd1-ng1-gt10.0-gr0.0-grNone-oradam-mm0.0-b10.5-b20.99-dsNone-ntGaussian-nm100-ns30000-be64-sv0.02-se1-ip0.01-idslerp-ne100.0-np1.0-ntTrue-ieFalse-/2019-11-28-11:47:39.422148/extrema-10000.pdf}}
\subfloat[Land. 20000]{
\adjincludegraphics[width=0.11\textwidth, trim={0 {0.74\width} {0.5\width} {0\width}}, clip]{dcgan-celeba/gan-dcgan-celeba-nn512-nr3-nf64-nf64-ld0.0002-lg0.0002-nd1-ng1-gt10.0-gr0.0-grNone-oradam-mm0.0-b10.5-b20.99-dsNone-ntGaussian-nm100-ns30000-be64-sv0.02-se1-ip0.01-idslerp-ne100.0-np1.0-ntTrue-ieFalse-/2019-11-28-11:47:39.422148/extrema-20000.pdf}}

\subfloat[Gen. 5000]{
\adjincludegraphics[width=0.11\textwidth, trim={0 {0.875\width} {0.5\width} {0\width}}, clip]{dcgan-celeba/gan-dcgan-celeba-nn512-nr3-nf64-nf64-ld0.0002-lg0.0002-nd1-ng1-gt10.0-gr0.0-grNone-oradam-mm0.0-b10.5-b20.99-dsNone-ntGaussian-nm100-ns30000-be64-sv0.02-se1-ip0.01-idslerp-ne100.0-np1.0-ntTrue-ieFalse-/2019-11-28-11:47:39.422148/fake-05000}
}
\subfloat[Gen. 10000]{
\adjincludegraphics[width=0.11\textwidth, trim={0 {0.875\width} {0.5\width} {0\width}}, clip]{dcgan-celeba/gan-dcgan-celeba-nn512-nr3-nf64-nf64-ld0.0002-lg0.0002-nd1-ng1-gt10.0-gr0.0-grNone-oradam-mm0.0-b10.5-b20.99-dsNone-ntGaussian-nm100-ns30000-be64-sv0.02-se1-ip0.01-idslerp-ne100.0-np1.0-ntTrue-ieFalse-/2019-11-28-11:47:39.422148/fake-10000}
}
\subfloat[Gen. 20000]{
\adjincludegraphics[width=0.11\textwidth, trim={0 {0.875\width} {0.5\width} {0\width}}, clip]{dcgan-celeba/gan-dcgan-celeba-nn512-nr3-nf64-nf64-ld0.0002-lg0.0002-nd1-ng1-gt10.0-gr0.0-grNone-oradam-mm0.0-b10.5-b20.99-dsNone-ntGaussian-nm100-ns30000-be64-sv0.02-se1-ip0.01-idslerp-ne100.0-np1.0-ntTrue-ieFalse-/2019-11-28-11:47:39.422148/fake-20000}\label{fig:celeba0GPFake20k}}
\end{flushright}

\caption{Result on CelebA. \protect\subref{fig:celebaNSReal} - \protect\subref{fig:celebaNSFake20k} DCGAN-NS. \protect\subref{fig:celeba0GPReal} - \protect\subref{fig:celeba0GPFake20k} DCGAN-0GP}
\label{fig:celeba}
\end{figure}

\begin{figure}
\centering
\begin{flushright}
\subfloat[Real]{
\adjincludegraphics[width=0.11\textwidth, trim={0 {0.875\width} {0.5\width} {0\width}}, clip]{dcgan-cifar10/gan-dcgan-cifar10-nn512-nr3-nf64-nf64-ld0.0003-lg0.0003-nd1-ng1-gt0.0-gr1.0-grNone-oradam-mm0.0-b10.5-b20.99-dsNone-ntGaussian-nm100-ns200001-be64-sv0.02-se1-ip0.01-idslerp-ne100.0-np1.0-ntTrue-ieFalse-/2019-11-27-00:48:15.239335/real}\label{fig:dcgannsReal}}
\subfloat[Land. 5000]{
\adjincludegraphics[width=0.11\textwidth, trim={0 {0.86\width} {0.5\width} {0\width}}, clip]{dcgan-cifar10/gan-dcgan-cifar10-nn512-nr3-nf64-nf64-ld0.0003-lg0.0003-nd1-ng1-gt0.0-gr1.0-grNone-oradam-mm0.0-b10.5-b20.99-dsNone-ntGaussian-nm100-ns200001-be64-sv0.02-se1-ip0.01-idslerp-ne100.0-np1.0-ntTrue-ieFalse-/2019-11-27-00:48:15.239335/extrema-05000.pdf}
\label{fig:dcgannsLand5k}}
\subfloat[Land. 10000]{
\adjincludegraphics[width=0.11\textwidth, trim={0 {0.86\width} {0.5\width} {0\width}}, clip]{dcgan-cifar10/gan-dcgan-cifar10-nn512-nr3-nf64-nf64-ld0.0003-lg0.0003-nd1-ng1-gt0.0-gr1.0-grNone-oradam-mm0.0-b10.5-b20.99-dsNone-ntGaussian-nm100-ns200001-be64-sv0.02-se1-ip0.01-idslerp-ne100.0-np1.0-ntTrue-ieFalse-/2019-11-27-00:48:15.239335/extrema-10000.pdf}}
\subfloat[Land. 20000]{
\adjincludegraphics[width=0.11\textwidth, trim={0 {0.86\width} {0.5\width} {0\width}}, clip]{dcgan-cifar10/gan-dcgan-cifar10-nn512-nr3-nf64-nf64-ld0.0003-lg0.0003-nd1-ng1-gt0.0-gr1.0-grNone-oradam-mm0.0-b10.5-b20.99-dsNone-ntGaussian-nm100-ns200001-be64-sv0.02-se1-ip0.01-idslerp-ne100.0-np1.0-ntTrue-ieFalse-/2019-11-27-00:48:15.239335/extrema-20000.pdf}
\label{fig:dcgannsLand20k}}

\subfloat[Gen. 5000]{
\adjincludegraphics[width=0.11\textwidth, trim={0 {0.875\width} {0.5\width} {0\width}}, clip]{dcgan-cifar10/gan-dcgan-cifar10-nn512-nr3-nf64-nf64-ld0.0003-lg0.0003-nd1-ng1-gt0.0-gr1.0-grNone-oradam-mm0.0-b10.5-b20.99-dsNone-ntGaussian-nm100-ns200001-be64-sv0.02-se1-ip0.01-idslerp-ne100.0-np1.0-ntTrue-ieFalse-/2019-11-27-00:48:15.239335/fake-05000}
}
\subfloat[Gen. 10000]{
\adjincludegraphics[width=0.11\textwidth, trim={0 {0.875\width} {0.5\width} {0\width}}, clip]{dcgan-cifar10/gan-dcgan-cifar10-nn512-nr3-nf64-nf64-ld0.0003-lg0.0003-nd1-ng1-gt0.0-gr1.0-grNone-oradam-mm0.0-b10.5-b20.99-dsNone-ntGaussian-nm100-ns200001-be64-sv0.02-se1-ip0.01-idslerp-ne100.0-np1.0-ntTrue-ieFalse-/2019-11-27-00:48:15.239335/fake-10000}
}
\subfloat[Gen. 20000]{
\adjincludegraphics[width=0.11\textwidth, trim={0 {0.875\width} {0.5\width} {0\width}}, clip]{dcgan-cifar10/gan-dcgan-cifar10-nn512-nr3-nf64-nf64-ld0.0003-lg0.0003-nd1-ng1-gt0.0-gr1.0-grNone-oradam-mm0.0-b10.5-b20.99-dsNone-ntGaussian-nm100-ns200001-be64-sv0.02-se1-ip0.01-idslerp-ne100.0-np1.0-ntTrue-ieFalse-/2019-11-27-00:48:15.239335/fake-20000}\label{fig:dcgannsFake20k}}

\subfloat[Real]{
\adjincludegraphics[width=0.11\textwidth, trim={0 {0.75\width} {0.5\width} {0\width}}, clip]{dcgan-imba/gan-dcgan-cifar10-nn512-rt1.0-ft1.0-nr3-nf64-nf64-ld0.0002-lg0.0002-nd1-ng1-gt0.0-gr0.0-grNone-oradam-mm0.0-b10.5-b20.99-dsNone-ntGaussian-nm100-ns50001-be64-sv0.02-se1-ip0.01-idslerp-ne100.0-np1.0-ntTrue-ieFalse-/2019-11-29-15:25:39.921630/real}\label{fig:dcganimbaReal}}
\subfloat[Land. 5000]{
\adjincludegraphics[width=0.11\textwidth, trim={0 {0.74\width} {0.5\width} {0\width}}, clip]{dcgan-imba/gan-dcgan-cifar10-nn512-rt1.0-ft1.0-nr3-nf64-nf64-ld0.0002-lg0.0002-nd1-ng1-gt0.0-gr0.0-grNone-oradam-mm0.0-b10.5-b20.99-dsNone-ntGaussian-nm100-ns50001-be64-sv0.02-se1-ip0.01-idslerp-ne100.0-np1.0-ntTrue-ieFalse-/2019-11-29-15:25:39.921630/extrema-05000.pdf}}
\subfloat[Land. 10000]{
\adjincludegraphics[width=0.11\textwidth, trim={0 {0.74\width} {0.5\width} {0\width}}, clip]{dcgan-imba/gan-dcgan-cifar10-nn512-rt1.0-ft1.0-nr3-nf64-nf64-ld0.0002-lg0.0002-nd1-ng1-gt0.0-gr0.0-grNone-oradam-mm0.0-b10.5-b20.99-dsNone-ntGaussian-nm100-ns50001-be64-sv0.02-se1-ip0.01-idslerp-ne100.0-np1.0-ntTrue-ieFalse-/2019-11-29-15:25:39.921630/extrema-10000.pdf}}
\subfloat[Land. 20000]{
\adjincludegraphics[width=0.11\textwidth, trim={0 {0.73\width} {0.5\width} {0\width}}, clip]{dcgan-imba/gan-dcgan-cifar10-nn512-rt1.0-ft1.0-nr3-nf64-nf64-ld0.0002-lg0.0002-nd1-ng1-gt0.0-gr0.0-grNone-oradam-mm0.0-b10.5-b20.99-dsNone-ntGaussian-nm100-ns50001-be64-sv0.02-se1-ip0.01-idslerp-ne100.0-np1.0-ntTrue-ieFalse-/2019-11-29-14:24:58.219254/extrema-20000.pdf}\label{fig:dcganimbaLand20k}}

\subfloat[Gen. 5000]{
\adjincludegraphics[width=0.11\textwidth, trim={0 {0.75\width} {0.5\width} {0\width}}, clip]{dcgan-imba/gan-dcgan-cifar10-nn512-rt1.0-ft1.0-nr3-nf64-nf64-ld0.0002-lg0.0002-nd1-ng1-gt0.0-gr0.0-grNone-oradam-mm0.0-b10.5-b20.99-dsNone-ntGaussian-nm100-ns50001-be64-sv0.02-se1-ip0.01-idslerp-ne100.0-np1.0-ntTrue-ieFalse-/2019-11-29-15:25:39.921630/fake-05000}
}
\subfloat[Gen. 10000]{
\adjincludegraphics[width=0.11\textwidth, trim={0 {0.75\width} {0.5\width} {0\width}}, clip]{dcgan-imba/gan-dcgan-cifar10-nn512-rt1.0-ft1.0-nr3-nf64-nf64-ld0.0002-lg0.0002-nd1-ng1-gt0.0-gr0.0-grNone-oradam-mm0.0-b10.5-b20.99-dsNone-ntGaussian-nm100-ns50001-be64-sv0.02-se1-ip0.01-idslerp-ne100.0-np1.0-ntTrue-ieFalse-/2019-11-29-15:25:39.921630/fake-10000}
}
\subfloat[Gen. 20000]{
\adjincludegraphics[width=0.11\textwidth, trim={0 {0.75\width} {0.5\width} {0\width}}, clip]{dcgan-imba/gan-dcgan-cifar10-nn512-rt1.0-ft1.0-nr3-nf64-nf64-ld0.0002-lg0.0002-nd1-ng1-gt0.0-gr0.0-grNone-oradam-mm0.0-b10.5-b20.99-dsNone-ntGaussian-nm100-ns50001-be64-sv0.02-se1-ip0.01-idslerp-ne100.0-np1.0-ntTrue-ieFalse-/2019-11-29-15:25:39.921630/fake-20000}\label{fig:dcganimbaFake20k}}
\end{flushright}
\caption{Result on CIFAR-10. \protect\subref{fig:dcgannsReal} - \protect\subref{fig:dcgannsFake20k} DCGAN-NS.
% \protect\subref{fig:dcgan0GPReal} - \protect\subref{fig:dcgan0GPFake20k} DCGAN-0GP, $\lambda=10$. 
 \protect\subref{fig:dcganimbaReal} - \protect\subref{fig:dcganimbaFake20k} DCGAN-imba, $\gamma=10$.}
\label{fig:cifar10}
\end{figure}

\begin{figure}
\subfloat[Img.]{
\adjincludegraphics[width=0.08\textwidth, trim={0 0 {0.983\width} {0\width}}, clip]{mnist-randscore/gan-mlp-mnist-rt1-ft1-nn512-nr3-ld0.0003-lg0.0003-nd1-ng1-gt0.0-gr0.0-grNone-orsgd-mm0.0-b10.5-b20.99-dsNone-ntGaussian-nm50-ns200001-be64-sv0.02-se1-ip0.01-idslerp-ne100.0-np1.0-ntTrue-ieFalse-/2019-11-30-02:08:16.033915/fixed-fake-10000}\label{fig:gannsImg}}
\subfloat[Score]{
\adjincludegraphics[width=0.2\textwidth, trim={0 {0.9831\height} {0.\width} {0\width}}, clip]{mnist-randscore/gan-mlp-mnist-rt1-ft1-nn512-nr3-ld0.0003-lg0.0003-nd1-ng1-gt0.0-gr0.0-grNone-orsgd-mm0.0-b10.5-b20.99-dsNone-ntGaussian-nm50-ns200001-be64-sv0.02-se1-ip0.01-idslerp-ne100.0-np1.0-ntTrue-ieFalse-/2019-11-30-02:08:16.033915/fixed-fake-scores-10000.pdf}\label{fig:gannsScore}}

\subfloat[Img.]{
\adjincludegraphics[width=0.08\textwidth, trim={0 {0.96875\height} {0.\width} {0\width}}, clip]{mnist-randscore/gan-mlp-mnist-rt1-ft1-nn512-nr3-ld0.0003-lg0.0003-nd1-ng1-gt100.0-gr0.0-grNone-oradam-mm0.0-b10.5-b20.99-dsNone-ntGaussian-nm50-ns200001-be64-sv0.02-se1-ip0.01-idslerp-ne100.0-np1.0-ntTrue-ieFalse-/2019-11-30-10:06:53.148293/fixed-fake-10000}\label{fig:gan0gpImg}}
\subfloat[Score]{
\adjincludegraphics[width=0.2\textwidth, trim={0 {0.968\height} {0.\width} {0\width}}, clip]{mnist-randscore/gan-mlp-mnist-rt1-ft1-nn512-nr3-ld0.0003-lg0.0003-nd1-ng1-gt100.0-gr0.0-grNone-oradam-mm0.0-b10.5-b20.99-dsNone-ntGaussian-nm50-ns200001-be64-sv0.02-se1-ip0.01-idslerp-ne100.0-np1.0-ntTrue-ieFalse-/2019-11-30-10:06:53.148293/fixed-fake-scores-10000.pdf}\label{fig:gan0gpScore}}
\caption{Score of fixed fake images during training from iteration 10000 to 200000. The same MLP in Fig. 2 was trained with SGD with learning rate $3e-4$. \protect\subref{fig:gannsImg} - \protect\subref{fig:gannsScore} GAN-NS. \protect\subref{fig:gan0gpImg} - \protect\subref{fig:gan0gpScore} GAN-0GP with $\lambda=100$. GAN-NS assigns random scores to the same fake image, implying that it does not remember information about this fake sample. GAN-0GP is much more stable and consistently assigns scores lower than 0.5 to old fake samples.}
\label{fig:ganScore}
\end{figure}

\begin{table}
\centering
\begin{tabular}{|c|c|}
\hline
Architecture & DCGAN Pytorch example \\
Learning rate & 2e-4 \\
Batch size & 64 \\
Optimizer & Adam, $\beta_1 = 0.5, \beta_2 = 0.99$ \\
No. filters at 1st layer & 64 \\
\hline
\end{tabular}
\caption{DCGAN model architecture \& hyper parameters.}
\label{tab:dcganArch}
\end{table}

\begin{table}
\centering
\begin{tabular}{|c|c|}
\hline
 & mean/std \\
\hline
DCGAN & 2.054/0.913 \\
DCGAN-imba, $\gamma=10$ & 3.381/0.078 \\
DCGAN-0GP, $\lambda=100$ &  2.705/0.901 \\
DCGAN-0GP-imba, $\lambda=100, \gamma=10$ & 3.038/0.342 \\
\hline
\end{tabular}
\caption{Inception scores of models at iteration 50k. The result is averaged over 10 different runs.}% For DCGAN-imba, $\mathcal{L}_D = \gamma \times \mathcal{L}_{real} + \mathcal{L}_{fake}$}
\label{tab:dcganResult}
\end{table}

%%%%%%%%%%%%%%%%%%%%%%%%%%%%%%%%%%%%%%%%%%%%%%%%%%%%%%%%%

\section{Preventing catastrophic forgetting}
\label{sec:methods}
Based on the reasons identified in Section \ref{sec:cfGAN}, we propose the following ways to address CF problem:
\begin{enumerate}
\item \textit{Preserve and use information from previous tasks in the current task}. 
%This method alleviates catastrophic forgetting by helping the players to access information from previous iterations. %Examples include optimizers with momentum e.g. SGD with momentum \cite{momentum}, Adam \cite{adam} and continual learning algorithms e.g. Elastic Weight Consolidation (EWC) \cite{ewc}, Synaptic Intelligence (SI) \cite{synapticIntel}.
\item \textit{Introduce prior knowledge to the game in a way such that old knowledge is useful for the new task and is not erased by the new task.} 
%In this way, prior knowledge is introduced to the game to guide the game toward an equilibrium (which is not necessarily the  equilibrium where $p_g = p_r$). 
%We classify zero centered gradient penalties (R1, R2 \cite{whichGANConverge}, 0GP \cite{improveGeneralization}) and WGAN \cite{wgan, wgangp} under this category. %Link to \cite{manypatheq}.
\end{enumerate}

%CF happens when knowledge in task $\mathcal{T}^t$ cannot be used for task $\mathcal{T}^{t - n}$.
%We want the sequence $\{p_g^i\}_{i = 1}^T$ to evolve in a way such that a classifier (discriminator) that separates $\{ p_g^t, p_r \}$ also separates $\{ p_g^u, p_r \}, \forall \ u < t$. 
%\begin{enumerate}
%\item Do we need to prevent CF? Yes, we do. Because mode collapse is not completely solved by current methods, CF will happen when mode collapse happen. Reducing mode collapse $\Leftrightarrow$ catastrophic forgetting.
%\item Review methods for preventing catastrophic forgetting. Show how they help to shape the value surface.
%\end{enumerate}

\subsection{Preserving and using old information}
\textbf{Optimizers with momentum.} The update rule of SGD with momentum
\begin{eqnarray*}
\bm g^{t} & = & \gamma \bm g^{t - 1} + \eta \nabla_{\theta}^{t} \\
\bm \theta^{t + 1} & = & \bm \theta^{t} - \bm g^t
\end{eqnarray*}
The momentum term $\gamma \bm g^{t - 1}$ is a simple form of memory that carries gradient information from previous training iterations to the current iteration. 
When the discriminator/generator is updated with $\bm g^t$, the performance of the network on previous tasks is also improved.
The effectiveness of momentum in preventing CF is demonstrated in Fig. \ref{fig:8GaussAdam1500}: the discriminator's gradient pattern is more stable and similar to those of GAN-0GP and GAN-R1. % and Fig. \ref{figappx:8GaussAdam} in Appendix \ref{appx:ctSynthetic}.
%Fig. \ref{figappx:8GaussAdam} shows that when the model distribution move from one mode to another, the gradient field remains stable, implying that old information is preserved in the discriminator. 

\textbf{Continual learning algorithms} such as EWC \cite{ewc} and online EWC \cite{progressCompress} %and Synaptic Intelligence (SI) \cite{synapticIntel} 
prevent important knowledge of previous tasks from being overwritten by the new task. 
At the end of a task $\mathcal{T}^t$, online EWC computes the importance $\hat{\omega}_i^t$ of each parameter $\theta_i^t$ to the task and adds a regularization term to the loss function of task $\mathcal{T}^{t + 1}$:
\begin{eqnarray*}
\omega_i^t & = & \alpha \hat{\omega}_i^t + (1 - \alpha) \omega_i^{t - 1} \\
\mathcal{L}_{EWC}^{t + 1} & = & \mathcal{L}^{t + 1} + \lambda \sum_i \omega_i^t (\theta_i - \theta_i^t)^2
\end{eqnarray*}
where $\theta_i^t$ is the value of $\theta_i$ at the end of task $\mathcal{T}^t$, $\alpha$ balances the importance of the current task and previous tasks, $\omega_i^t$ accumulates the importance of $\theta_i$ throughout the training process. Because consecutive model distributions are similar, we consider a chunk of $\tau$ distributions as a task to the discriminator. The importance $\omega_i$ is computed every $\tau$ GAN training iteration. 
The regularizer prevents important weights from deviating too far from the values that are optimal to previous tasks while allowing less important weights to change more freely.
It helps the discriminator preserves important information about old distributions. 
\citeauthor{ganIsCont} independently proposed a similar way of adapting continual learning methods to GANs. Experiments in the paper showed that continual learning methods improve the quality of GANs.

\subsection{Introducing prior knowledge to the game}
In Dirac GAN, if the discriminator has a local maximum at the real datapoint then it can always classify the real and the fake datapoint correctly, regardless of location of the fake datapoint.
Because separating different fake distributions from the target distribution requires the same knowledge, that knowledge will not be erased from the discriminator.
We want to introduce to the game the knowledge that real datapoints should be local maxima.
R1 and 0GP are two ways to implement that.

\textbf{R1 regularizer} (the third row in Table \ref{tab:loss}) forces the gradients w.r.t. a real datapoint to be $\bm 0$, making it a local extremum of the discriminator. 
As the discriminator maximizes the score of real datapoints, real datapoints become local maxima of the discriminator.
Fig. \ref{fig:8GaussR11000} - \ref{fig:8GaussR15000} shows that real datapoints are always local maxima and the gradient pattern of the discriminator stay unchanged as $p_g$ moves toward $p_r$.
Fig. \ref{fig:ganr1} demonstrates the same effect of R1 on MNIST. Note that noisy images that are far away from the real images (e.g. $\bm x + k\hat{\bm u}$ for $k < -50$) have higher scores than real images. This is because no regularizer is applied to these noisy images.

\textbf{0GP regularizer} (the forth row in Table \ref{tab:loss}) pushes gradients w.r.t. datapoints on the line connecting a real datapoint $\bm x$ and a fake datapoint $\bm y$ toward $\bm 0$. 0GP forces the score to increase gradually as we move from $\bm y$ to $\bm x$. During training, $\bm x$ is paired with different $\bm y_i$. Thus, the score $D(\bm x)$ is greater than the scores of fake datapoints in a wider neighborhood. That fixes the problem of R1 and creates wider local maxima (Fig. \ref{fig:gan0gp}, \ref{fig:celeba}). \citeauthor{improveGeneralization} \cite{improveGeneralization} showed that GAN-0GP generalizes better than GAN-R1.
Although generalization is beyond the scope of this paper, we believe that the sharpness of the discriminator's landscape is related to its generalization capability. Prior works on generalization of neural networks \cite{flatminima} showed flat (wide) minima of the loss surface generalize better than sharp minima. %The loss of the generator is directly related to the discriminator's output (Table \ref{tab:loss}) and .
Creating discriminators with wide local maxima is a good way to improve GANs' generalizability.

\textbf{WGAN-GP}  (the first row in Table \ref{tab:loss}) uses 1-centered gradient penalty (1GP) which pushes gradients w.r.t. datapoints on the line connecting a real datapoint $\bm x$ and a fake datapoint $\bm y$ toward $\bm 1$, forcing the score to increase gradually from $\bm y$ to $\bm x$. Fig. \ref{fig:wgangp} shows that real datapoints are local maxima of the discriminator.
\citeauthor{wassersteinDiv} \cite{wassersteinDiv} showed that WGAN-0GP performs slightly better than WGAN-1GP. Our hypothesis is that 0GP creates wider maxima than 1GP as it make the score on the line from $\bm y$ to $\bm x$ to change more slowly. 

\textbf{Imbalanced weights for real and fake samples.} To prevent the discriminator from forgetting distant real datapoints, we propose to increase the weight of the loss for real datapoints:
\begin{eqnarray}
\mathcal{L}_D = \gamma \mathcal{L}_{real} + \mathcal{L}_{fake}
\end{eqnarray}
where $\gamma > 1$ is an empirically chosen hyper parameter, $\mathcal{L}_{real},\ \mathcal{L}_{fake}$ are the losses for real and fake samples, respectively. 
When $\gamma > 1$, the discriminator is penalized more if it assigns a low score to a real datapoint. The situation where real datapoints are local minima like in Fig. \ref{fig:dcgannsLand5k} or have low scores like in the blue boxes in Fig. \ref{fig:8Gauss3000} - \ref{fig:8Gauss3500} will less likely to happen.
Fig. \ref{fig:dcganimbaLand20k} shows that the new loss successfully helps the discriminator to make more real datapoints local maxima and thus improve fake samples' quality.
Table \ref{tab:dcganResult} shows the effectiveness of imbalanced loss on CIFAR-10 dataset: it significantly improves Inception Score \cite{improvedgan} and reduces the score's variance. The imbalanced loss is orthogonal to gradient penalties and can be used to improve gradient penalties (the last two rows in Table \ref{tab:dcganResult}).

\section{Conclusion}
Catastrophic forgetting is a important problem in GANs. 
It is directly related to mode collapse and non-convergence. 
Addressing catastrophic forgetting leads to better convergence and less mode collapse. 
Methods such as imbalanced loss, zero centered gradient penalties, optimizers with momentum, and continual learning are effective at preventing catastrophic forgetting in GANs.
0GP helps GANs to converge to good local equilibria where real datapoints are wide local maxima of the discriminator.
The gradient penalty is a promising method for improving generalizability of GANs.
\bibliography{aistats2020}

\addtocounter{figure}{-1}
\addtocounter{table}{-1}
\refstepcounter{table}\label{LASTTABLE}
\refstepcounter{figure}\label{LASTFIGURE}

\pagebreak
\appendix
\section{Experiments on synthetic datasets}
\label{appx:ctSynthetic}

\begin{figure*}[!ht]
\centering
\subfloat[Iter. 0]{\includegraphics[width=0.19\textwidth]{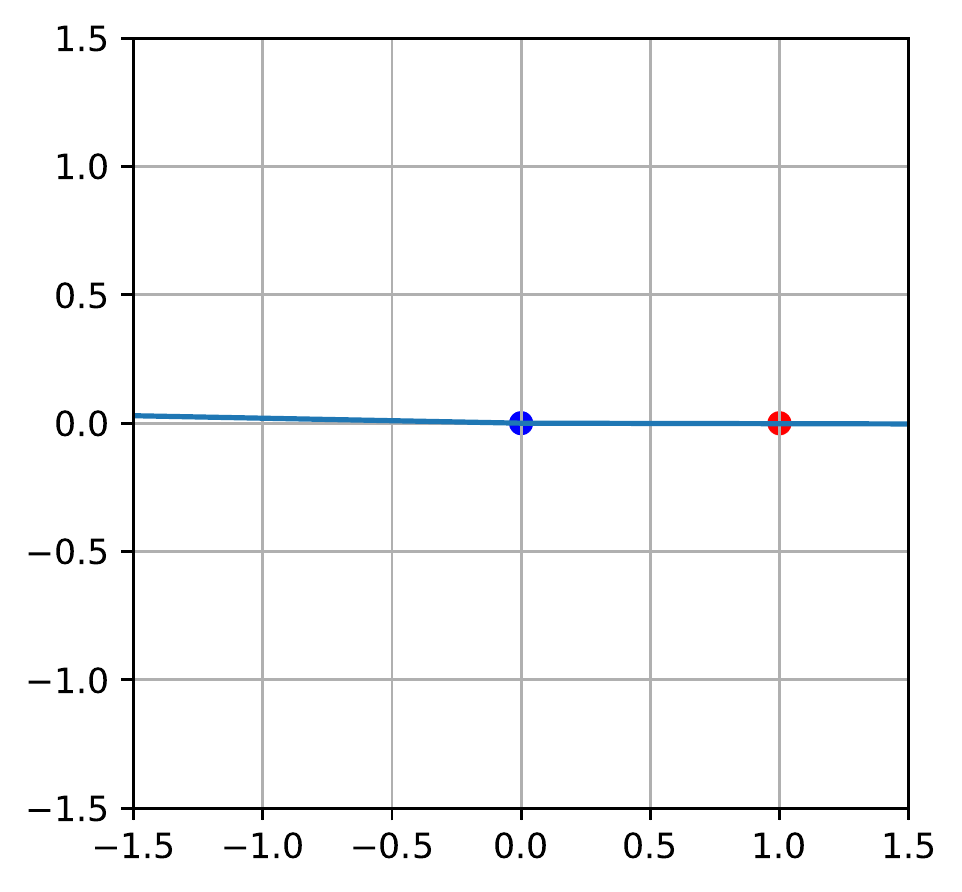}\label{figappx:dirac1SampleHigh0}}
\subfloat[Iter. 10]{\includegraphics[width=0.19\textwidth]{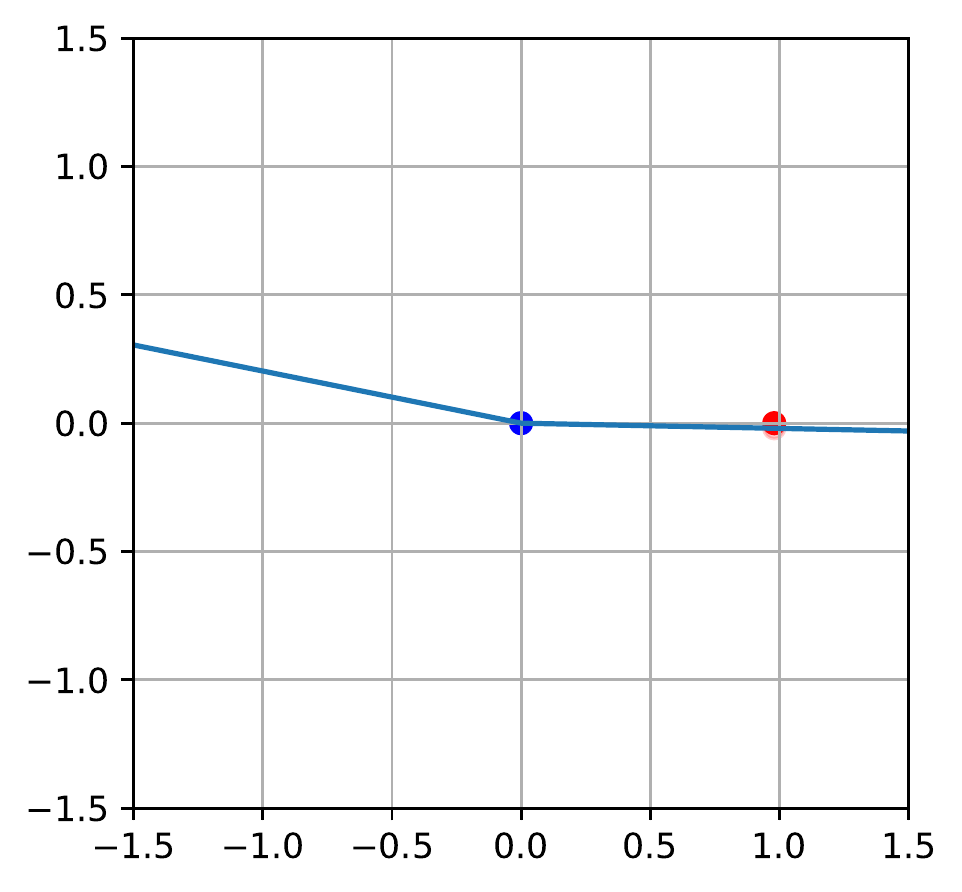}\label{figappx:dirac1SampleHigh10}}
\subfloat[Iter. 70]{\includegraphics[width=0.19\textwidth]{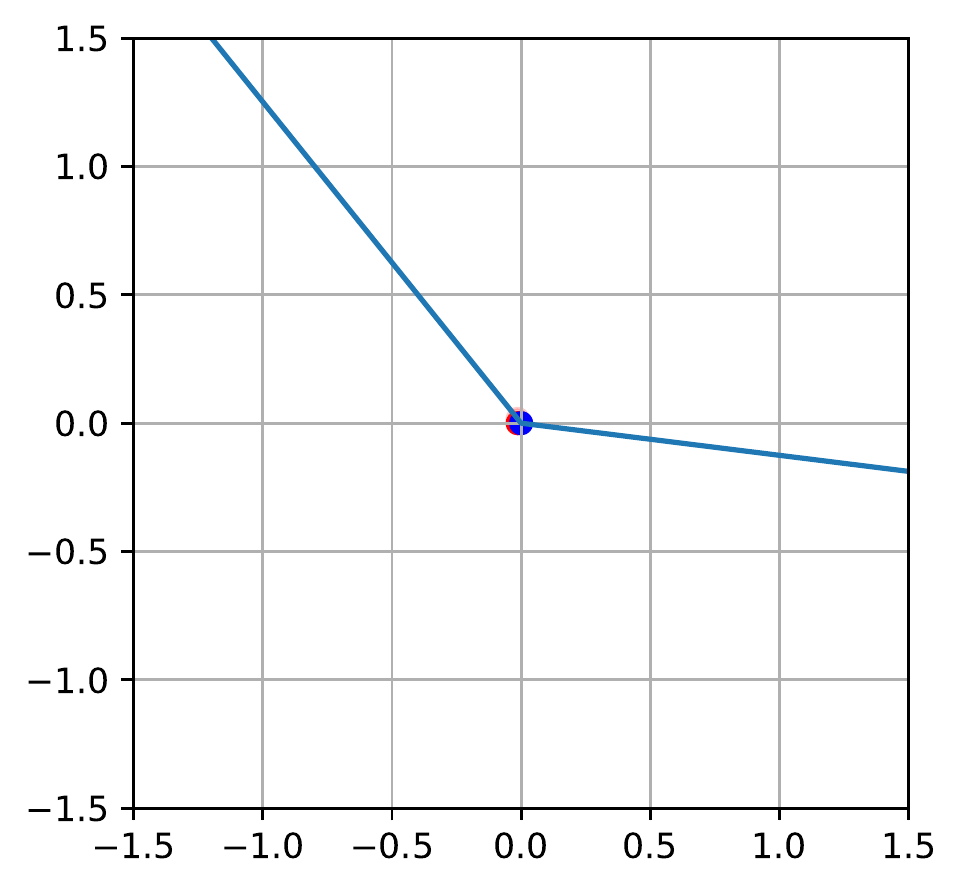}\label{figappx:dirac1SampleHigh70}}
\subfloat[Iter. 78]{\includegraphics[width=0.19\textwidth]{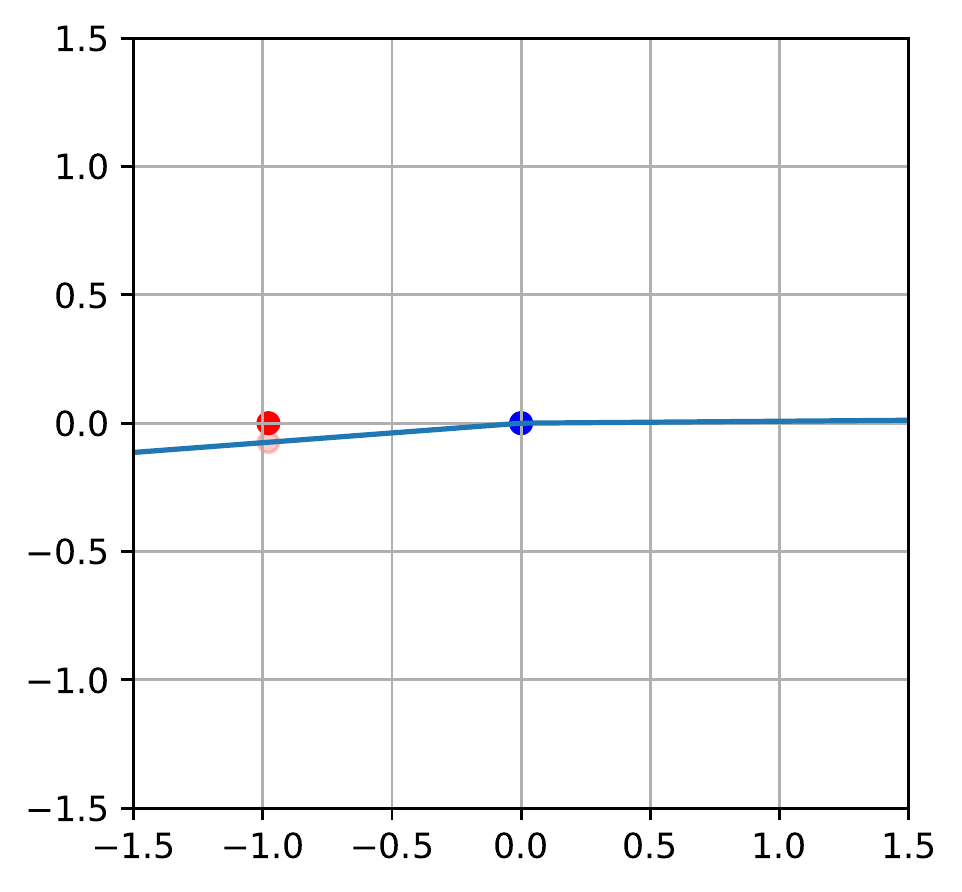}\label{figappx:dirac1SampleHigh78}}

\subfloat[Iter. 87]{\includegraphics[width=0.19\textwidth]{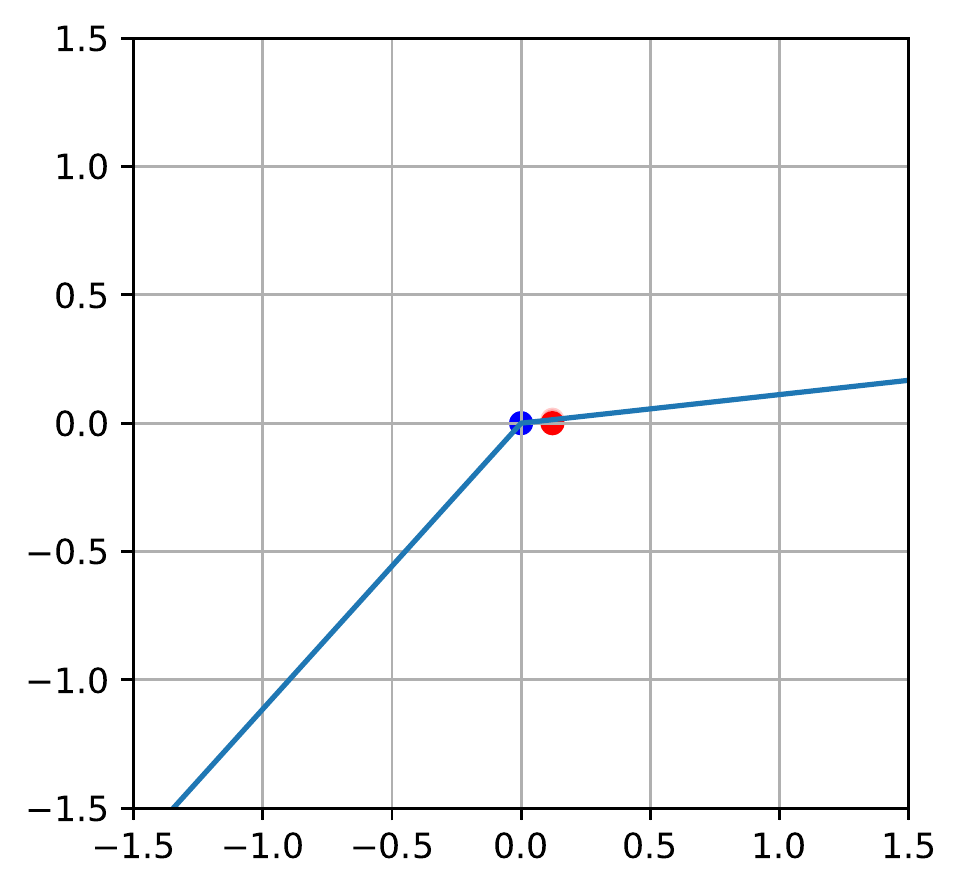}\label{figappx:dirac1SampleHigh87}}
\subfloat[Iter. 150]{\includegraphics[width=0.19\textwidth]{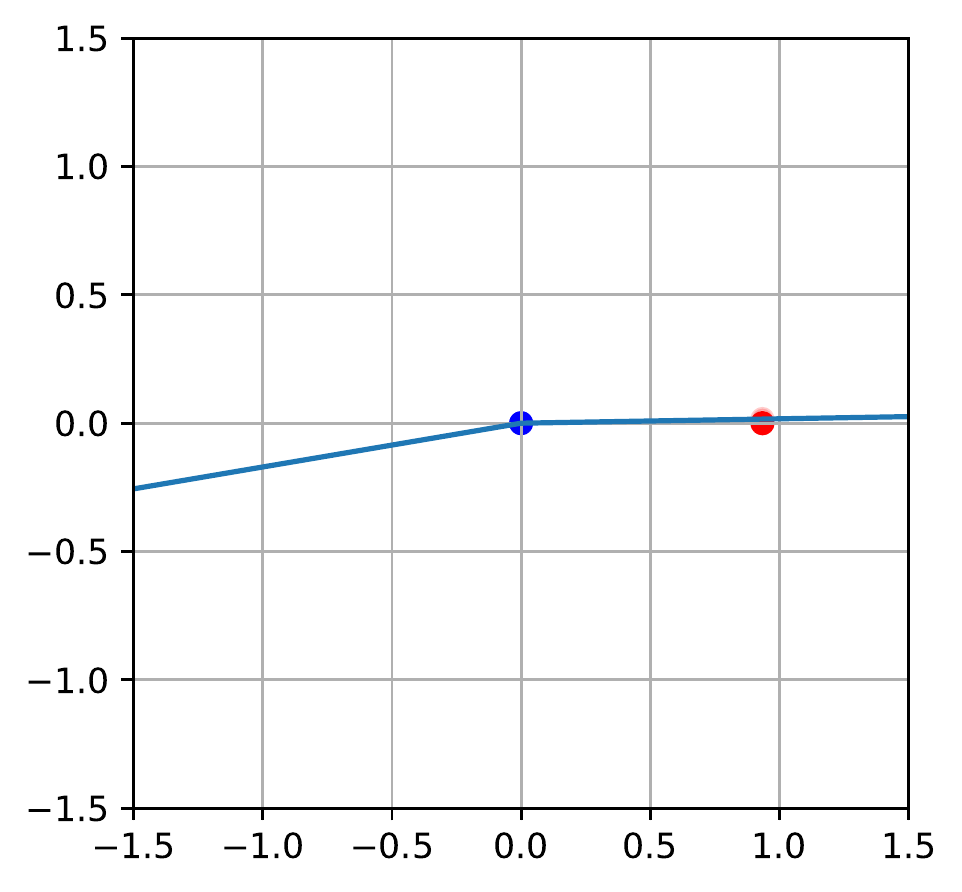}\label{figappx:dirac1SampleHigh150}}
\subfloat[Iter. 175]{\includegraphics[width=0.19\textwidth]{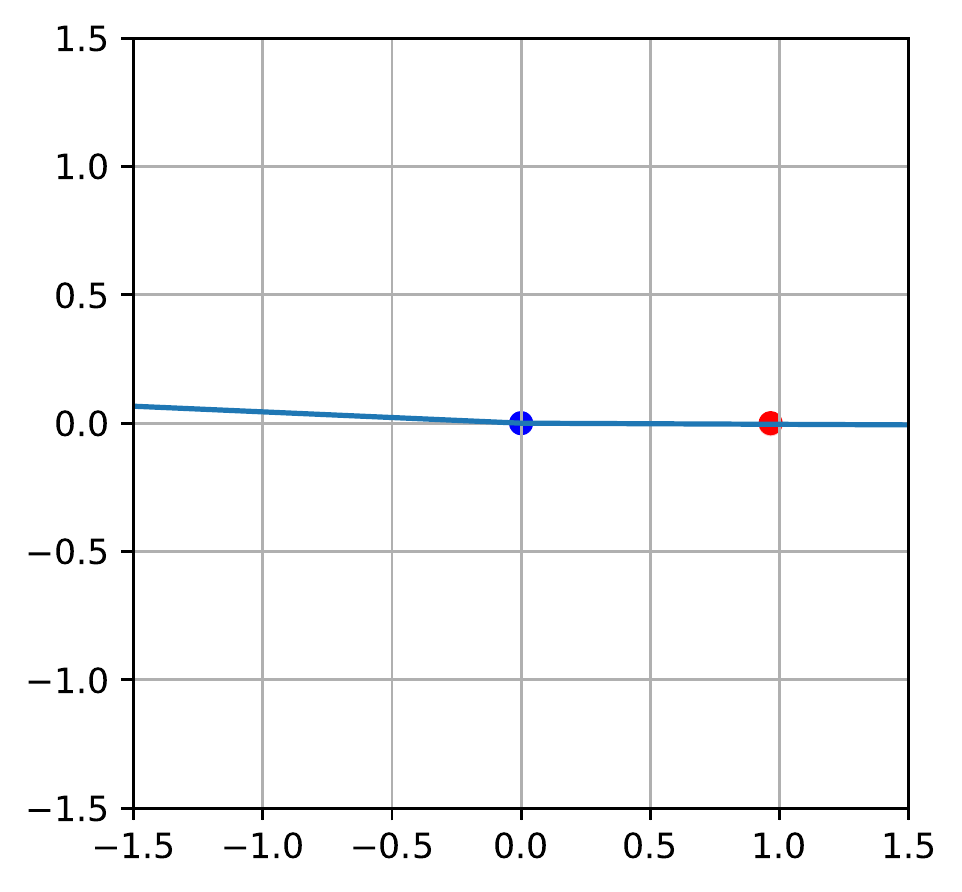}\label{figappx:dirac1SampleHigh175}}
\subfloat[Iter. 250]{\includegraphics[width=0.19\textwidth]{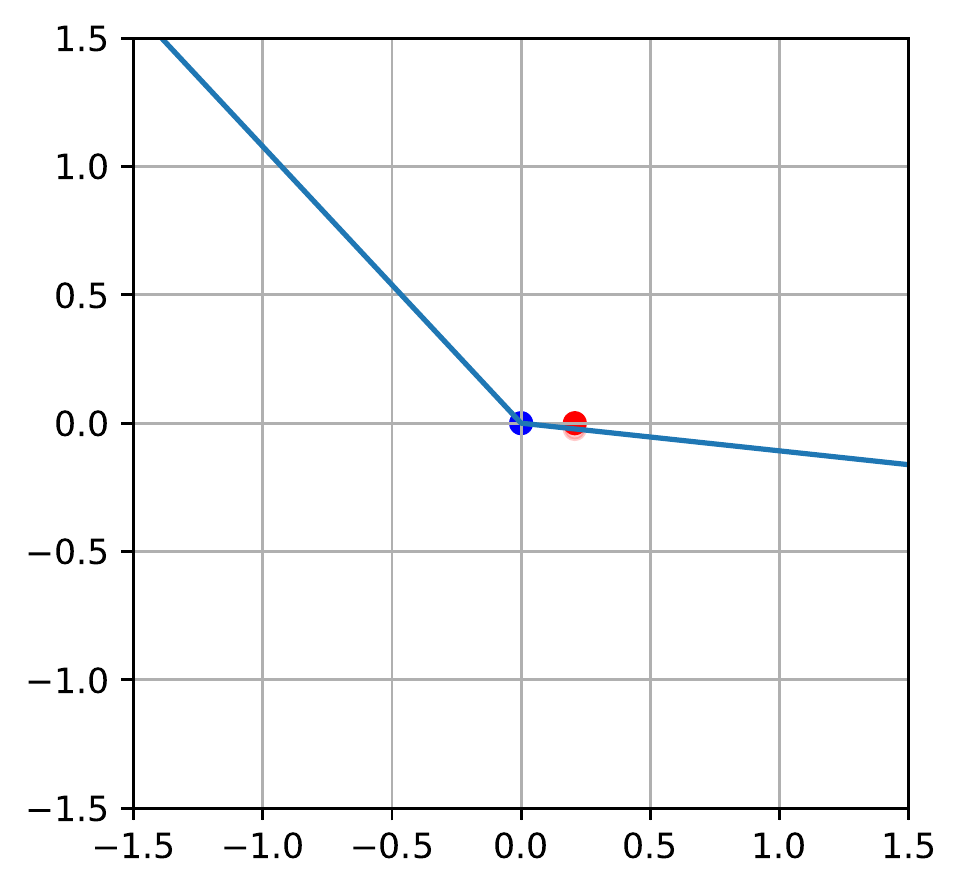}\label{figappx:dirac1SampleHigh250}}

\caption{Catastrophic forgetting in high capacity Dirac GAN. The discriminator is a 1 hidden layer neural network with Leaky ReLU activation function and 2 hidden neurons. Although the discriminator has enough capacity to become a non-monotonic function, catastrophic forgetting makes it a monotonic function. High capacity Dirac GAN still oscillates around the equilibrium.}
\label{figappx:dirac1SampleHigh}
\end{figure*}

\begin{figure*}[ht!]
\centering
\subfloat[Iteration 0]{\includegraphics[width=0.33\textwidth]{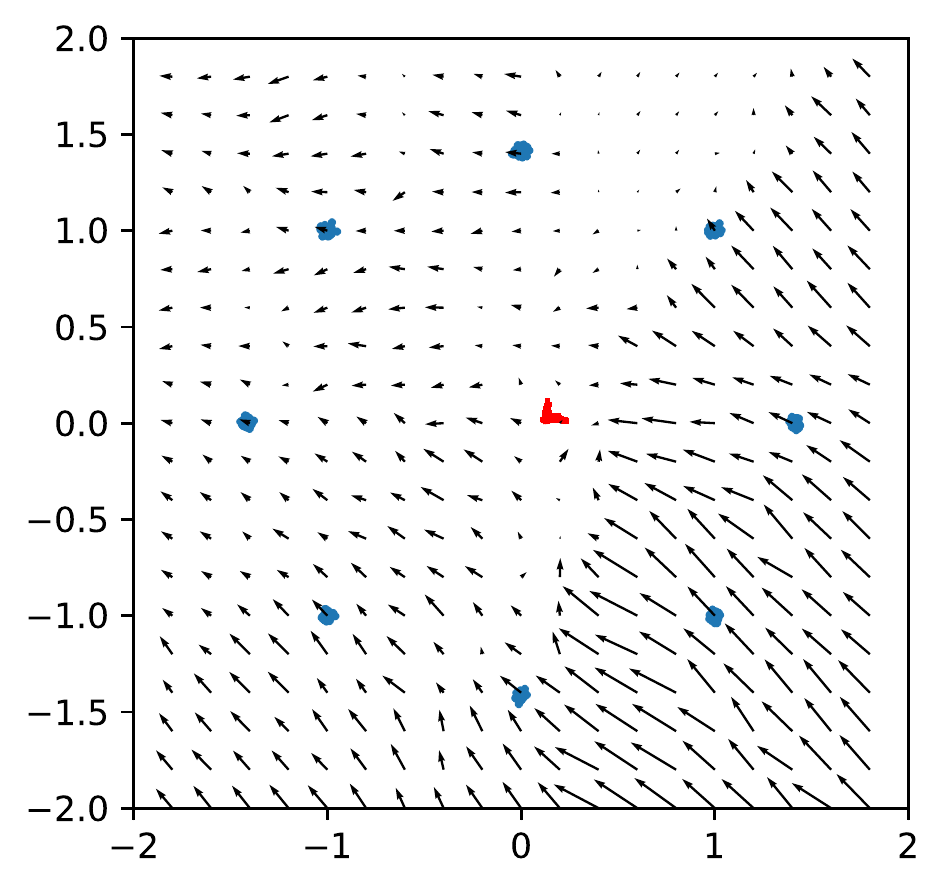}{\label{figappx:8Gauss0}}}
\subfloat[Iteration 3000]{\includegraphics[width=0.33\textwidth]{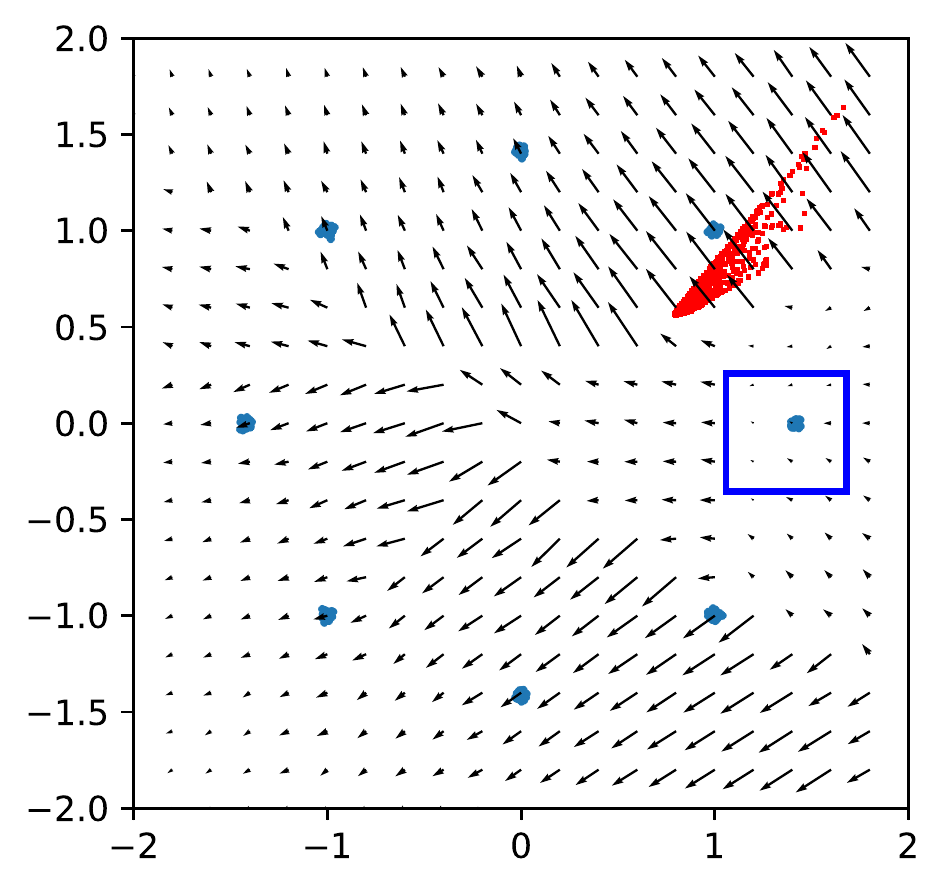}{\label{figappx:8Gauss3000}}}
\subfloat[Iteration 3500]{\includegraphics[width=0.33\textwidth]{figs/gan_8Gaussians_gradfield_center_1.00_alpha_None_lambda_0.00_lrg_0.00300_lrd_0.00300_nhidden_64_scale_2.00_optim_SGD_gnlayers_2_dnlayers_1_gradweight_0.1000_ncritic_1/fig_03500_new_ann.pdf}\label{figappx:8Gauss3500}}

\subfloat[Iteration 4000]{\includegraphics[width=0.33\textwidth]{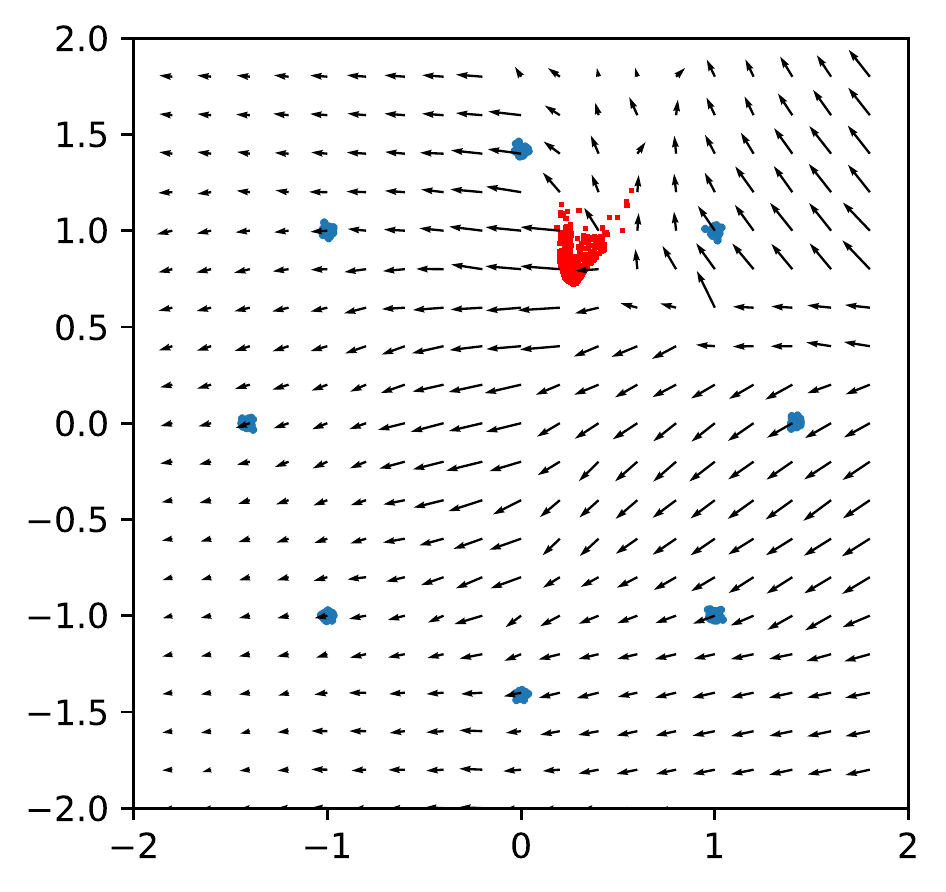}{\label{figappx:8Gauss4000}}}
\subfloat[Iteration 4500]{\includegraphics[width=0.33\textwidth]{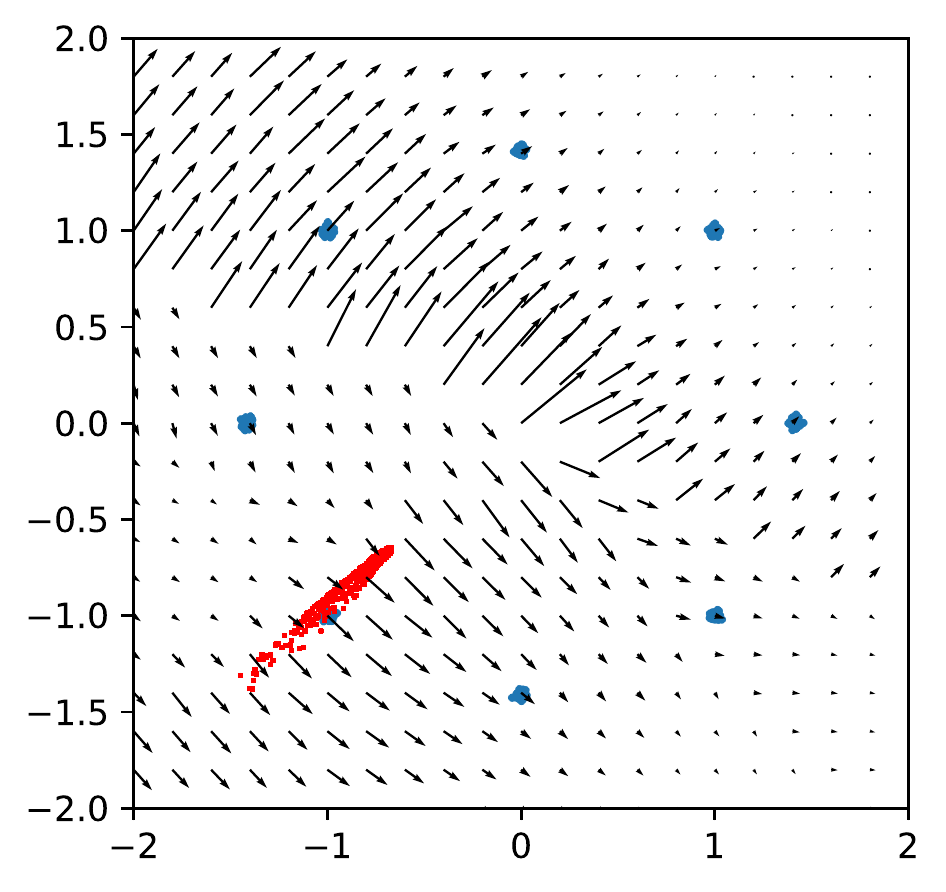}{\label{figappx:8Gauss4500}}}
\subfloat[Iteration 5000]{\includegraphics[width=0.33\textwidth]{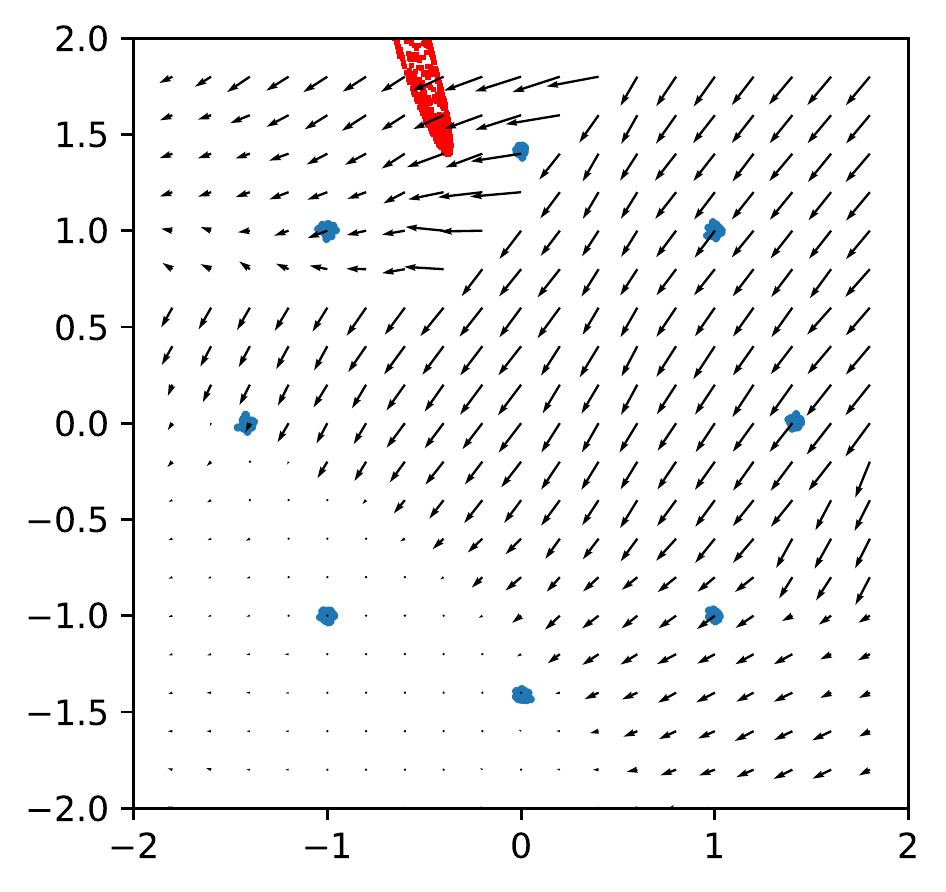}{\label{figappx:8Gauss5000}}}

\subfloat[Iteration 5500]{\includegraphics[width=0.33\textwidth]{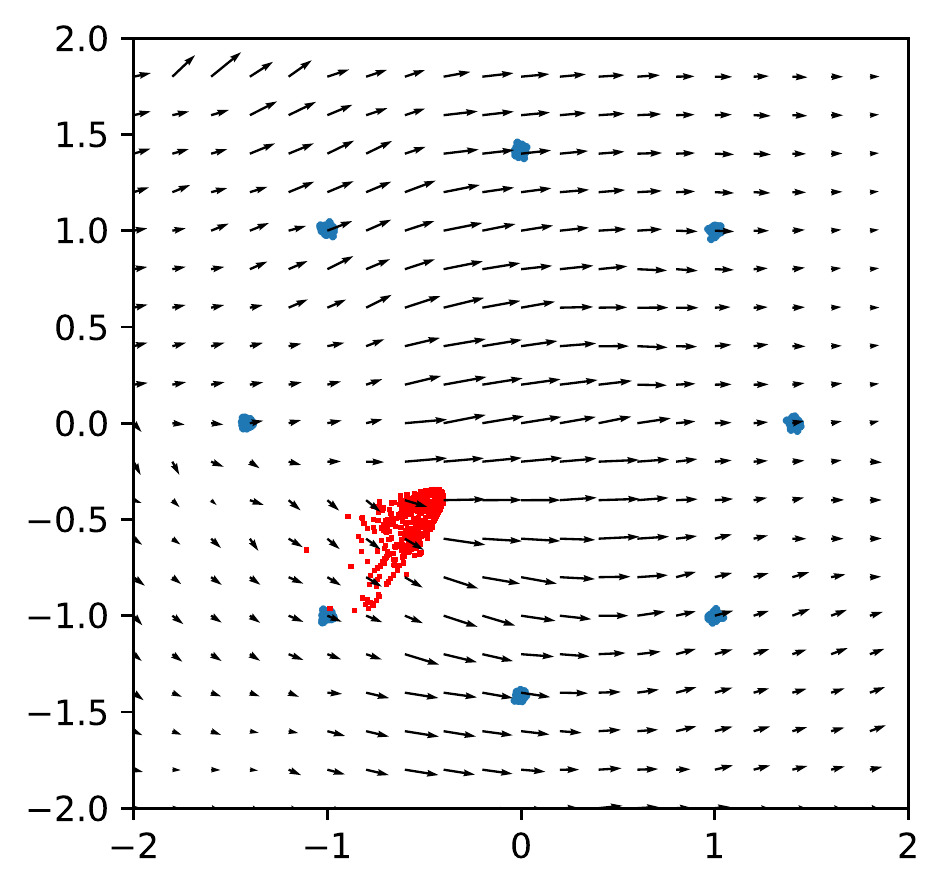}{\label{figappx:8Gauss5500}}}
\subfloat[Iteration 10000]{\includegraphics[width=0.33\textwidth]{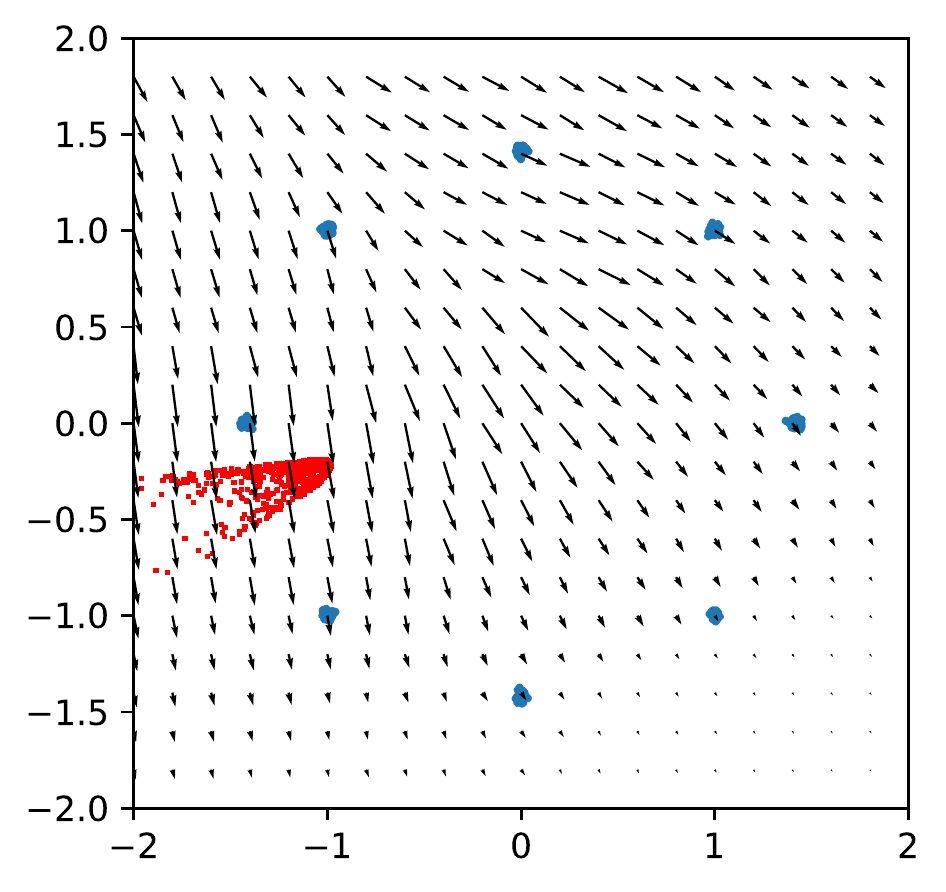}\label{figappx:8Gauss10000}}
%\subfloat[Iteration 15000]{\includegraphics[width=0.33\textwidth]{gan_8Gaussians_gradfield_center_1.00_alpha_None_lambda_0.00_lrg_0.00300_lrd_0.00300_nhidden_64_scale_2.00_optim_SGD_gnlayers_2_dnlayers_1_gradweight_0.1000_ncritic_1/fig_15000.pdf}\label{figappx:8Gauss15000}}
\subfloat[Iteration 20000]{\includegraphics[width=0.33\textwidth]
{figs/gan_8Gaussians_gradfield_center_1.00_alpha_None_lambda_0.00_lrg_0.00300_lrd_0.00300_nhidden_64_scale_2.00_optim_SGD_gnlayers_2_dnlayers_1_gradweight_0.1000_ncritic_1/fig_20000_ann.pdf}{\label{figappx:8Gauss20000}}} 
\caption{Catastrophic forgetting on the 8 Gaussian dataset.}
\label{figappx:8Gauss}
\end{figure*}

\begin{figure*}[ht!]
\subfloat[Iteration 0]{\includegraphics[width=0.33\textwidth]{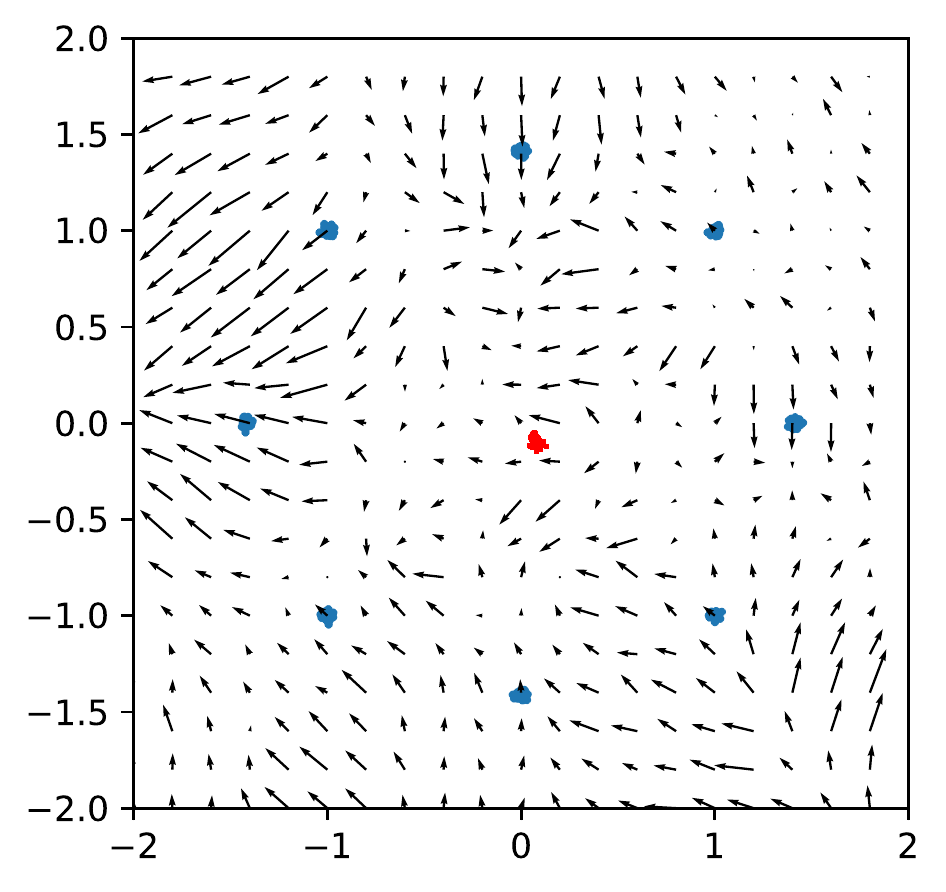}{\label{figappx:8GaussR10}}}
\subfloat[Iteration 500]{\includegraphics[width=0.33\textwidth]{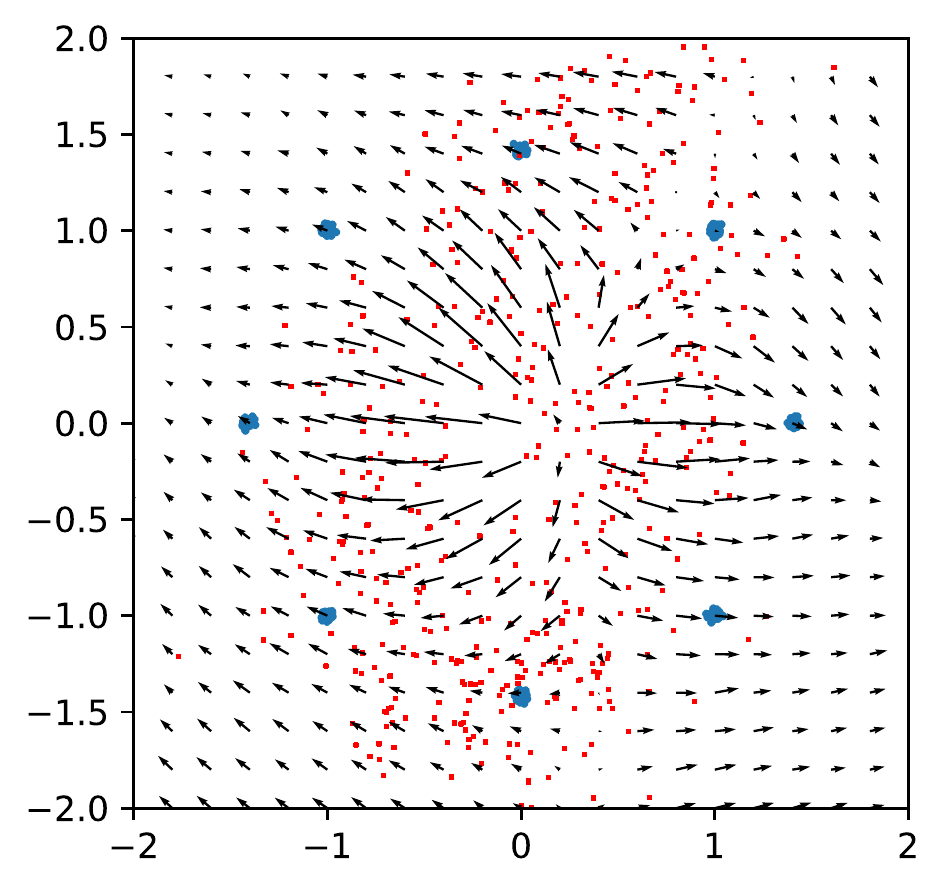}{\label{figappx:8GaussR1500}}}
\subfloat[Iteration 1000]{\includegraphics[width=0.33\textwidth]{/continual-gan-8Gaussians-gradfield-center-0.00-alpha-1.0-lambda-10.00-lrg-0.00300-lrd-0.00300-nhidden-512-scale-2.00-optim-SGD-gnlayers-2-dnlayers-2-gradweight-0.0100-discount-0.9900-batch-1-ewcweight-0.00-ncritic-1/fig-01000.pdf}{\label{figappx:8GaussR11000}}} 

\subfloat[Iteration 2500]{\includegraphics[width=0.33\textwidth]{/continual-gan-8Gaussians-gradfield-center-0.00-alpha-1.0-lambda-10.00-lrg-0.00300-lrd-0.00300-nhidden-512-scale-2.00-optim-SGD-gnlayers-2-dnlayers-2-gradweight-0.0100-discount-0.9900-batch-1-ewcweight-0.00-ncritic-1/fig-02500.pdf}{\label{figappx:8GaussR12500}}} 
\subfloat[Iteration 5000]{\includegraphics[width=0.33\textwidth]{/continual-gan-8Gaussians-gradfield-center-0.00-alpha-1.0-lambda-10.00-lrg-0.00300-lrd-0.00300-nhidden-512-scale-2.00-optim-SGD-gnlayers-2-dnlayers-2-gradweight-0.0100-discount-0.9900-batch-1-ewcweight-0.00-ncritic-1/fig-05000.pdf}{\label{figappx:8GaussR15000}}}
\subfloat[Iteration 10000]{\includegraphics[width=0.33\textwidth]{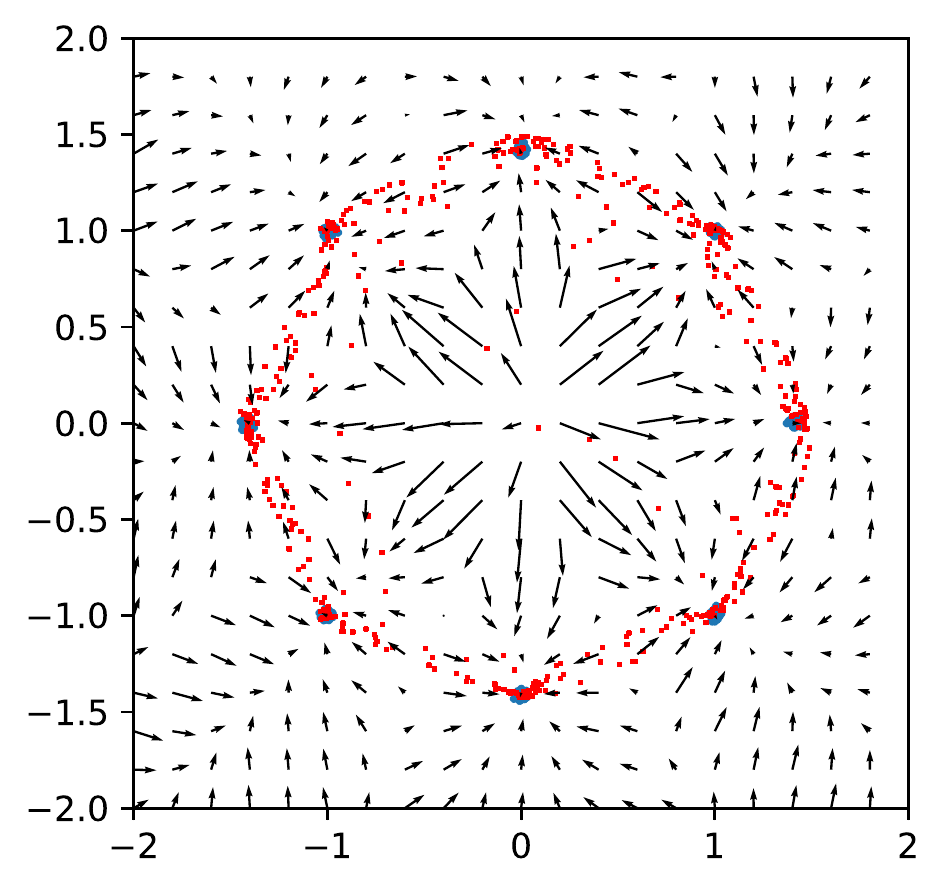}{\label{figappx:8GaussR110000}}} 

\caption{GAN-R1 with $\lambda=10$ on 8 Gaussian dataset.}
\end{figure*}

\begin{figure*}[!ht]
\centering
\subfloat[Iter. 0]{\includegraphics[width=0.30\textwidth]{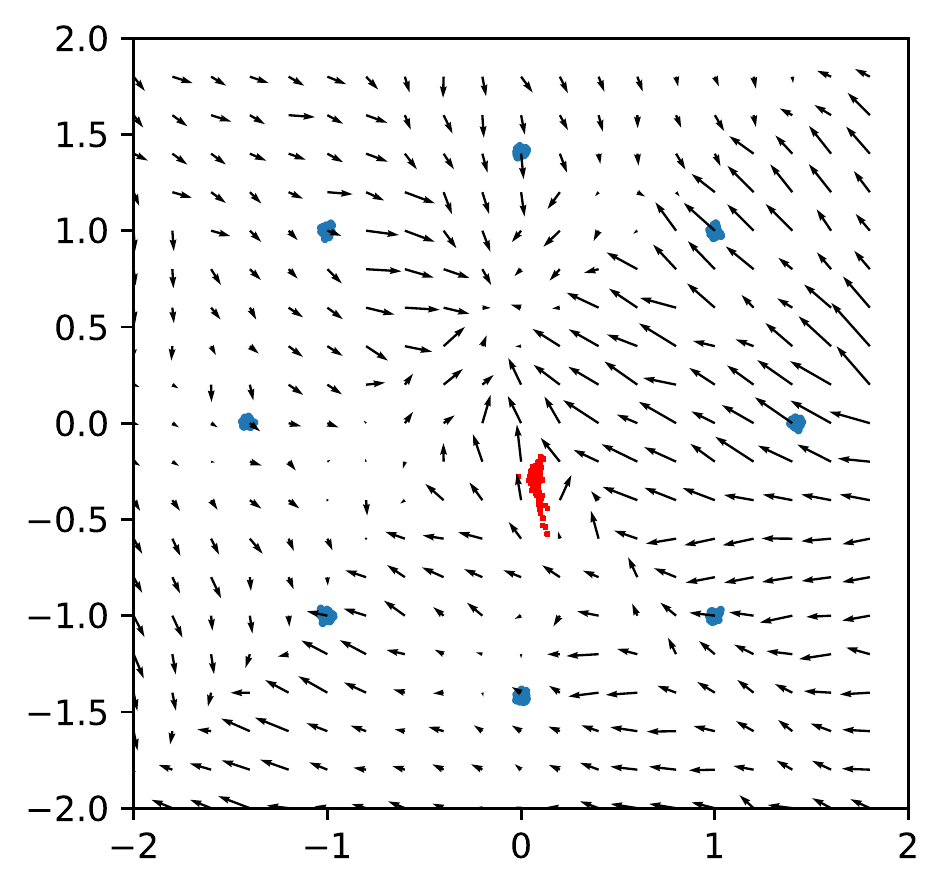}}
\subfloat[Iter. 400]{\includegraphics[width=0.30\textwidth]{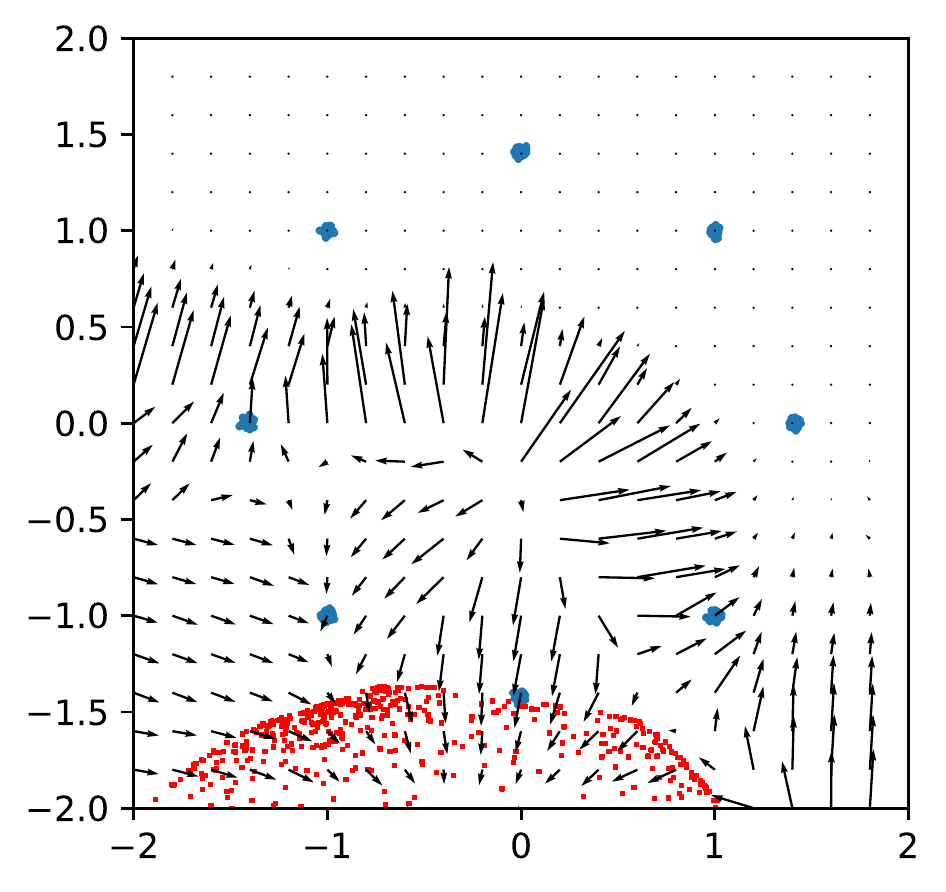}}
\subfloat[Iter. 600]{\includegraphics[width=0.30\textwidth]{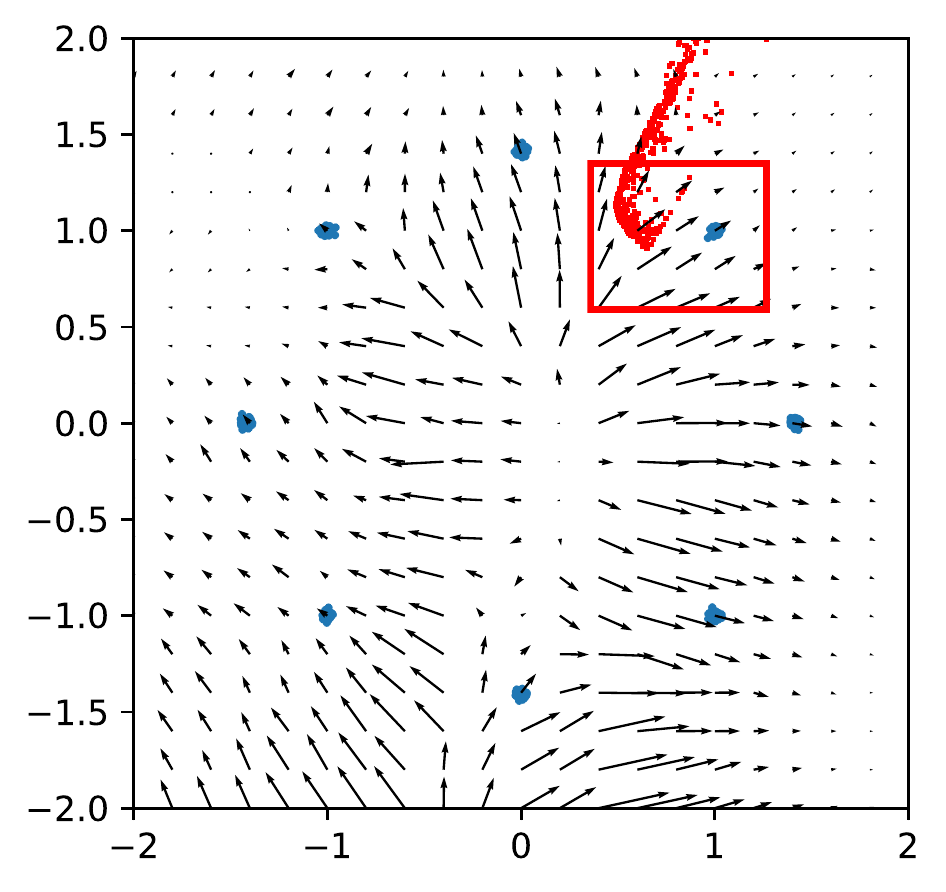}}

\subfloat[Iter. 1500]{\includegraphics[width=0.30\textwidth]{figs/gan_8Gaussians_gradfield_center_0.00_alpha_None_lambda_0.00_lrg_0.00300_lrd_0.00300_nhidden_64_scale_2.00_optim_Adam_gnlayers_1_dnlayers_1_gradweight_0.0100_dweight_0.1000_gweight_0.0100_ncritic_1/fig_01500.pdf}}

\caption{Evolution sequence of GAN-NS with Adam on the 8 Gaussian dataset. The gradient pattern is much more stable than that of GAN-NS with SGD. Note that the gradients in the red box still point toward the real datapoint despite the fact that the fake datapoints are close.}
\label{figappx:8GaussAdam}
\end{figure*}

\section{Landscapes of different GANs}
\label{appx:landscapes}
This section includes figures for different GANs. The general configuration for all experiments are shown in Table \ref{tab:config}. Hyper parameters specific to each experiment is shown in the caption of the corresponding figure. 

In each figure, the 'Real' subfloat shows real samples from MNIST dataset. Each cell in a 'Landscape' subfloat shows a slice of the landscape - the value of $f(k),\ k \in [-100, 100]$, for the corresponding real sample at the specified iteration. Each 'Generated' subfloat shows the generated samples at that iteration.

\begin{table*}
\centering
\begin{tabular}{|c|c|}
\hline
Architecture & 3 hidden layer MLP \\
\hline
Hidden layer activation & ReLU \\
\hline 
Output layer activation & Sigmoid for GAN-NS, Linear for WGAN \\
\hline
Number of hidden neurons & 512 \\
\hline
Latent dimensionality & 50 \\
\hline
Optimizer & ADAM with $\beta-1 = 0.5, \beta-2 = 0.99$ \\
\hline
Learning rate & $3\times 10^{-4}$ \\
\hline
Batch size & 64 \\
\hline
\end{tabular}
\caption{Experiments configuration}
\label{tab:config}
\end{table*}

\begin{figure*}
\subfloat[$k = -100$]{
\includegraphics[width=0.33\textwidth]{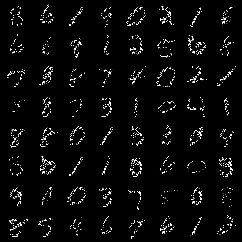}}
\subfloat[$k = -50$]{
\includegraphics[width=0.33\textwidth]{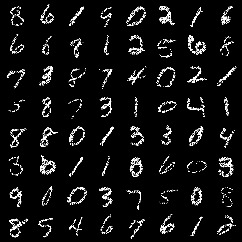}}
\subfloat[$k = -20$]{
\includegraphics[width=0.33\textwidth]{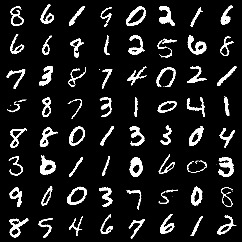}}

\subfloat[$k = -10$]{
\includegraphics[width=0.33\textwidth]{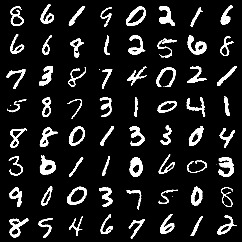}}
\subfloat[$k = 0$]{
\includegraphics[width=0.33\textwidth]{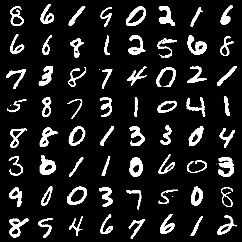}}
\subfloat[$k = 10$]{
\includegraphics[width=0.33\textwidth]{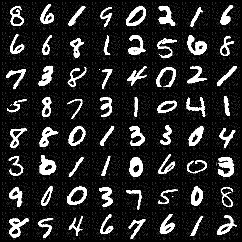}}

\subfloat[$k = 20$]{
\includegraphics[width=0.33\textwidth]{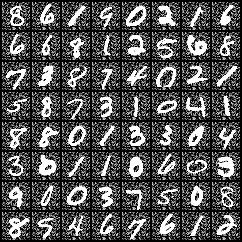}}
\subfloat[$k = 50$]{
\includegraphics[width=0.33\textwidth]{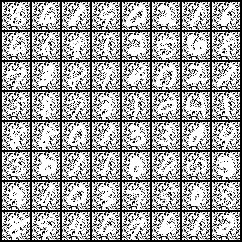}}
\subfloat[$k = 100$]{
\includegraphics[width=0.33\textwidth]{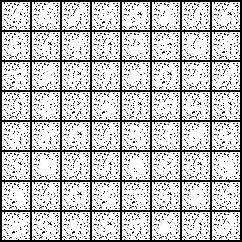}}

\caption{Real examples with different levels of noise.}
\label{figappx:realnoise}
\end{figure*}

\begin{figure*}
\centering
\subfloat[Real]{
\includegraphics[width=0.44\textwidth]{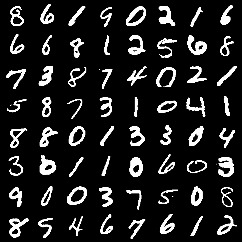}}
\subfloat[Landscape 5000]{
\includegraphics[width=0.44\textwidth]{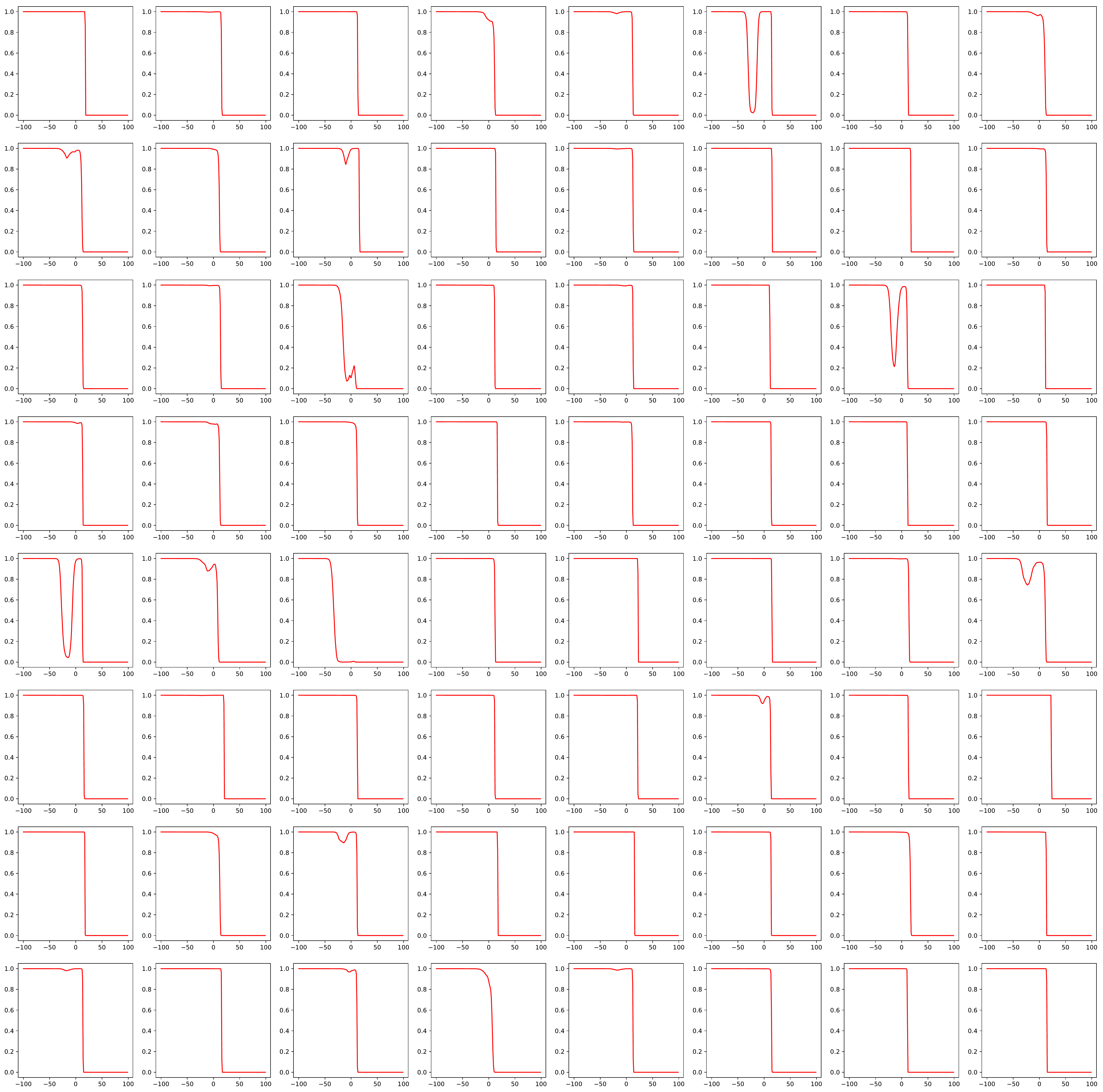}
} \\

\subfloat[Generated 50000]{
\includegraphics[width=0.44\textwidth]{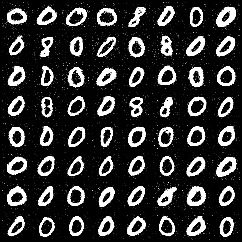}
}
\subfloat[Landscape 50000]{
\includegraphics[width=0.44\textwidth]{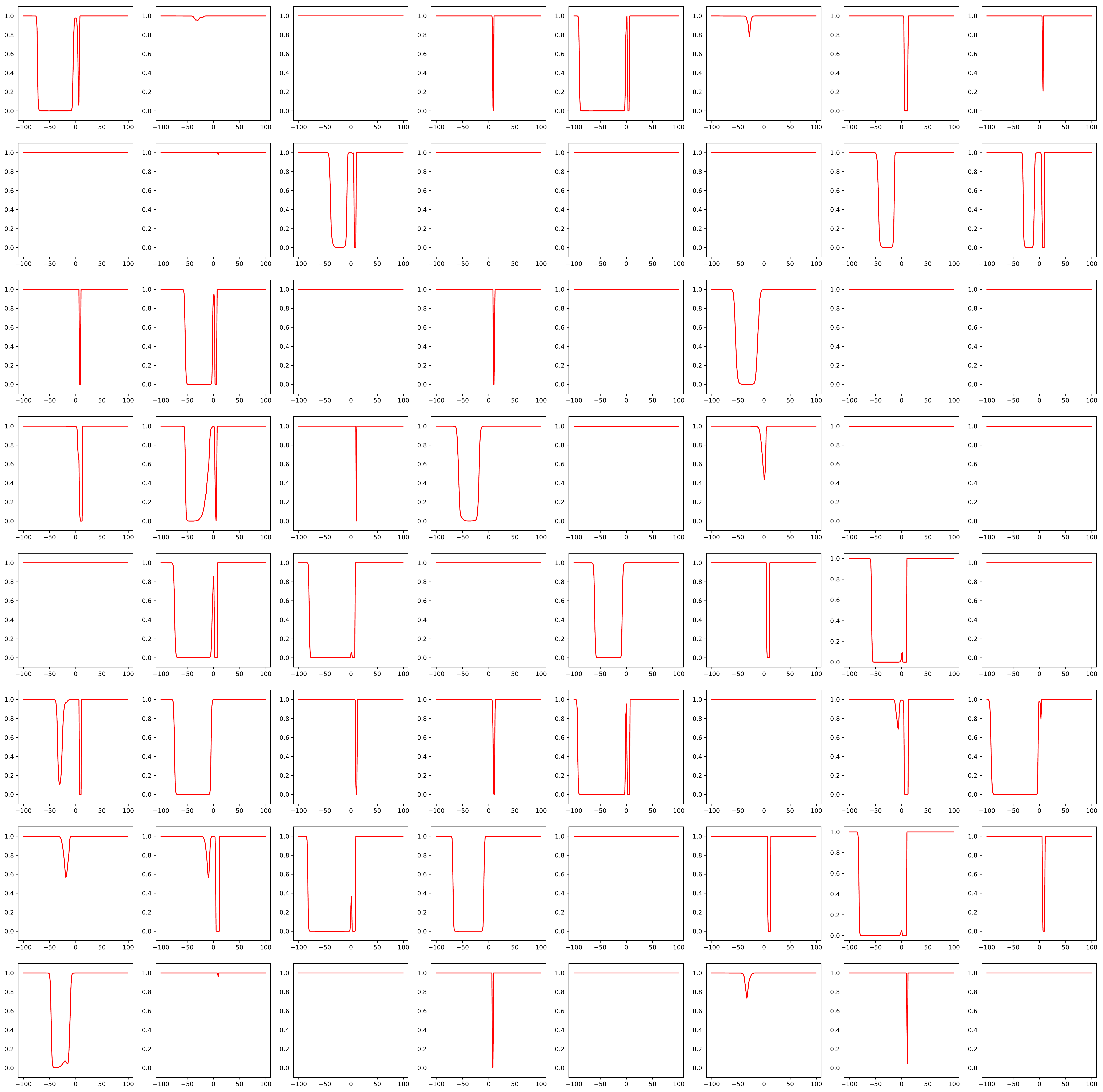}
} \\

\subfloat[Generated 100000]{
\includegraphics[width=0.44\textwidth]{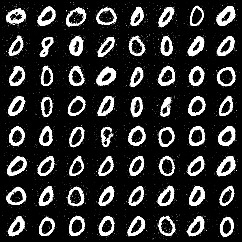}
}
\subfloat[Landscape 100000]{
\includegraphics[width=0.44\textwidth]{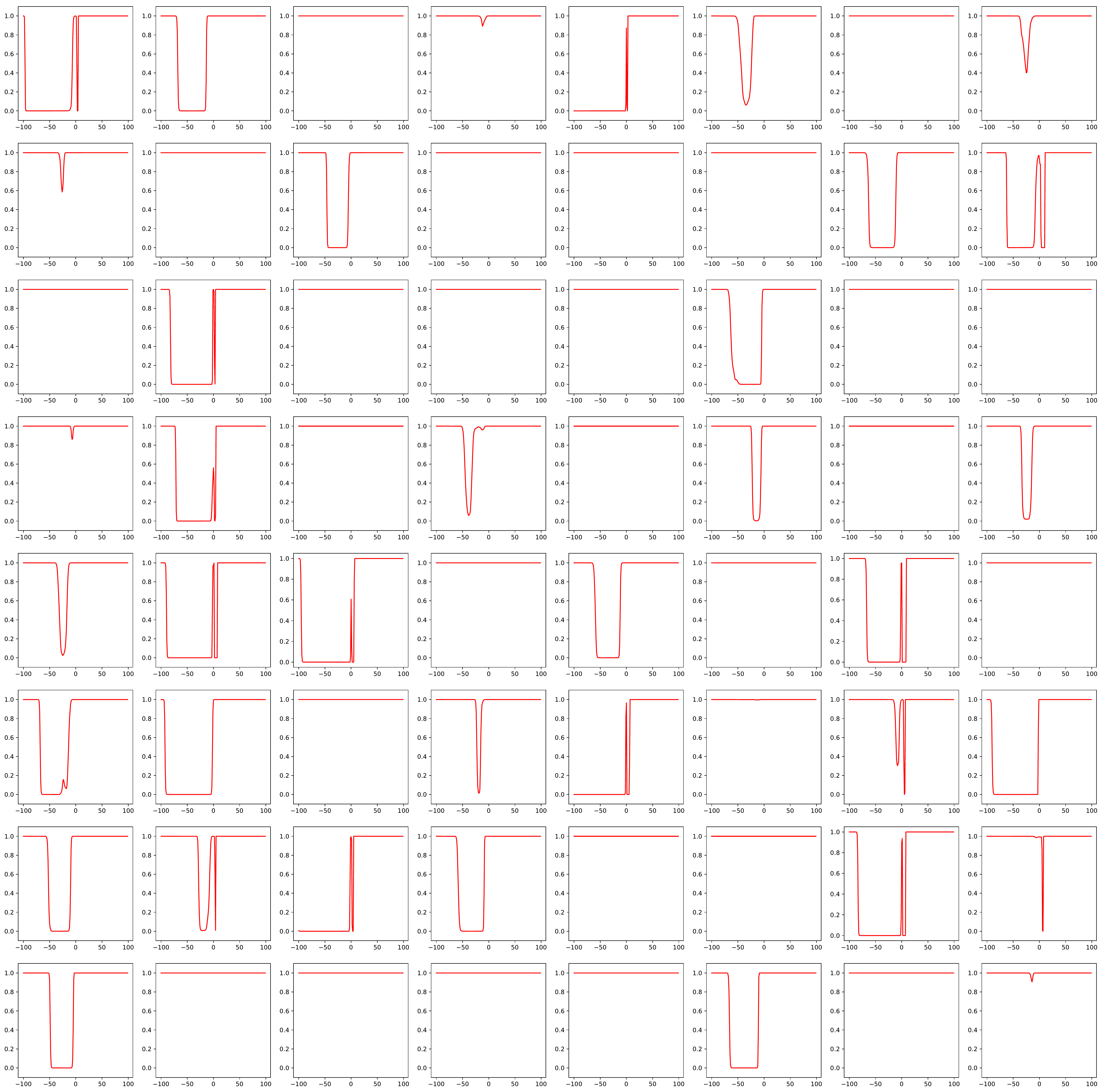}
} \\
\phantomcaption
\end{figure*}

\begin{figure*}
\centering
\ContinuedFloat
\subfloat[Generated 200000]{
\includegraphics[width=0.44\textwidth]{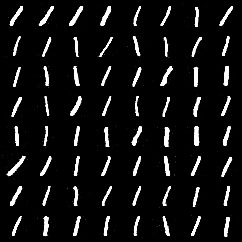}
}
\subfloat[Landscape 200000]{
\includegraphics[width=0.44\textwidth]{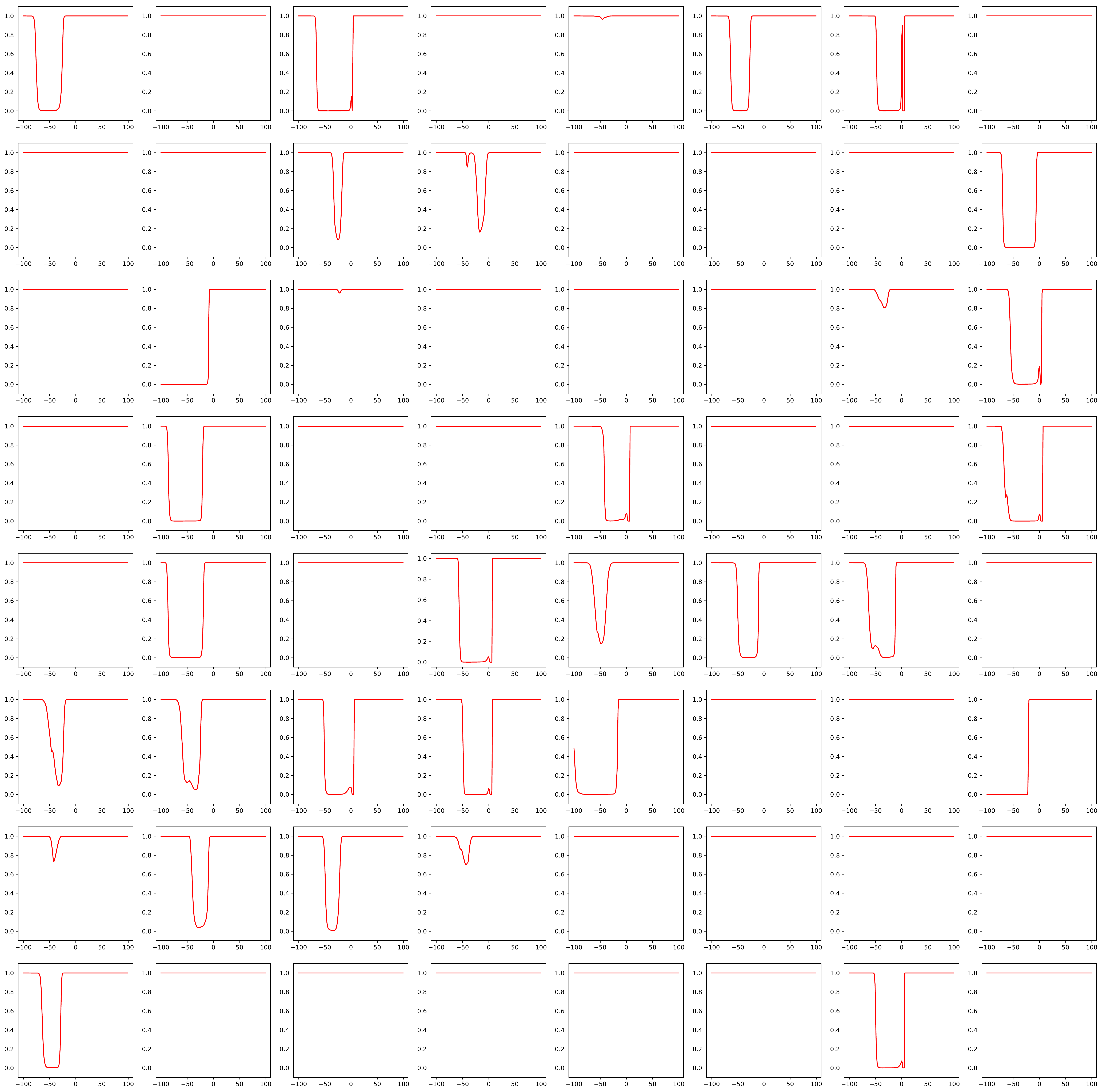}
}
\caption{GAN-NS}
\label{figappx:ganns}
\end{figure*}

\begin{figure*}
\centering
\subfloat[Real]{
\includegraphics[width=0.44\textwidth]{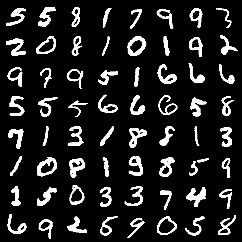}
}
\subfloat[Landscape 5000]{
\includegraphics[width=0.44\textwidth]{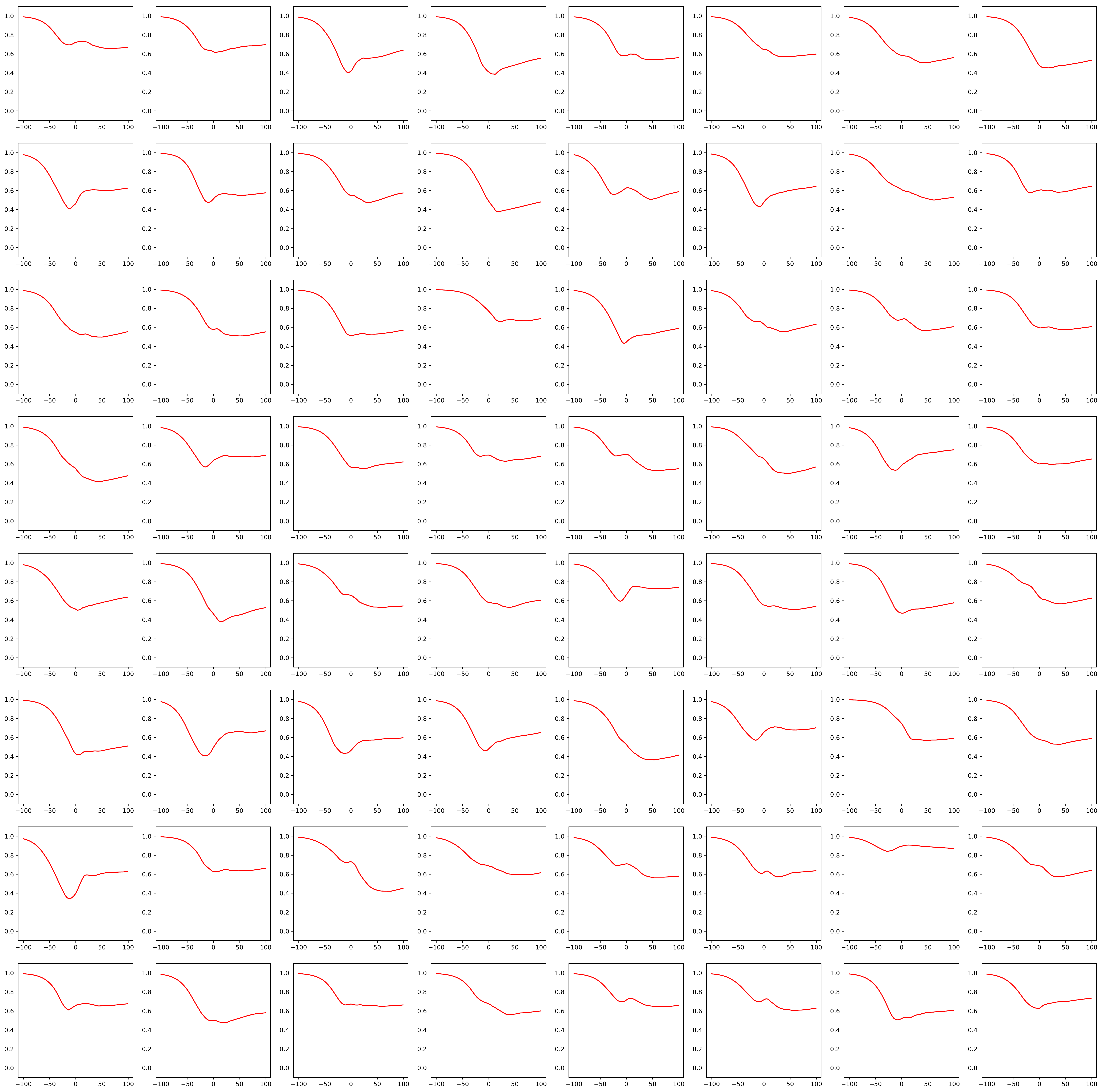}
}

\subfloat[Generated 50000]{
\includegraphics[width=0.44\textwidth]{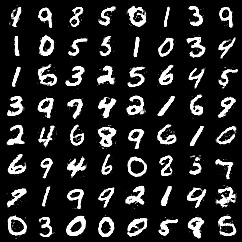}
}
\subfloat[Landscape 50000]{
\includegraphics[width=0.44\textwidth]{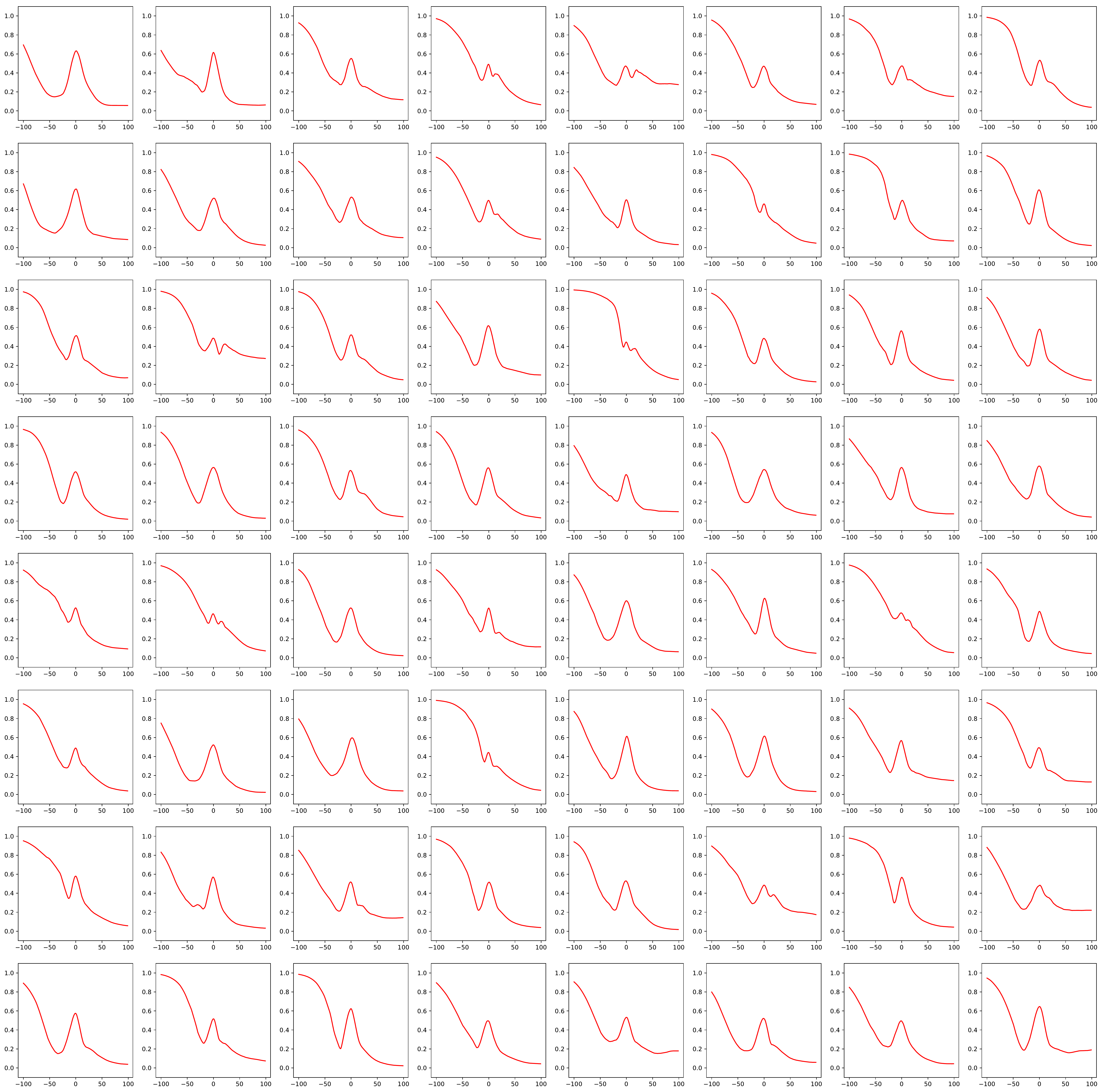}
}

\subfloat[Generated 100000]{
\includegraphics[width=0.44\textwidth]{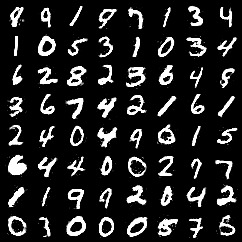}
}
\subfloat[Landscape 100000]{
\includegraphics[width=0.44\textwidth]{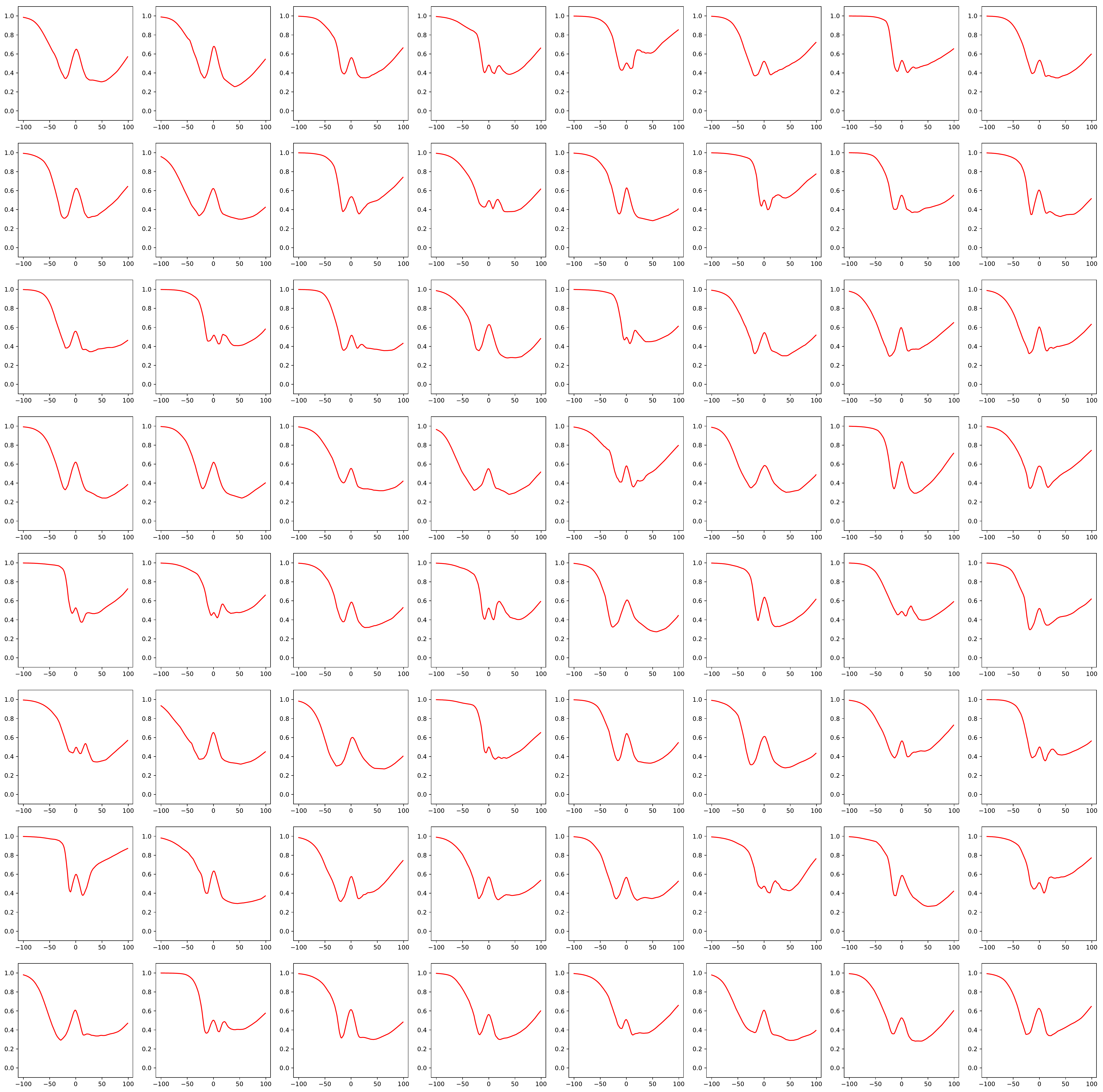}
}
\phantomcaption
\end{figure*}

\begin{figure*}
\centering
\ContinuedFloat
\subfloat[Generated 200000]{
\includegraphics[width=0.44\textwidth]{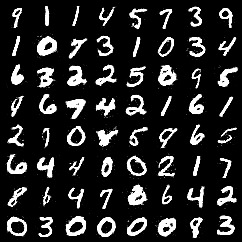}
}
\subfloat[Landscape 200000]{
\includegraphics[width=0.44\textwidth]{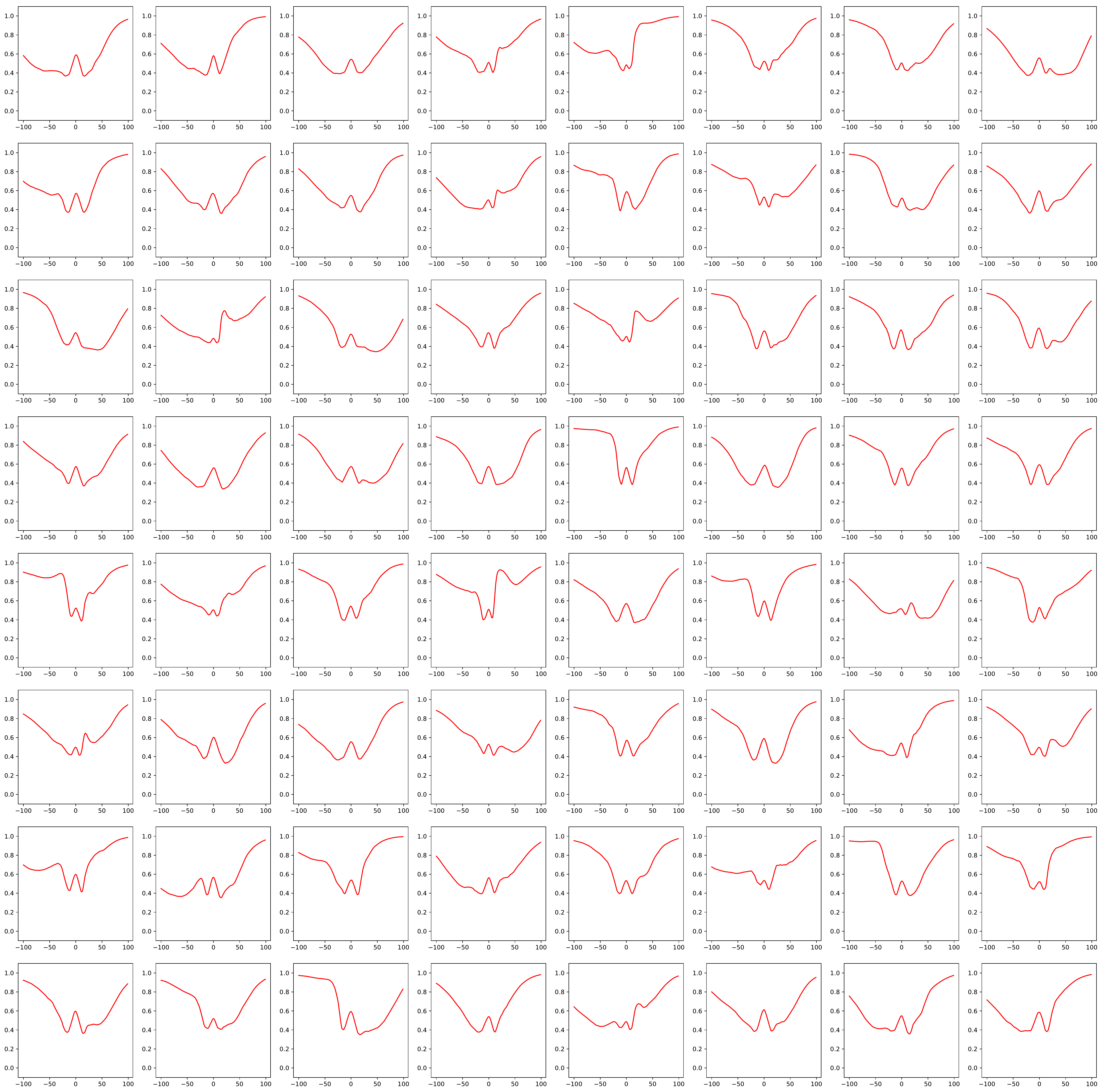}
}

\caption{GAN-R1, $\lambda=100$}
\label{figappx:ganr1}
\end{figure*}

\begin{figure*}
\centering
\subfloat[Real]{
\includegraphics[width=0.44\textwidth]{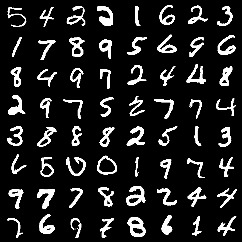}
}
\subfloat[Landscape 5000]{
\includegraphics[width=0.44\textwidth]{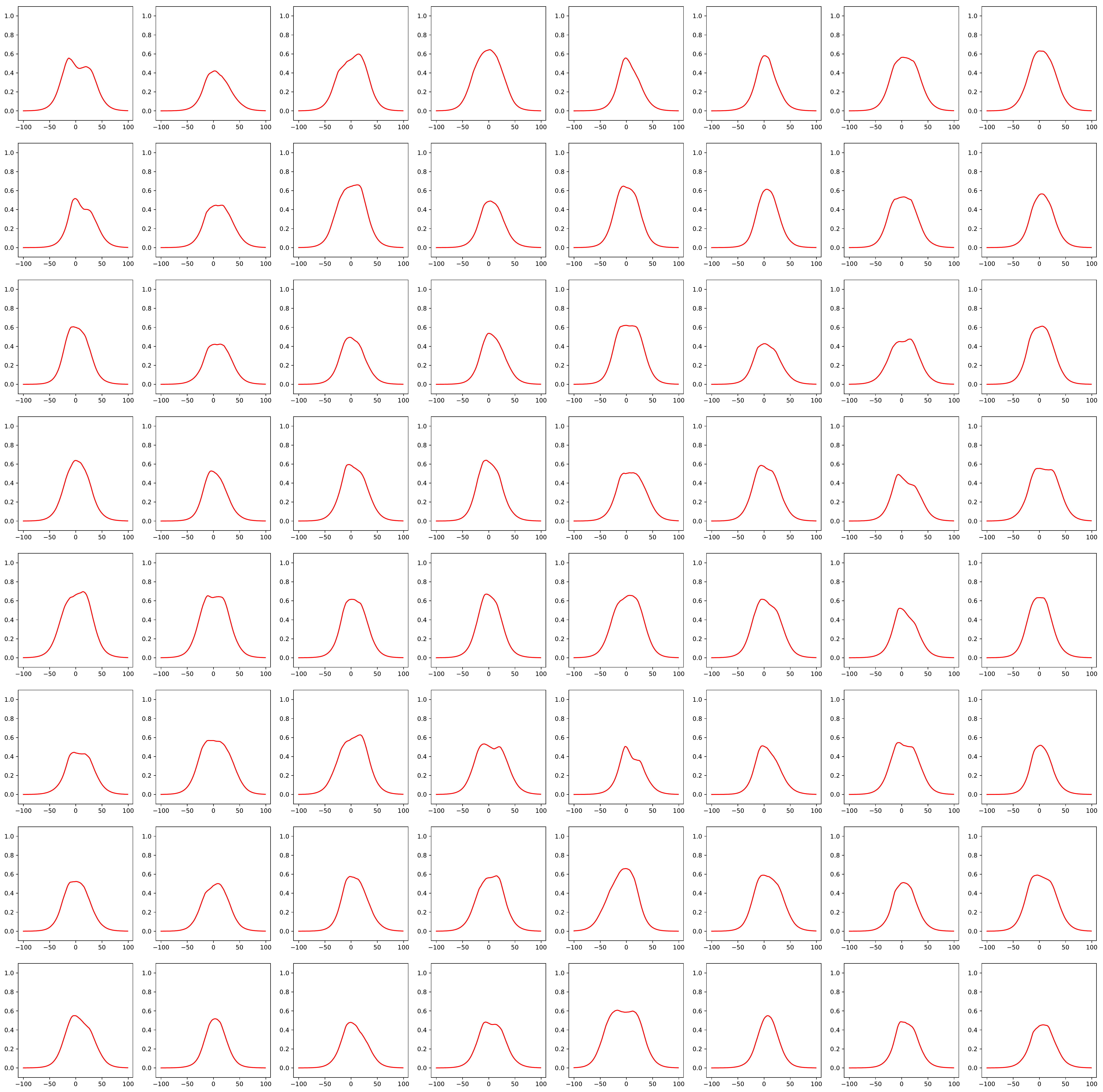}
}\\
\subfloat[Generated 50000]{
\includegraphics[width=0.44\textwidth]{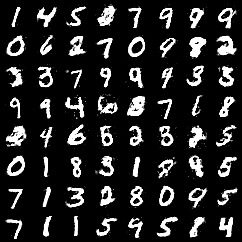}
}
\subfloat[Landscape 50000]{
\includegraphics[width=0.44\textwidth]{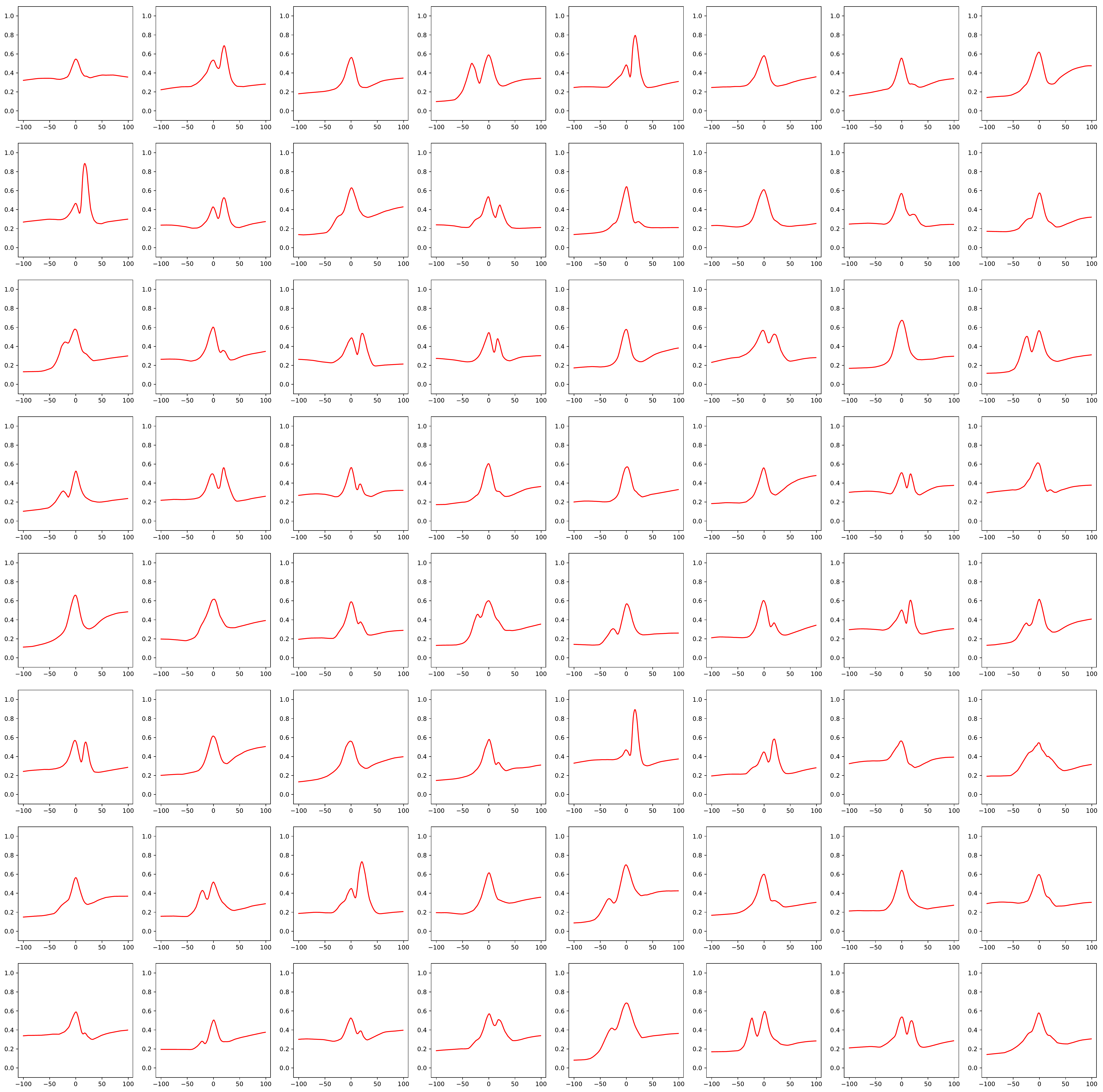}
}\\
\subfloat[Generated 100000]{
\includegraphics[width=0.44\textwidth]{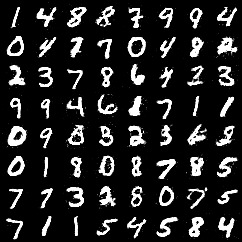}
}
\subfloat[Landscape 100000]{
\includegraphics[width=0.44\textwidth]{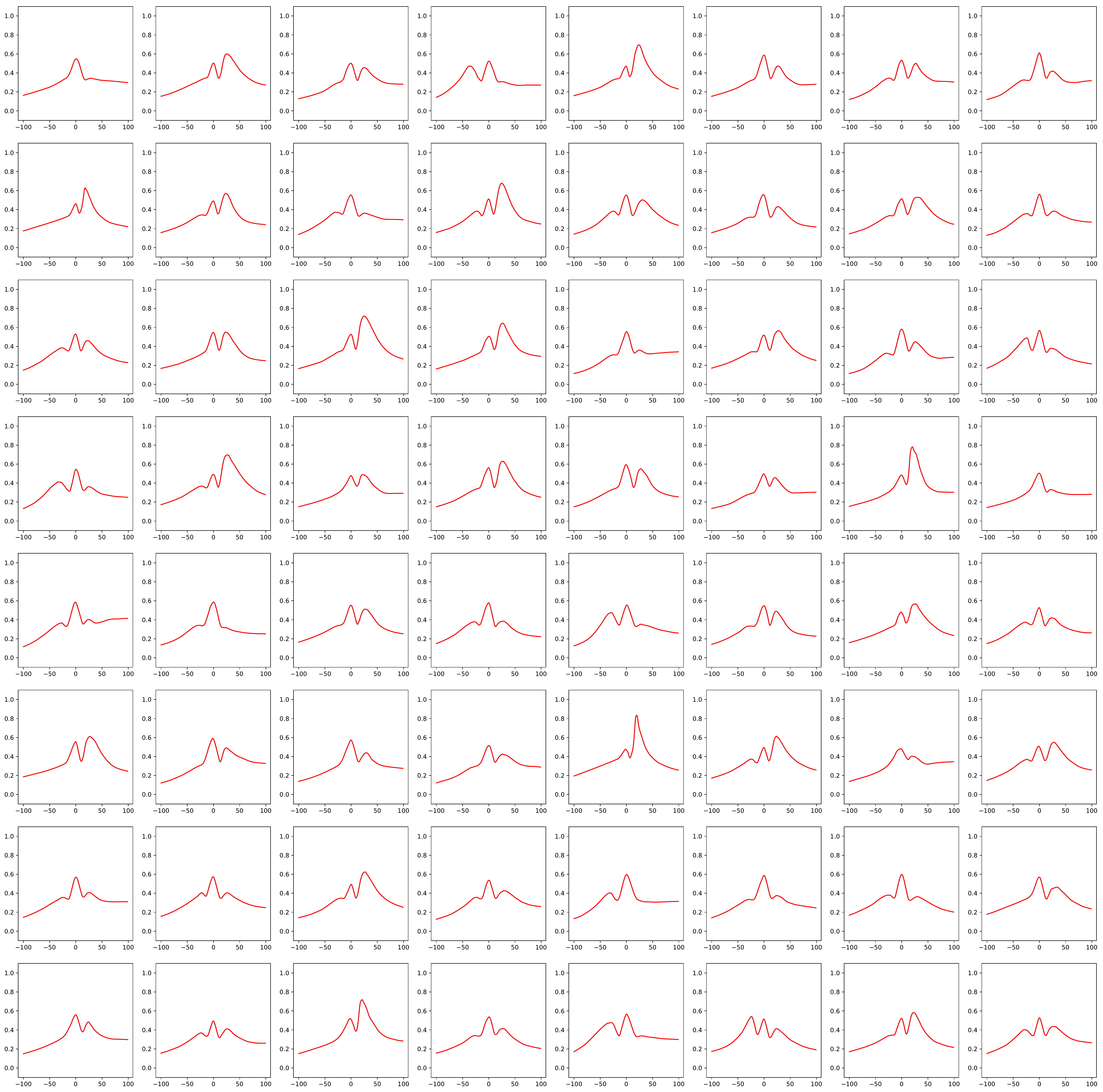}
}
\phantomcaption
\end{figure*}

\begin{figure*}
\centering
\ContinuedFloat
\subfloat[Generated 200000]{
\includegraphics[width=0.44\textwidth]{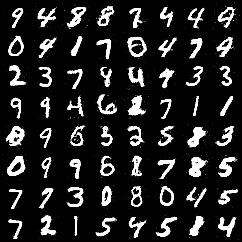}
}
\subfloat[Landscape 200000]{
\includegraphics[width=0.44\textwidth]{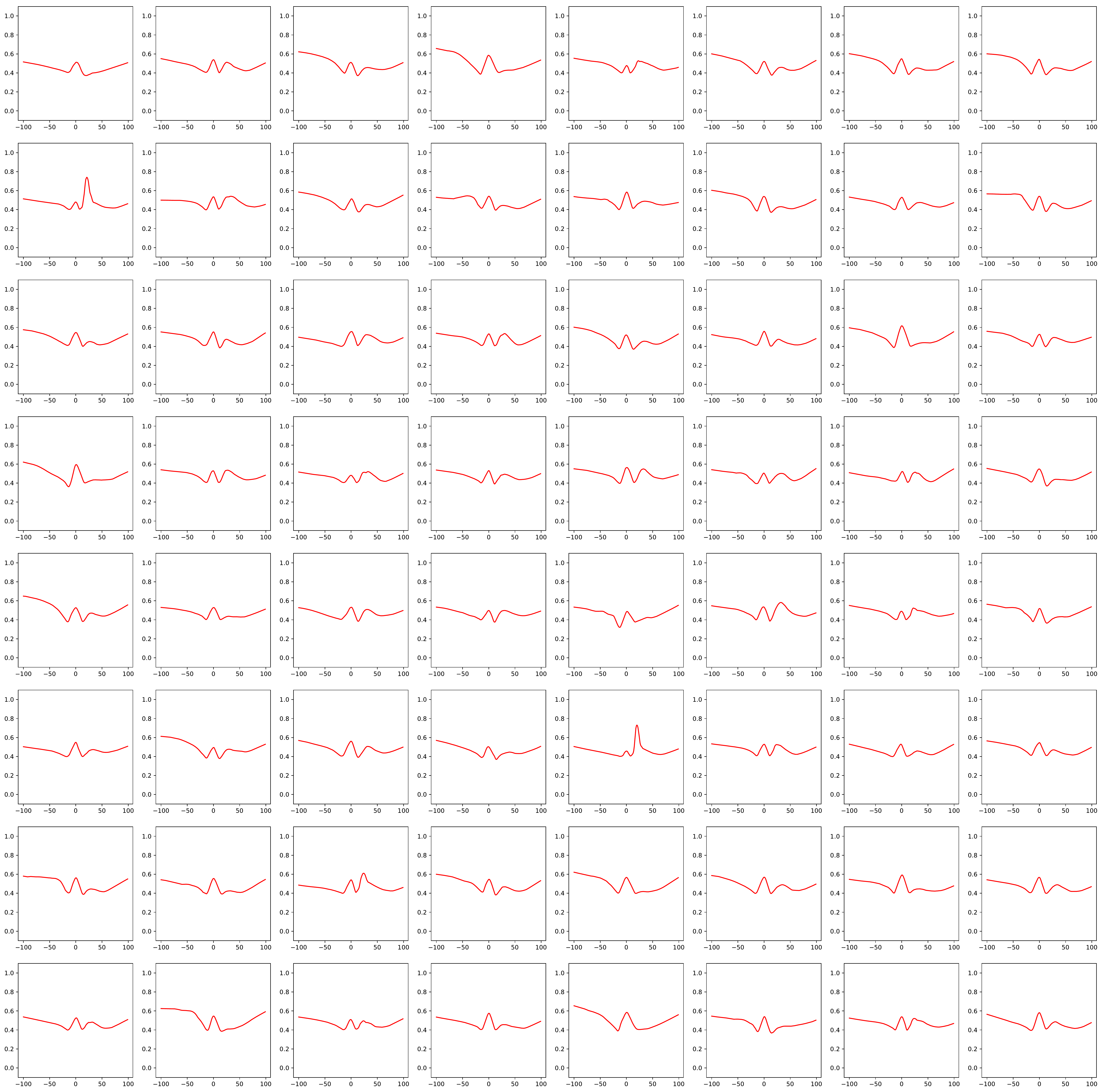}
}
\caption{GAN-0GP, $\lambda=100$.}
\label{figappx:gan0gp}
\end{figure*}

\begin{figure*}
\centering
\subfloat[Real]{\includegraphics[width=0.44\textwidth]{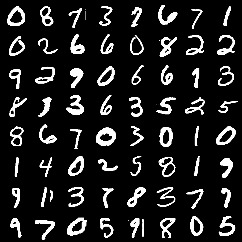}}
\subfloat[Landscape 5000]{\includegraphics[width=0.44\textwidth]{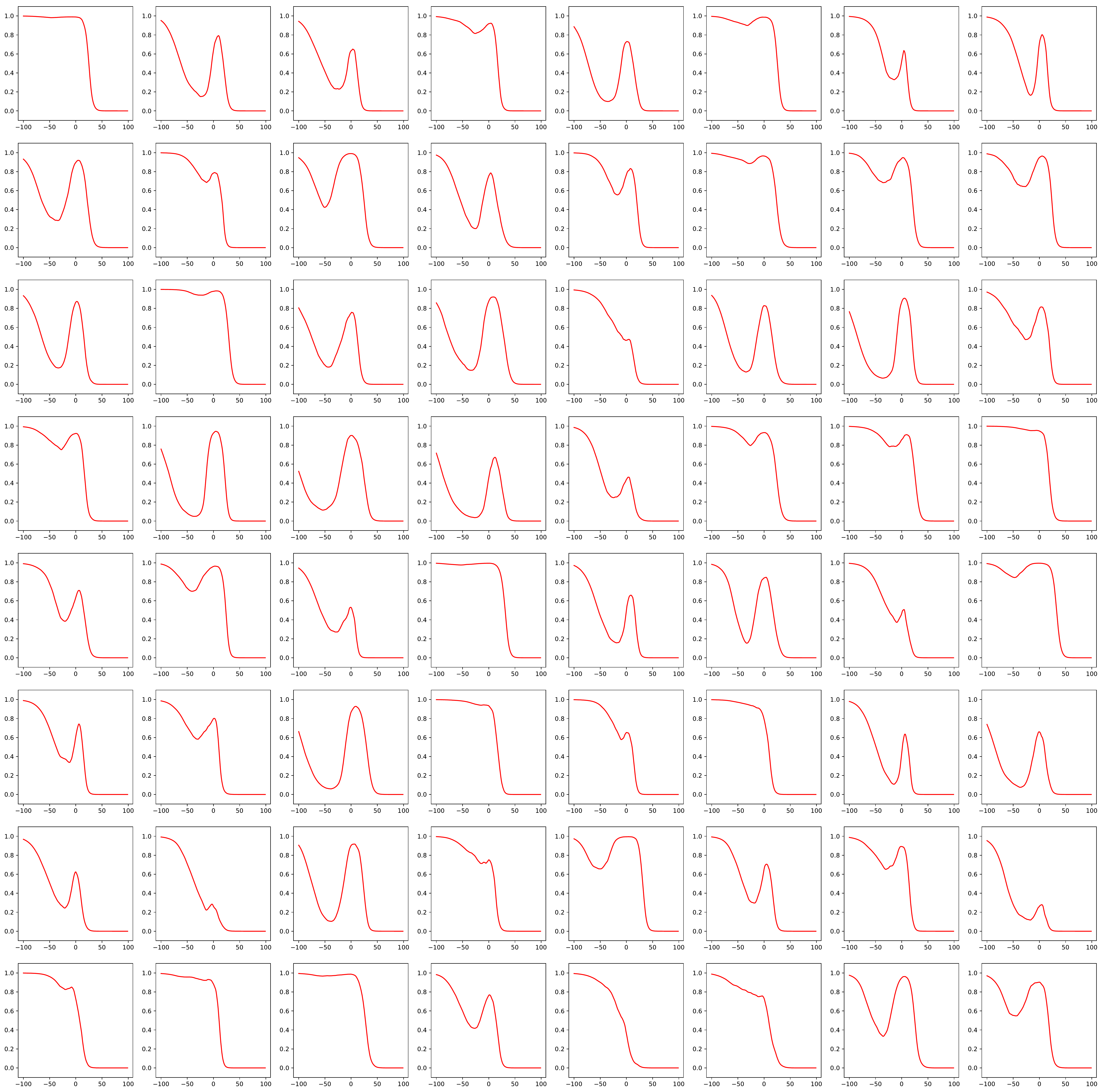}}

\subfloat[Generated 50000]{\includegraphics[width=0.44\textwidth]{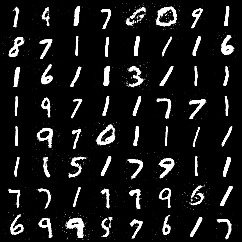}}
\subfloat[Landscape 50000]{\includegraphics[width=0.44\textwidth]{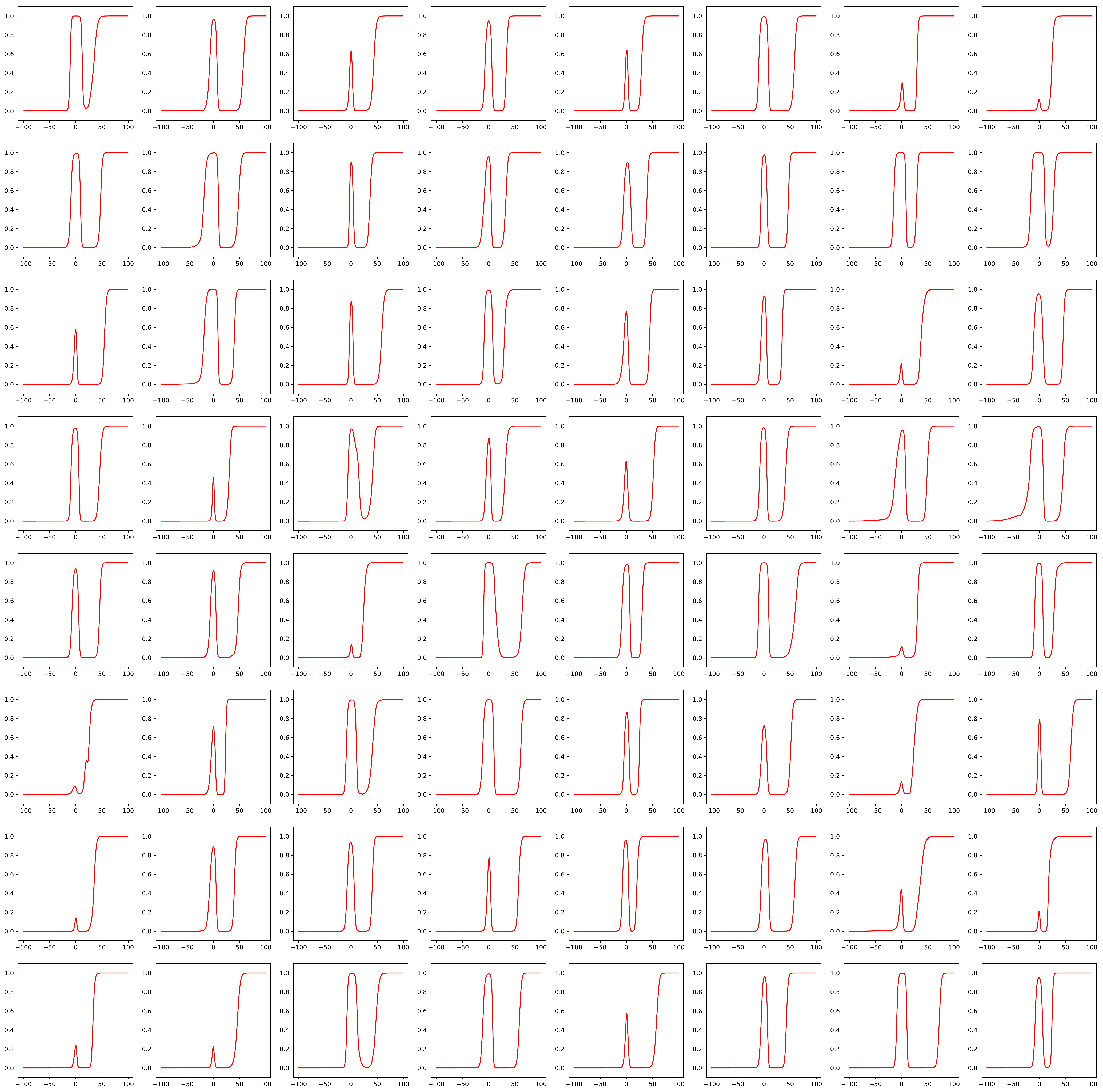}}

\subfloat[Generated 100000]{\includegraphics[width=0.44\textwidth]{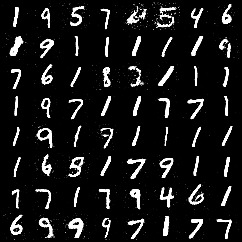}}
\subfloat[Landscape 100000]{\includegraphics[width=0.44\textwidth]{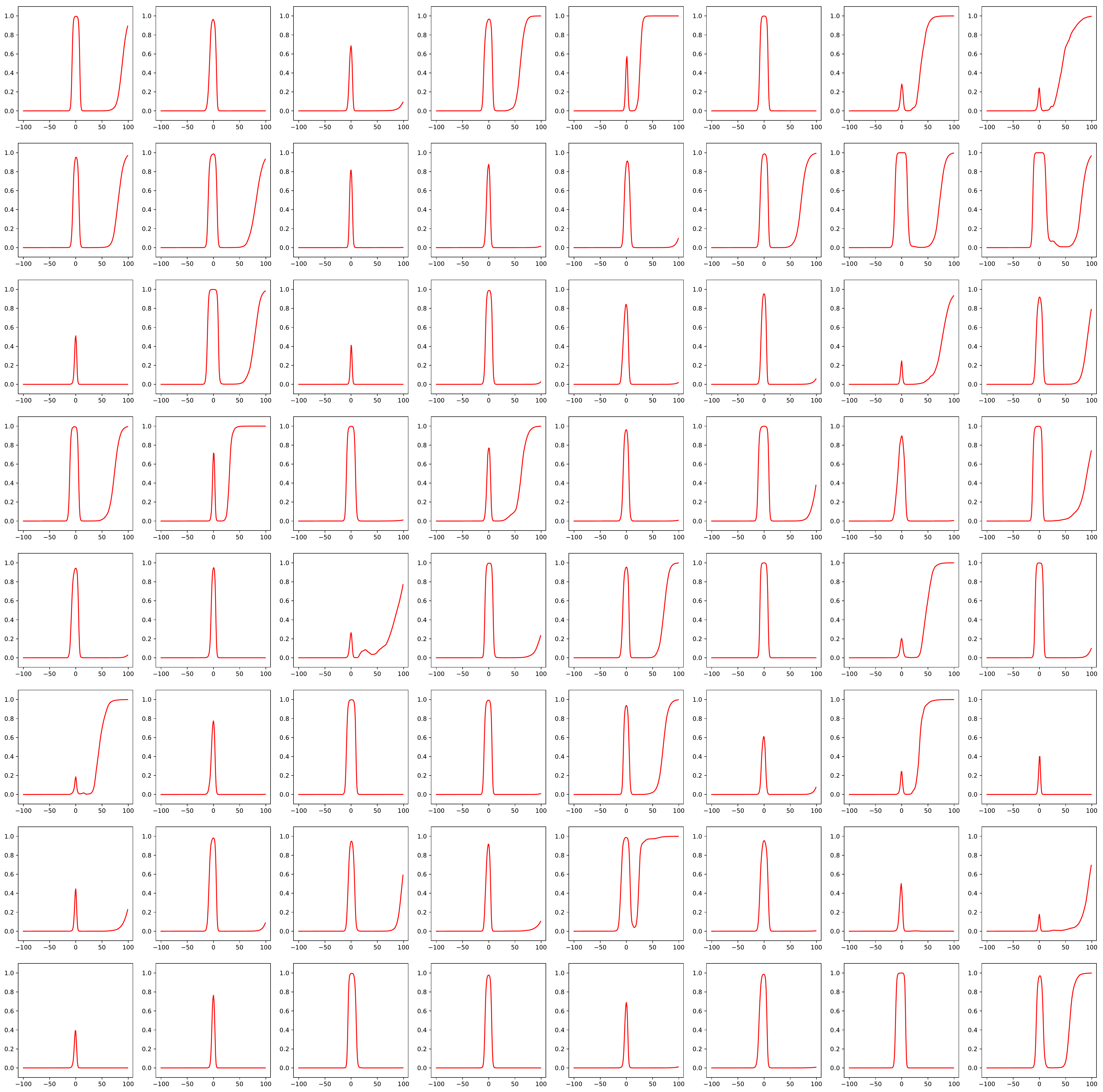}}
\phantomcaption
\end{figure*}

\begin{figure*}
\centering
\ContinuedFloat
\subfloat[Generated 200000]{\includegraphics[width=0.44\textwidth]{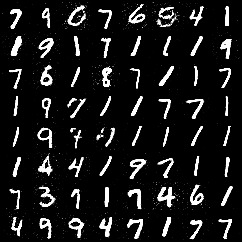}}
\subfloat[Landscape 200000]{\includegraphics[width=0.44\textwidth]{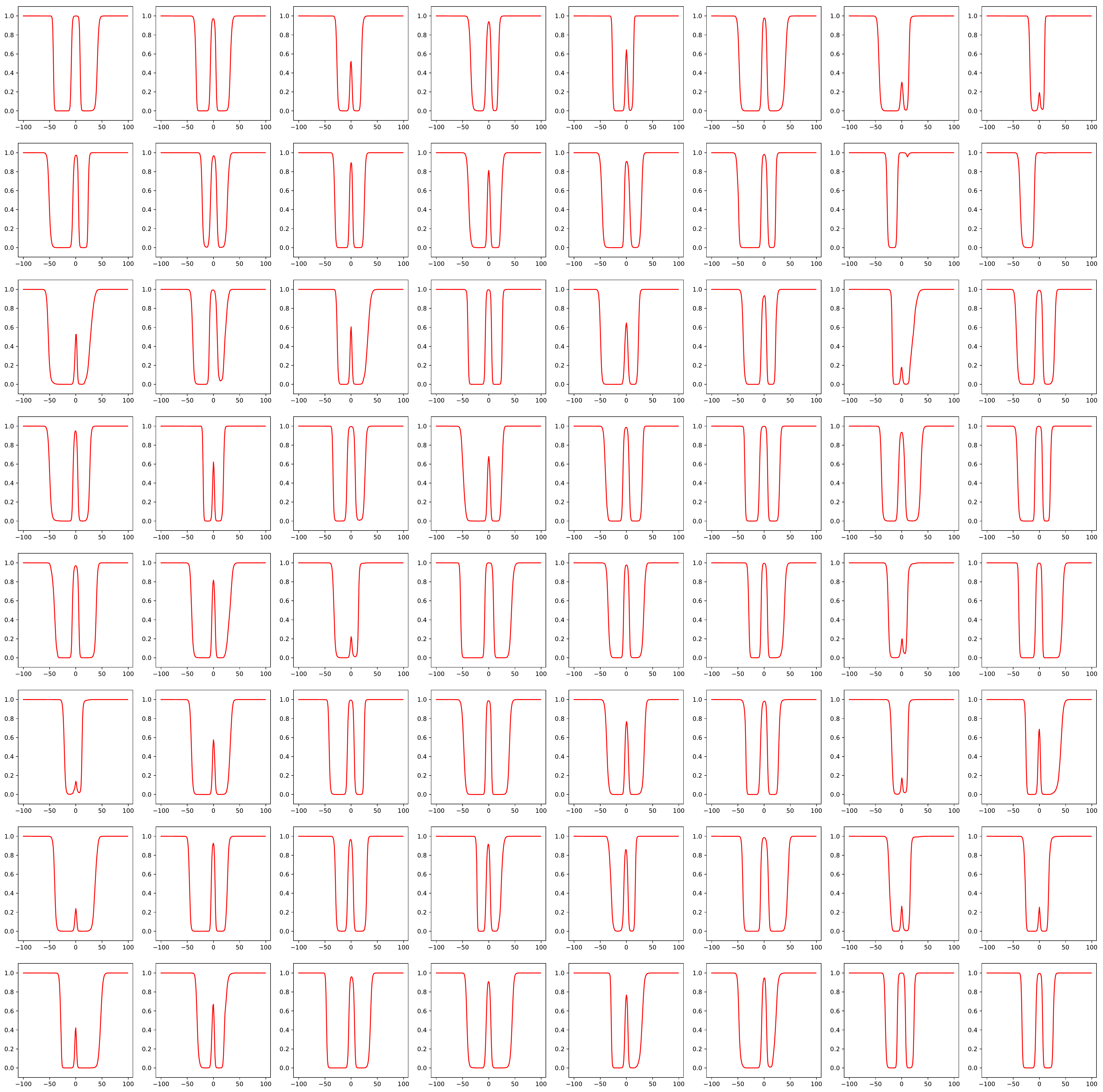}}
\caption{GAN-0GP with $\lambda = 10$.}
\label{figappx:gan0gp10}
\end{figure*}

\begin{figure*}
\centering
\subfloat[Real]{
\includegraphics[width=0.44\textwidth]{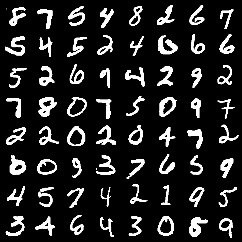}
}
\subfloat[Landscape 5000]{
\includegraphics[width=0.44\textwidth]{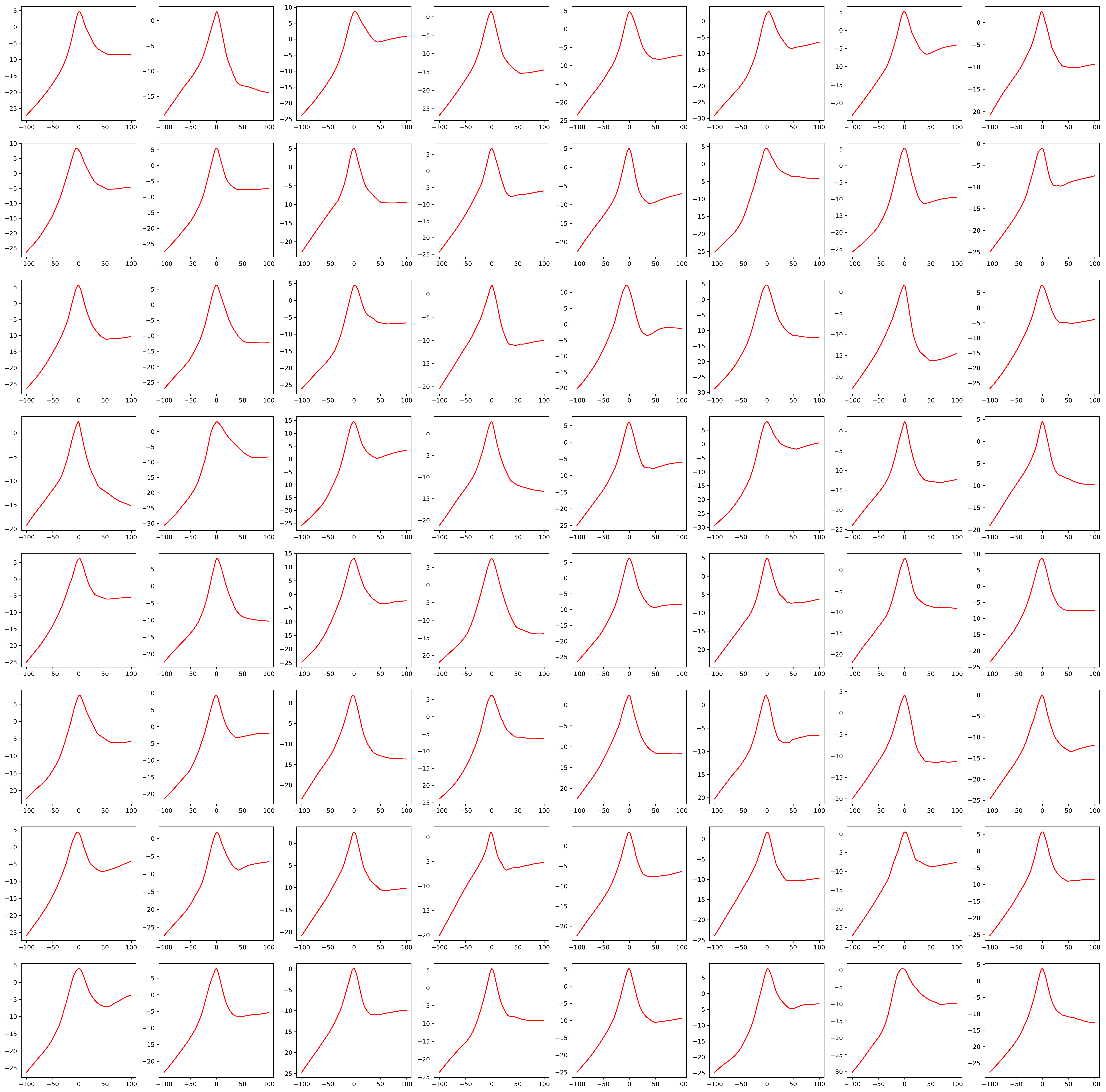}}

\subfloat[Generated 50000]{
\includegraphics[width=0.44\textwidth]{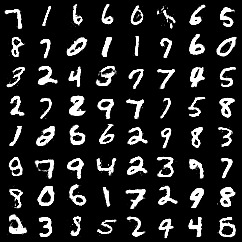}
}
\subfloat[Landscape 50000]{
\includegraphics[width=0.44\textwidth]{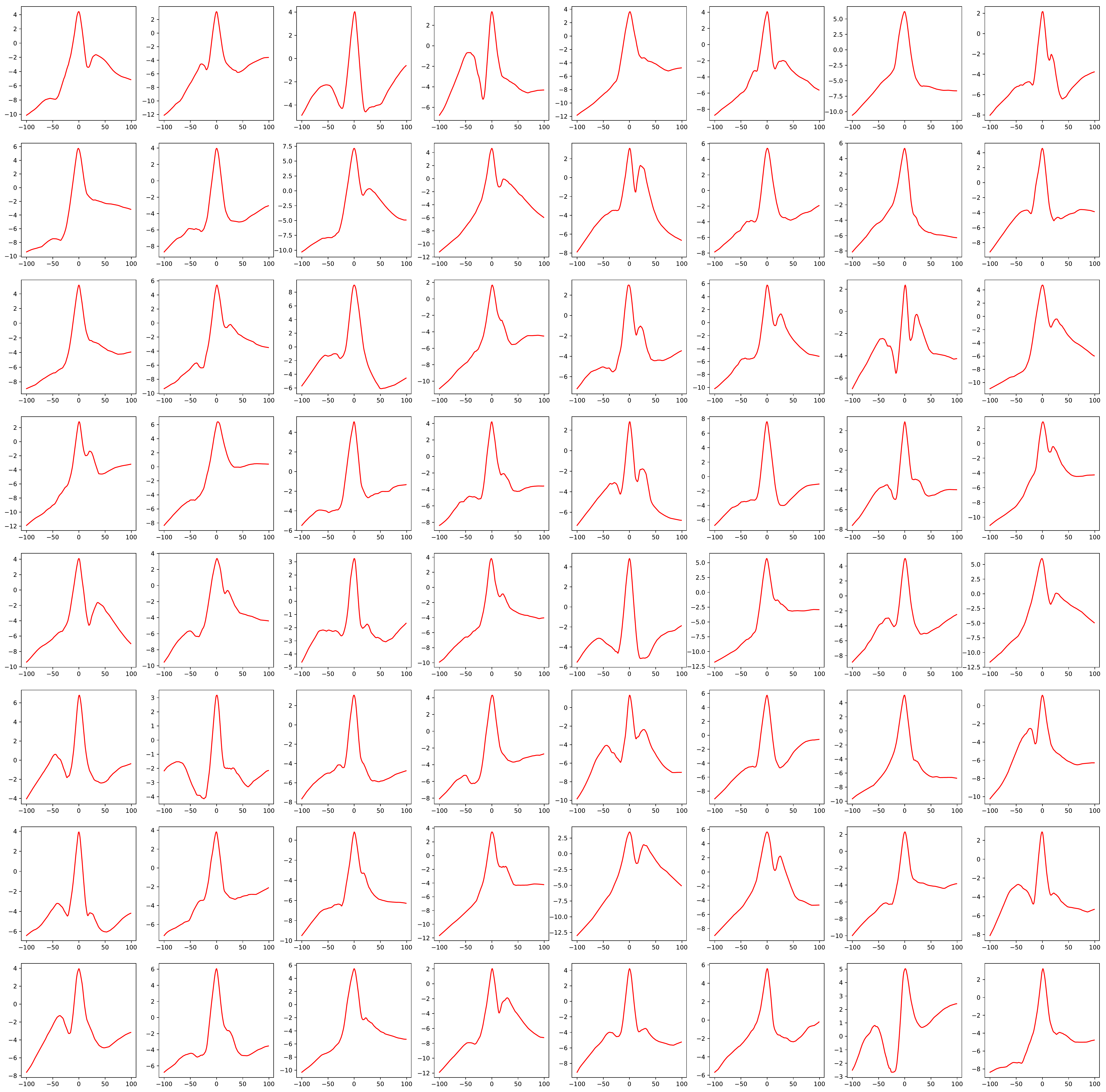}}

\subfloat[Generated 100000]{
\includegraphics[width=0.44\textwidth]{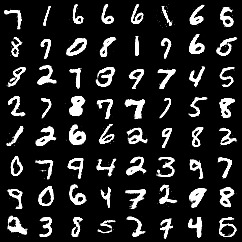}
}
\subfloat[Landscape 100000]{
\includegraphics[width=0.44\textwidth]{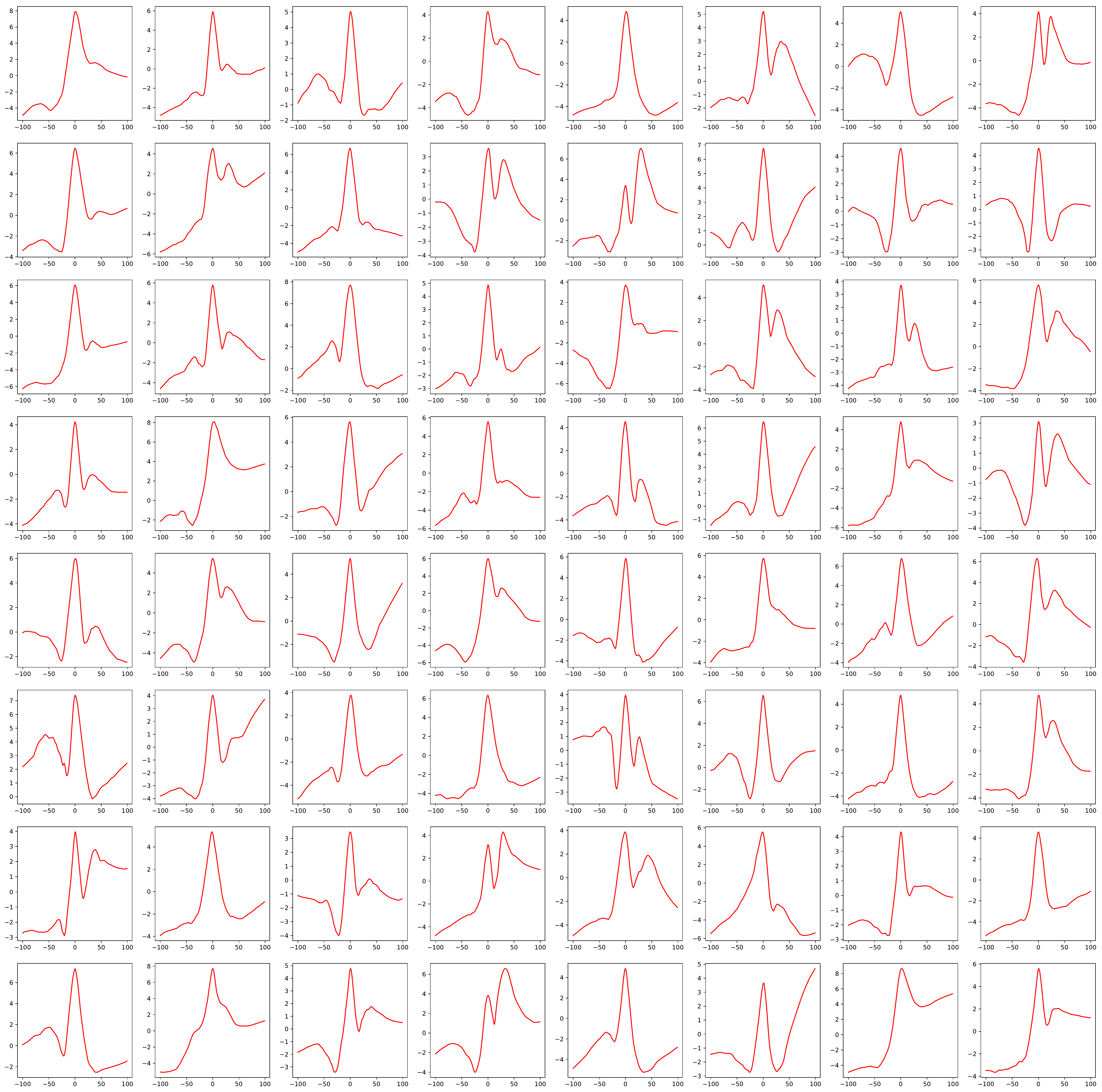}}
\phantomcaption
\end{figure*}

\begin{figure*}
\centering
\ContinuedFloat

\subfloat[Generated 200000]{
\includegraphics[width=0.44\textwidth]{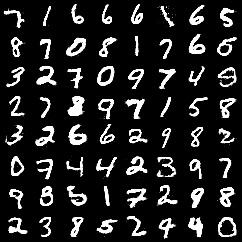}
}
\subfloat[Landscape 200000]{
\includegraphics[width=0.44\textwidth]{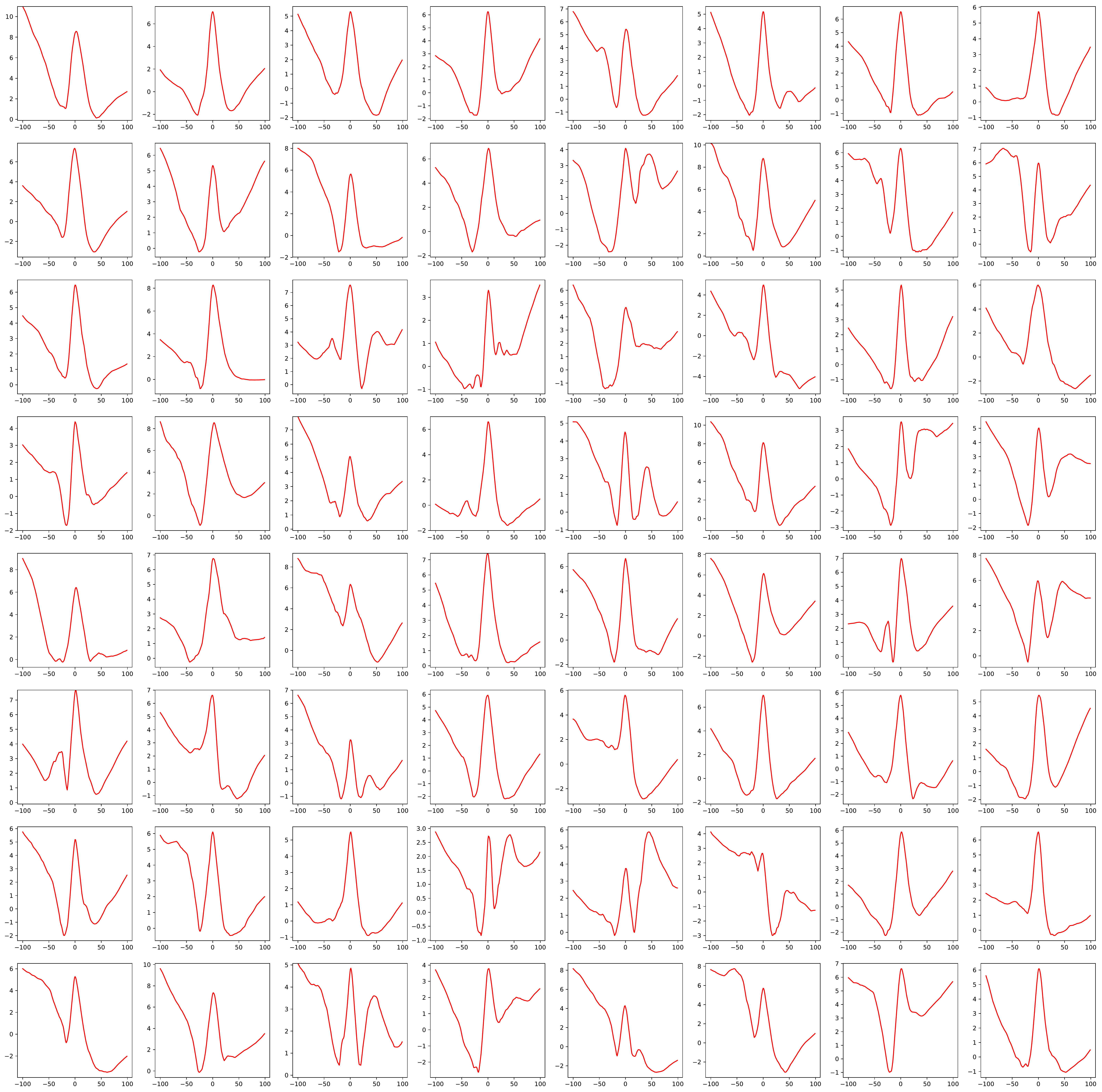}}

\caption{WGAN-GP with $\lambda=10$}
\label{figappx:wgangp}
\end{figure*}

\begin{figure*}
\centering
\subfloat[Real]{\includegraphics[width=0.44\textwidth]{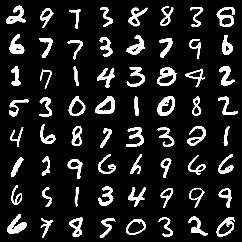}}
\subfloat[Landscape 5000]{\includegraphics[width=0.44\textwidth]{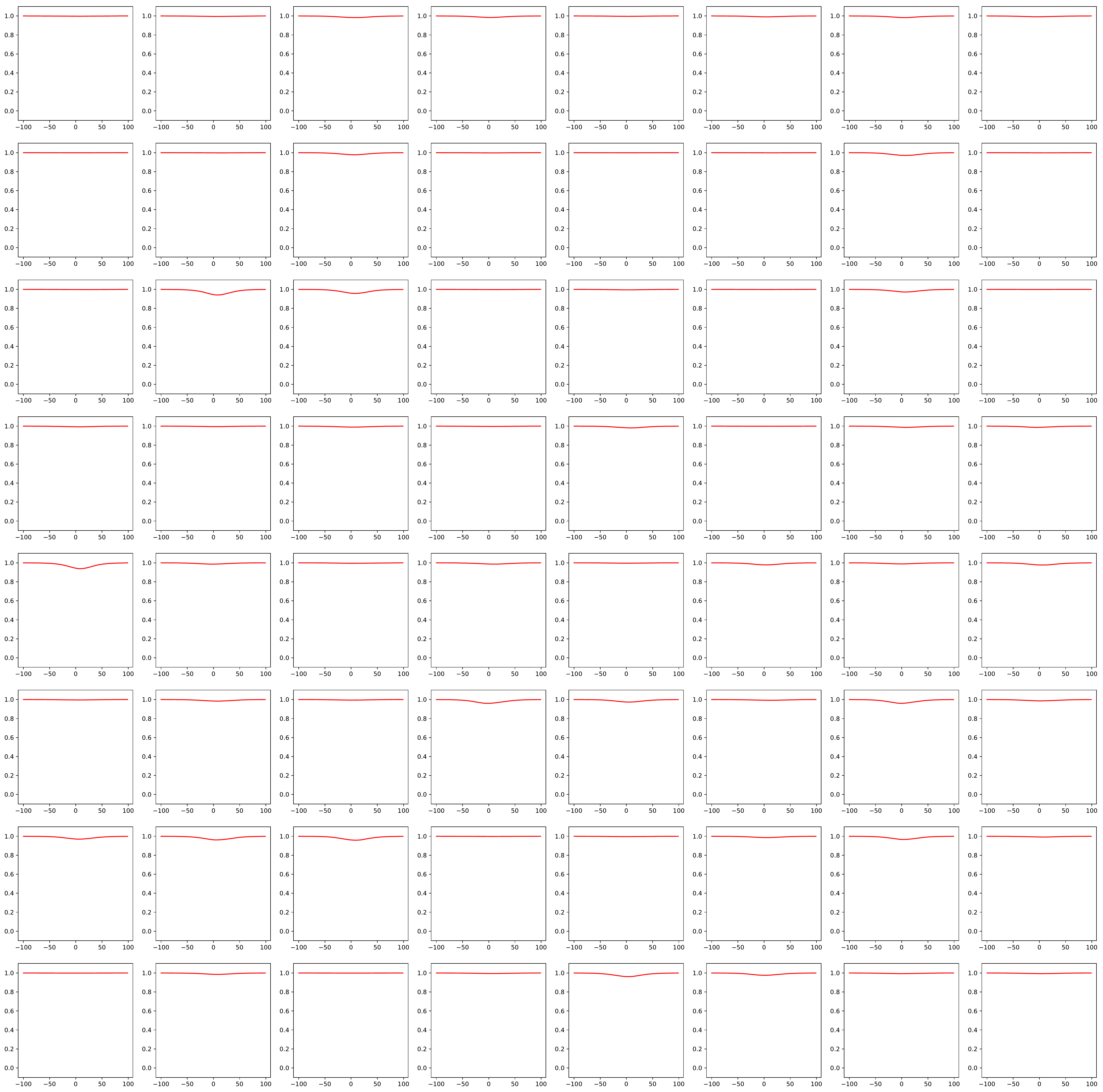}}

\subfloat[Generated 50000]{\includegraphics[width=0.44\textwidth]{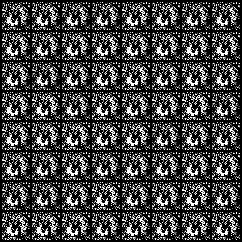}}
\subfloat[Landscape 50000]{\includegraphics[width=0.44\textwidth]{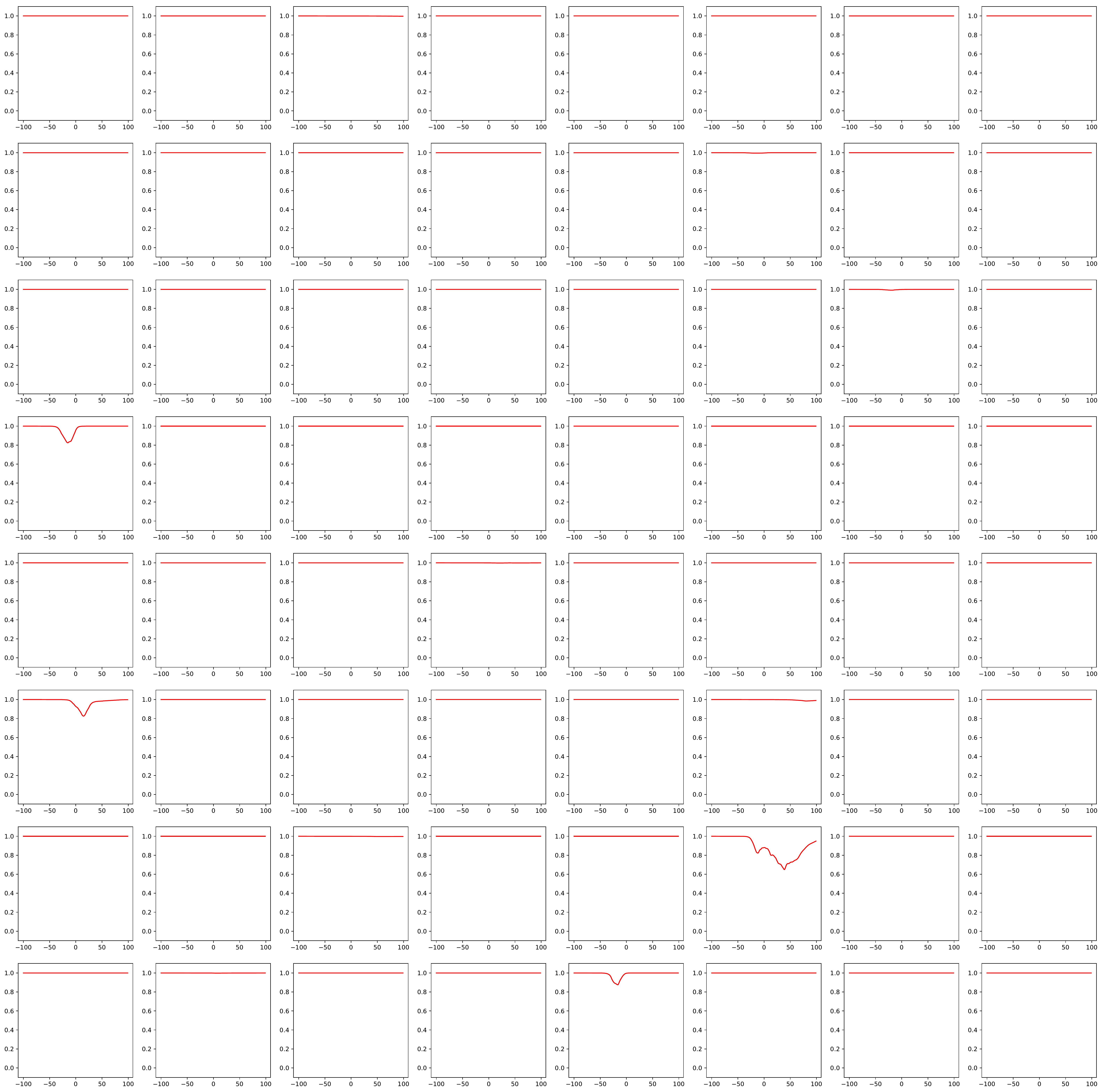}}

\subfloat[Generated 100000]{\includegraphics[width=0.44\textwidth]{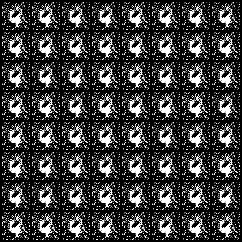}}
\subfloat[Landscape 50000]{\includegraphics[width=0.44\textwidth]{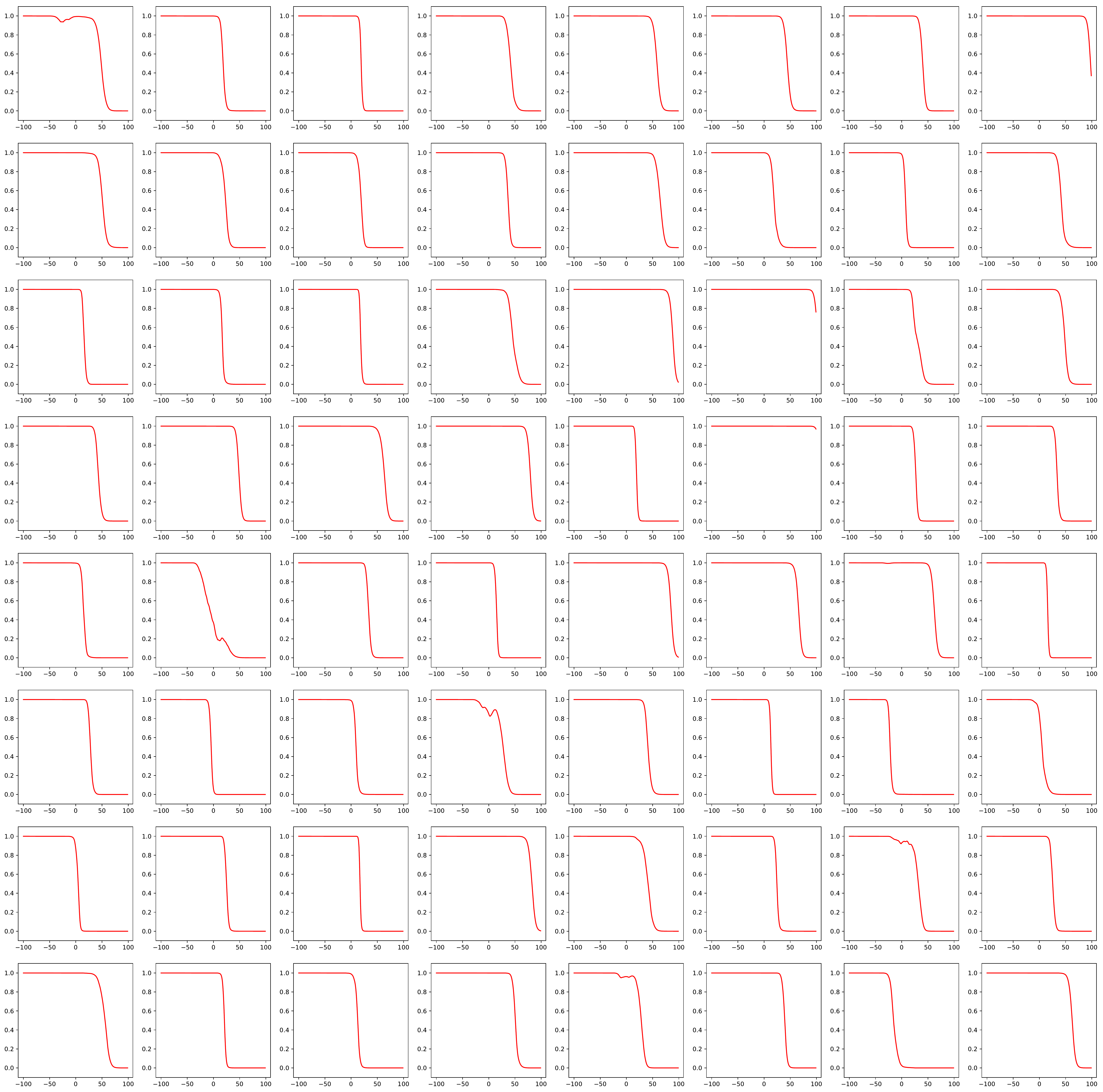}}
\phantomcaption
\end{figure*}

\begin{figure*}
\centering
\ContinuedFloat
\subfloat[Generated 200000]{\includegraphics[width=0.44\textwidth]{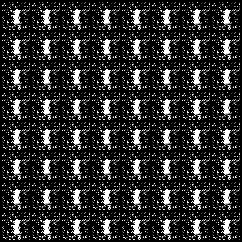}}
\subfloat[Landscape 200000]{\includegraphics[width=0.44\textwidth]{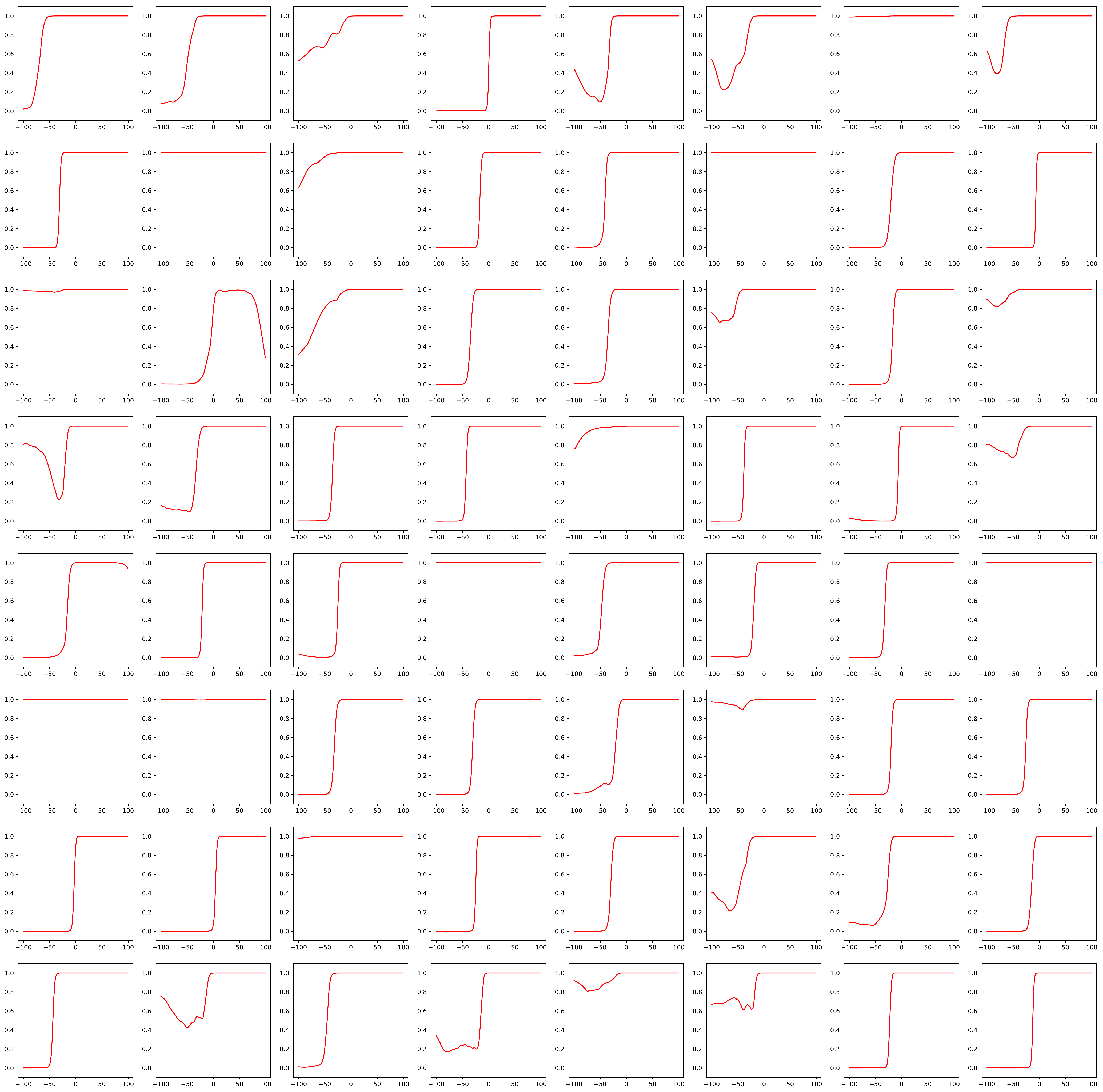}}

\caption{GAN-NS with SGD optimizer. The generator keep moving between regions of the data space and does not converge.}
\label{figappx:gannsSGD}
\end{figure*}

\end{document}